\documentclass[Afour,sageh,times]{sagej}

\usepackage{moreverb,url}
\usepackage{array}
\usepackage{textcomp}
\usepackage{stfloats}
\usepackage{url}
\usepackage{verbatim}
\usepackage{graphicx,wrapfig,lipsum}
\usepackage{amsmath}
\usepackage{tabularx} 
\usepackage{amsthm}
\usepackage{bm}
\usepackage{booktabs}
\usepackage{subcaption}
\usepackage{multirow, booktabs}

\usepackage{algorithm}% http://ctan.org/pkg/algorithm
\usepackage{algpseudocode}
\usepackage{amssymb}
\usepackage{xcolor}

\newcommand{\termdef}[2]{\hypertarget{#1}{#2}}
\newcommand{\termref}[2]{\hyperlink{#1}{#2}}

\usepackage[colorlinks,bookmarksopen,bookmarksnumbered,citecolor=red,urlcolor=red]{hyperref}

\newcommand\BibTeX{{\rmfamily B\kern-.05em \textsc{i\kern-.025em b}\kern-.08em
T\kern-.1667em\lower.7ex\hbox{E}\kern-.125emX}}

\def\volumeyear{2025}
\newtheorem{proposition}{Proposition}
\begin{document}

\runninghead{Zhang and Calinon}
% use `\textit{et~al.}' if there are three or more authors.

\title{Physics-Informed Eikonal Caging for Whole-Arm Manipulation Planning
}

\author{
Yan Zhang\affilnum{1,2}, Yiming Li\affilnum{1,2}, Yifei Dong\affilnum{3}, Florian T. Pokorny\affilnum{3}, and Sylvain Calinon\affilnum{1,2}}

\affiliation{\affilnum{1}Idiap Research Institute, Martigny, Switzerland\\
\affilnum{2}\'Ecole Polytechnique F\'ed\'erale de Lausanne (EPFL), Lausanne, Switzerland\\
\affilnum{3}KTH Royal Institute of Technology, Stockholm, Sweden
}

\corrauth{Yan Zhang, Idiap Research Institute, 
Rue Marconi 19,
1920, Martigny
Switzerland}

\email{yan.zhang@idiap.ch}

\begin{abstract}
Planning contact-rich whole-arm manipulation is challenging because interactions that involve extended robot geometry give rise to complex contact dynamics that are difficult to model accurately. This creates a need for planning principles that do not rely heavily on precise contact models. Caging offers one such geometric notion of robustness to modeling inaccuracy by restricting object escape through geometrically enclosing the object and restricting its escape. However, existing caging formulations are difficult to incorporate into continuous optimization-based manipulation planning.
We reformulate caging as a minimum-time escape problem in which the object seeks to leave an enclosing robot geometry in the shortest time. This yields a continuous escape-time field that measures the robot's enclosure quality and we show it satisfies an eikonal equation. We therefore can approximate this field using a physics-informed neural network, producing a smooth differentiable representation that can be embedded directly into manipulation planning. The resulting objective supports whole-arm manipulation planning to favor robot configurations resisting object escape. This improves the manipulation robustness to contact model mismatch, thus enabling planning with simplified contact models, including quasi-dynamic approximations and simplified object geometry. Across simulation and real-world experiments, we show improved robustness to disturbances and contact-model mismatch relative to baselines. These results suggest that geometric enclosure can serve as a practical robustness primitive for whole-arm manipulation. A supplementary video, which includes an intuitive overview of our method and experiment video results, is available on our \href{https://sites.google.com/view/piec4wam?usp=sharing}{project webpage.}
\end{abstract}

\keywords{Whole-Arm Manipulation, Robotic Caging, Eikonal Equation, Robust Planning}

\maketitle

\section{Introduction}
Planning contact-rich manipulation remains a central challenge in robotics, particularly when interactions extend beyond the end-effector to involve extended portions of the robot arm \citep{barreiros2025learning,leve2025scaling}. Whole-arm manipulation can use distributed contact across the arm to move objects that are difficult to grasp. However, planning such behaviors is challenging because object motion depends on whole-arm contact geometry and mode transitions at multiple contact points, both of which are difficult to model accurately. As a result, existing contact-rich manipulation pipelines tend to simplify the contact dynamics \citep{hogan2020reactive, pang2023global} by restricting interaction to the end-effector, using smoothed contact approximations, or assuming quasi-static/dynamic interaction, which can degrade performance or fail during execution.

A central failure mode in contact-rich manipulation is object \emph{escape}: the object slips or reorients, thus leaving the robot's controllable region. Robotic caging \citep{makita2017survey} provides a geometric notion of robustness against such escape. An object is \emph{caged} when its feasible escape motions are blocked by surrounding robot geometry, so that the robot geometrically \emph{encloses} the object. Prior work has mainly studied caging for grasp analysis, certification, and enclosure reasoning \citep{bircher2021complex, aceituno2023certified, dong2024characterizing}. However, most existing formulations are static: they evaluate a fixed configuration or estimate escape difficulty through sampling \citep{dong2024characterizing}, graph search \citep{bircher2021complex}, or certification procedures \citep{aceituno2023certified}. These approaches are valuable for static analysis, yet difficult to integrate into continuous-time whole-arm trajectory optimization.

We revisit caging from a dynamical and adversarial perspective by asking: \emph{how difficult is it for an object to escape an enclosing robot geometry when it moves to escape as quickly as possible?} We reformulate caging as a minimum-time escape optimal control problem in which the object seeks to escape in shortest time with bounded maximal velocity. This formulation yields a continuous escape-time field over collision-free robot--object configurations, as illustrated by the colormap and white contours in Figure~\ref{fig:illustration_wa_mani}.

We further show that this escape-time field satisfies an eikonal equation with boundary conditions defined by configurations outside the robot workspace. Classical grid-based eikonal solvers \citep{sethian1996fast,sanguinetti2015sub} do not scale to high-dimensional robot--object configuration spaces, nor do they provide smooth gradients with respect to articulated robot geometry. We therefore approximate the field in a self-supervised manner using physics-informed neural networks (PINNs), without requiring time-consuming escape difficulty labeling, unlike prior robotic caging approaches \citep{bircher2021complex,dong2024quasi}.

More importantly, this approximation yields a smooth representation that can be queried in batch and differentiated for whole-arm manipulation planning, as shown in Figure~\ref{fig:illustration_wa_mani}. Moreover, intuitively, the resulting formulation can be interpreted as a min--max game: the robot seeks configurations with large escape time, while the object would prefer motions that minimize it. Empirically, optimizing this escape-time field biases the whole-arm manipulation planner toward enclosing robot configurations that resist object escape, which improves robustness to contact-model mismatch in the evaluated tasks. As a result, the planner can better tolerate simplified planning models, such as quasi-dynamic contact models or simplified object geometry. For example, in Figure~\ref{fig:illustration_wa_mani}, the robot successfully pushes the yellow \emph{box} into the goal region even though the object is approximated as a \emph{circle} during planning.

\begin{figure*}[!t]
    \centering
        \centering
        \includegraphics[width=\linewidth]{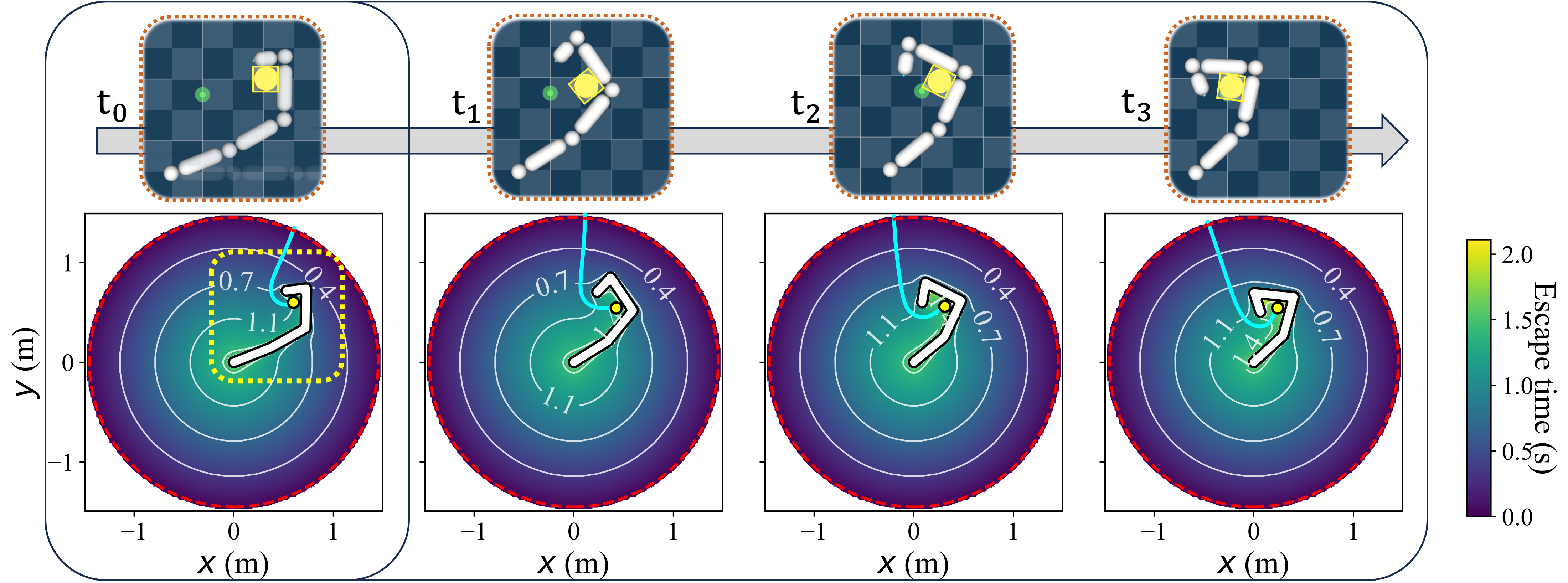}
        \caption{Conceptual overview of whole-arm caging manipulation for robust whole-arm manipulation planning with simplified contact dynamics. Top: the robot transports a box-shaped object using distributed whole-arm contact along a manipulation trajectory planned with a simplified object model, where the box is under-approximated by a circular geometry (yellow circle). Bottom: the corresponding escape-time field over the robot workspace (boundary shown by the red dashed line), where colors and contours indicate the minimum time required for the object to escape the enclosing robot geometry under bounded velocity. The cyan curves illustrate minimum-time escape paths at each time step. By planning motions that maximize object escape time, the robot maintains geometric enclosure throughout the manipulation. This geometric enclosure provides improved robustness to contact model mismatch, enabling the planner to tolerate simplified contact dynamics during planning, such as the contact-geometry simplification illustrated here. The upper panel corresponds to the region highlighted by the dashed rectangle in the bottom panel.}
        \label{fig:illustration_wa_mani}
        \vspace{-0.5cm}
\end{figure*}

The learned escape-time field therefore provides an \emph{explicit geometric robustness measure} for whole-arm manipulation planning: larger values correspond to robot--object configurations from which escape is more difficult under worst-case object escape motion. By optimizing this quantity jointly with task objectives, the planner can explicitly trade task progress against resistance to object escape, while exploiting whole-arm geometry more effectively during manipulation.

In summary, our contributions include:
\begin{itemize}
    \item We reformulate robotic caging as a minimum-time escape problem and show that the resulting escape-time field satisfies an eikonal equation.

    \item We introduce a physics-informed neural approximation of the escape-time field that is smooth and differentiable with respect to robot--object configurations, and can therefore be directly embedded into whole-arm manipulation planning.
    
    \item We show that maximizing the escape-time field biases planners toward enclosing robot geometry that resists object escape, and provide a game-theoretic interpretation of the resulting whole-arm manipulation formulation.
    
    \item Through simulation and real-world experiments, we demonstrate improved whole-arm manipulation robustness to disturbances and contact-model mismatch, enabling effective planning with simplified dynamics and geometry models.
\end{itemize}

\emph{Manuscript structure:} The remainder of the paper develops these ideas progressively. Section~\ref{sec:eikonal_caging} introduces the minimum-time escape formulation and its eikonal relationship. Section~\ref{sec:wacaging_conf} learns a differentiable escape-time field for static whole-arm \emph{configuration} optimization. Section~\ref{sec:wacaging_mani} embeds this field into whole-arm \emph{manipulation} trajectory planning. Each of these three sections concludes with an illustrative example to convey the key ideas. Experiments are conducted in Section~\ref{sec:robust_analyse} and~\ref{sec:experiments}. Section~\ref{sec:robust_analyse} analyzes whole-arm manipulation robustness under controlled disturbances and contact dynamic simplifications. Section~\ref{sec:experiments} then validates transfer to realistic geometry and real-world experiments.

\section{Related Work}
\subsection{Robotic Caging}

Caging was introduced in robotics as a geometric notion describing configurations in which an object, while not rigidly grasped, is constrained to remain within a bounded region of configuration space and cannot escape arbitrarily far without collision \citep{rimon1996caging}. Unlike force-closure grasps, caging emphasizes geometric enclosure rather than force balance, making it an ideal alternative to form-closure grasping for robust manipulation planning. Following this definition, a line of work has studied caging as a means of certifying that an object is confined to a compact subset of configuration space. For example, \cite{rodriguez2012caging} and \cite{aceituno2023certified} establish conditions under which caging provides formal guarantees that an object remains bounded, enabling certified grasping and manipulation. For objects with specific topological features such as graspable handles, \citep{pokorny2013a, stork2013b, stork2013c} developed approaches based on Gauss Linking Integrals. In general, to construct caging configurations, previous approaches mainly rely on explicit exploration of contact or configuration space. Graph-search methods defined over contact manifolds have been proposed to identify caging regions \citep{allen2015robust, bunis2018equilateral}, while randomized and sampling-based planners have been used to test whether an object can escape a given enclosure \citep{song2021herding, varava2021free}.

Beyond binary notions of caging, prior work has introduced quantitative measures of \emph{partial caging} or enclosure quality. \cite{makapunyo2013measurement} proposed measuring caging quality using average escape path length or elapsed time estimated via sampling-based motion planning. 
The work of \cite{varava2019partial} introduced clearance-based definitions of partial caging and estimated escape difficulty through repeated sampling-based searches. A related line of work considers \emph{energy-bounded caging}, in which escape is defined in terms of the minimum energy required to overcome an external potential, such as gravity \citep{mahler2016energy}. Extensions of this idea have explored partial and margin-based caging under energy constraints \citep{dong2024quasi, dong2024characterizing, dong2026robustness}, as well as applications to gripper, hand or tool design that maximize caging regions or robustness margins \citep{bircher2021complex, dong2025cagecoopt}. These works highlight the importance of escape-based reasoning for robustness, but remain primarily focused on static analysis, gripper design, and tool selection rather than whole-arm manipulation planning.

More broadly, existing approaches treat escape difficulty primarily as a static or sampling-based evaluation metric. Escape effort is typically estimated indirectly through repeated planning \citep{makapunyo2013measurement, varava2019partial, dong2024quasi, dong2024characterizing}, graph search \citep{allen2015robust, bunis2018equilateral}, or expensive mixed-integer optimization \citep{aceituno2023certified}. In contrast, we formulate caging as a minimum-time optimal control problem and show that the resulting escape-time field satisfies an eikonal equation with boundary conditions defined on configurations outside the robot's workspace, which enables the escape-time field to be approximated using physics-informed neural networks in a self-supervised manner, without requiring expensive escape effort labeling as previous caging metric learning approaches \citep{varava2019partial, dong2026robustness}.

\subsection{Signed Distance Fields and Eikonal-Based Methods}
Signed distance fields (SDFs) are widely used in robotics to represent robot and environment geometry, enabling smooth distance query and gradient-based optimization for motion planning \citep{koptev2022neural, li2024representing, Li24RSS}. In this work, we assume access to an SDF representation of the robot arm and workspace boundary for distance query. In our implementation, this representation follows prior work in \citep{li2024representing}, while the proposed formulation is agnostic to distance query modules.

Eikonal equations and their neural approximations have been studied for motion planning and navigation, where the resulting time or cost fields characterize shortest paths or traversal costs in collision-free configuration space \citep{huh2021cost, ni2022ntfields, Li26IJRR}. These approaches typically exploit the gradient of the learned field to guide robot motion while avoiding obstacles. In this paper, we show that the escape-time field arising from our minimum-time caging formulation satisfies an eikonal equation with a fixed boundary condition that corresponds to the escaping goal region boundary. To approximate this field efficiently, we adopt a physics-informed neural network (PINN) framework and tailor it to our setting by incorporating fixed goal-region boundary conditions and conditioning the field on robot configurations. Rather than using the resulting field gradients for navigation, we interpret the escape-time field as a measure of robot enclosure and object escape difficulty under whole-arm geometric constraints. This field is used directly to reason about caging quality and to shape robust whole-arm manipulation behaviors.

\subsection{Contact-Rich Manipulation Planning}
Contact-rich manipulation planning remains challenging due to the inherently hybrid nature of contact dynamics, in which the system evolves under multiple smooth modes separated by contact mode transitions \citep{hogan2020reactive, pang2023global}. These transitions induce non-smooth dynamics for which local Taylor approximations no longer hold, substantially complicating the application of trajectory optimization methods \citep{pang2023global}. As a result, the complexity of contact dynamics makes model-based planning for contact-rich manipulation particularly difficult.

\subsubsection{End-Effector–Centric Manipulation}
A common strategy for mitigating contact complexity is to restrict interaction to the end effector and adopt simplified contact models, such as quasi-static or quasi-dynamic approximations, that reduce second-order dynamics to first-order velocity-level constraints.
For example, \cite{hogan2020reactive} model contact modes using discrete variables under quasi-static approximation and formulate fingertip pushing as a mixed-integer nonlinear program. Related work incorporates non-smooth contact dynamics through complementarity constraints within trajectory optimization frameworks under quasi-dynamics approximation, giving rise to contact-implicit trajectory optimization (CITO) methods for pushing, sliding, and pivoting tasks \citep{moura2022non, zhang2025simultaneous}.
More recently, zeroth-order optimization approaches have been shown to effectively bypass non-smoothness by treating the contact dynamics as a black box. For instance, \cite{jankowski2025robust} explore Covariance Matrix Adaptation Evolution Strategy (CMA-ES) \citep{hansen2016cma} for trajectory optimization in end-effector–centric planar pushing tasks with quasi-static contact dynamics, demonstrating strong performance in contact-rich manipulation tasks.

Taken together, these methods highlight an important practical insight: contact dynamics should be simplified to make planning more tractable. Restricting contact to the end effector is one effective way to make quasi-static or quasi-dynamic contact models simpler, as it limits the range of possible contact interactions and reduces the complexity of the resulting dynamics. However, this restriction also limits the robot’s manipulation capabilities. Many tasks involving large or bulky objects cannot be reliably performed using only fingertip contact, and instead require extended portions of the robot arm to provide distributed support or geometric constraint.

\subsubsection{Whole-Arm (Body) Manipulation} removes end-effector–centric simplification by exploiting the full robot geometry for contact-rich interaction. This significantly increases the range of achievable manipulation behaviors yet introduces multi-contact planning problems with substantially more complex contact dynamics \citep{king2015nonprehensile, barreiros2025learning, leve2025scaling}, even under quasi-static or quasi-dynamic assumptions.

To address this complexity, several works adopt reinforcement learning (RL) approaches that learn reactive whole-body or whole-arm manipulation policies through extensive offline simulation \citep{zhang2023plan, barreiros2025learning}. Once trained, such policies can provide fast feedback control for specific tasks and contact conditions. However, RL-based approaches typically require large amounts of task-specific data and often generalize poorly beyond the training distribution. In practice, real-world manipulation involves diverse object geometries, physical properties, and contact conditions, making it difficult to train policies that generalize robustly across tasks. From a planning perspective, these methods also shift the burden of handling contact complexity to the learning phase, rather than addressing it directly within the optimization problem.

An alternative line of work seeks to recover similar performance through improved model-based optimization by smoothing contact models and combining them with global planning techniques \citep{pang2023global, suh2025dexterous}. While these approaches enable gradient-based optimization for whole-arm manipulation with less computation time, the required contact smoothing and quasi-dynamic approximations inevitably introduce modeling errors. When execution dynamics deviate from the planning model, such errors can lead to task failure \citep{pang2023global}, highlighting the need for robust planning objectives that explicitly tolerate contact-model mismatch.

\subsubsection{Robust Planning}
Robustness is essential for addressing modeling errors introduced by contact simplifications, as well as uncertainty and disturbances encountered in real-world manipulation. Prior work has explored robustness through belief-space planning \citep{jankowski2025robust, wang2025caging} and worst-case formulations \citep{ogunmolu2018minimax, chen2025adversarial}. In parallel, the RL community has studied robustness via domain randomization \citep{tobin2017domain, muratore2018domain}, online adaptation \citep{bousmalis2018using, arndt2020meta, xue2024robust}, and game-theoretic formulations \citep{kontoudis2019robust, song2022robust, liang2023game, shi2024rethinking}.

In model-based contact-rich manipulation planning, recent work has investigated belief-space formulations that reduce object pose uncertainty either explicitly \citep{jankowski2025robust} or implicitly via caging in time \citep{wang2025caging}. These methods provide important mechanisms for handling uncertainty, but they typically focus on uncertainty in initial object state or perception and are often developed for end-effector-centric manipulation. In contrast, the failure mode studied in this paper is due to contact-model mismatch during whole-arm manipulation, where errors arise from multi-contact geometry, unmodeled mode transitions, and shape simplification. Such mismatch is difficult to represent exhaustively as a belief over initial states or a fixed distribution over physical parameters. 

Our work addresses this complementary setting by exploring whole-arm caging manipulation, where a learned physics-informed eikonal caging metric biases planning toward enclosing robot geometries that make object escape difficult. Rather than replacing belief-space planning, the proposed escape-time field provides a geometric robustness objective that can be embedded into contact-rich planning. Empirically, we show that this objective improves tolerance to contact-model mismatch in open-loop execution with simplified contact models. In this way, exploiting whole-arm contact geometry both expands the range of achievable manipulation behaviors beyond end-effector-centric strategies and provides a mechanism for reducing dependence on highly accurate contact dynamics during planning.

\section{Eikonal Caging}\label{sec:eikonal_caging}
Our formulation begins from the observation that caging is inherently a worst-case phenomenon: an (adversarial) object is considered caged not because it is close to the robot, but because any admissible motion that would allow it to escape requires infinite time. We therefore take the escape-time function as the primary object of interest. This function represents the minimum time required for such an adversarial object to reach the boundary of the escape goal region, given a fixed robot configuration.

Crucially, by treating escape-time as the quantity to be optimized, robustness to contact-model mismatch is made explicit. Any trajectory that maintains a large escape-time margin is inherently tolerant to disturbances, modeling errors, and unmodeled contacts, because such perturbations must first reduce this margin before an escape can occur. The caging metric introduced in this work should thus be understood not as a heuristic objective, but as a representation of this escape-time margin that can be directly incorporated into trajectory optimization for whole-arm manipulation.

\subsection{Preliminary on Caging Theory}
\begin{figure}[!t]
    \centering
    \begin{subfigure}[t]{0.48\linewidth}
        \centering
        \includegraphics[width=\linewidth, trim=11cm 4.5cm 11cm 4.5cm, clip]{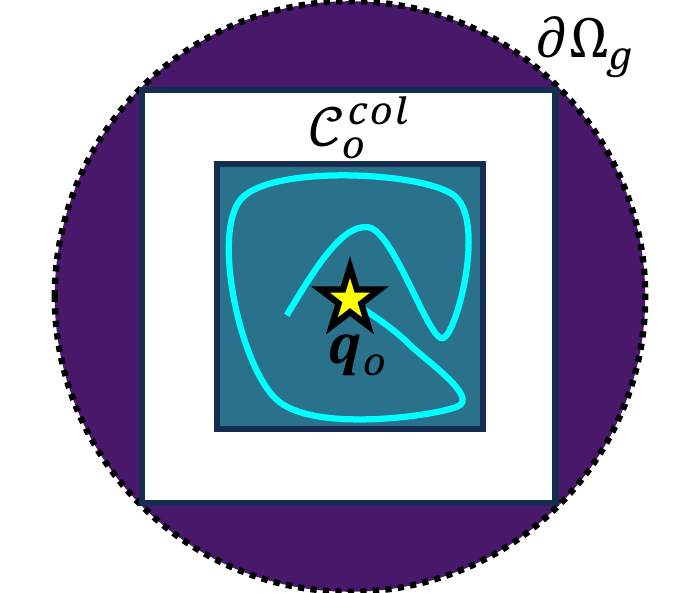}
        \caption{Caging}
        \label{subfig:caging}
    \end{subfigure}
    \hfill
    \begin{subfigure}[t]{0.48\linewidth}
        \centering
        \includegraphics[width=\linewidth, trim=11cm 4.5cm 11cm 4.5cm, clip]{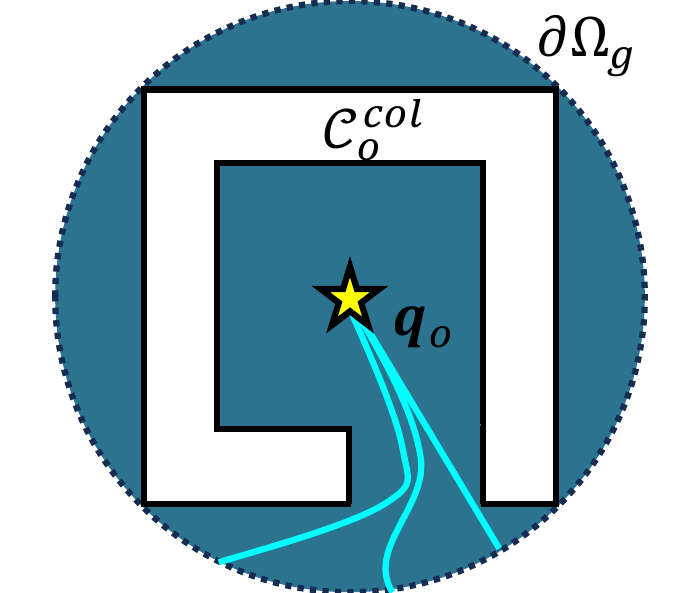}
        \caption{Partial Caging}
        \label{subfig:partial_caging}
    \end{subfigure}
    \caption{Illustrations of caging and partial caging with a point-mass object (yellow star) in its configuration space $\mathbb{R}^2$. The white region indicates environment obstacles, corresponding to the collision set $\mathcal{C}^{\mathrm{col}}_o$. The escape region consists of configurations outside the boundary $\partial \Omega_g$, shown as a dashed circle. Cyan curves indicate possible collision-free paths of the object: in the caging case they remain trapped, whereas in the partial-caging case they reach the escape region $\Omega_g$.} % once pass through $\partial \Omega_g$
    \label{fig:caging_comparison}
    \vspace{-0.5cm}
\end{figure}
\paragraph{Complete Caging:} We say that an environment cages an object if it confines the object’s motion to a bounded region of configuration space, so that, without penetrating with the environment, the object cannot move arbitrarily far away, as shown in Figure~\ref{subfig:caging}.

\paragraph{Partial Caging:} The geometric definition above induces a binary notion of caging: for a given configuration $\bm{q}_{o}$, the environment either cages the object or it does not. In practical grasping and manipulation settings, however, it is often more informative to distinguish how hard it is for an object to escape, even when escape is still possible. This motivates the notion of partial caging, in which the environment does not fully prevent escape, but imposes a non-trivial escape effort $J \geq 0$ that can be quantified by an appropriate cost functional. For example, one possible choice is the length of the object's escaping paths connecting its current configuration to an escape region in configuration space, as illustrated by the cyan curves in Figure~\ref{subfig:partial_caging}.

\subsection{Mathematical Formulation: Eikonal Caging}\label{subsec:eikonal_caging}
\paragraph{Caging as a minimum-time optimal control problem.}
In this work, we instantiate the abstract escape cost $J$ as the \emph{minimum time} $\mathcal{T}$ required for an object moving with bounded velocity to reach the escape region $\Omega_g$ along a collision-free trajectory. We model the object as a \emph{point in configuration space} whose motion is governed by a bounded velocity control. Specifically, the object follows a trajectory $\{\bm q_o^t\}_{t=0}^{T-1}$ driven by control inputs $\bm u_o^t$, subject to the velocity constraint $\|\bm u_o^t\|\le u_{\max}$. The object must remain within the collision-free configuration space $\mathcal{C}^{\mathrm{free}}$ and eventually reach $\Omega_g$ in minimal time, corresponding to the shortest escape path (e.g., the straight cyan escape path shown in Figure~\ref{subfig:partial_caging}). 

The resulting minimum-time escape problem can therefore be formulated as the following optimal control problem:
\begin{equation}\label{eq:mini_time_caging}
\begin{aligned}
    \min_{T, \{\bm{u}_o^t\}_{t=1}^{T-1}} \quad T \\
    \text{subject to} \quad 
    & \bm{q}_{o}^{t+1} = \bm{q}_{o}^t + \bm{u}_o^t \\
    & \dot{\bm q}_{o}^t=\bm u_o^t, \qquad \|\bm u_o^t\|\le u_{\max}, \\
    & \bm q_{o}^t \in \mathcal{C}^{\mathrm{free}} \quad \text{for all } t\in[0,T-1], \\
    & \bm q_{o}(0)=\bm q_0, \qquad \bm q_{o}^{T-1}\in \Omega_{g}, \\
\end{aligned}
\end{equation}

We denote the optimal value of~Equation~\eqref{eq:mini_time_caging} by $\mathcal{T}(\bm q_{o}^0)$. This
scalar function plays a dual role in our framework. First, it recovers complete caging: if no admissible bounded-velocity trajectory can reach $\Omega_g$ while satisfying $\bm q_{o}^{t} \in \mathcal{C}^{\mathrm{free}}$, then $\mathcal{T}(\bm q_{o}^0) = +\infty$, and the environment cages the object at $\bm q_{o}^0$ in the sense of the geometric definition. Second, whenever $0 < \mathcal{T}(\bm q_{o}^0) < +\infty$, this finite optimal time realizes an escape-time--based partial caging metric, directly instantiating $J^\star(\bm q_{o}^0)$ in the construction above. In other words, function $\mathcal{T}$ over all object configurations provides a unified description of complete and partial caging: the set $\{\bm q | \mathcal{T}(\bm q) = +\infty\}$ is the classical cage region, while the finite level sets $\{\bm q | \mathcal{T}(\bm q) \geq \alpha\}$ for $\alpha > 0$ represent graded degrees of partial caging around this region.

\subsection{Eikonal-Time Field and Eikonal Relation}
\paragraph{$\mathcal{T}(\bm q_{o}^0)$ as a scaled geodesic distance.}
Under the bounded-velocity kinematic model $\dot{\bm q}_{o}^{t} = \bm u_o^{t}$ with $\|\bm u_o^{t}\|\le u_{\max}$, the minimal escape time $\mathcal{T}(\bm q_{o}^0)$ in~Equation~\eqref{eq:mini_time_caging} corresponds to the geodesic (shortest-path) distance from $\bm q_{o}^0$ to $\Omega_g$ in the object’s collision-free configuration space, scaled by the maximal speed by $1/u_{\max}$.

First, any admissible trajectory $\bm q_{o}(\cdot)$ from
$\bm q_{o}^0$ to $\Omega_g$ over a time horizon $[0,T-1]$ satisfies
\[
    T \;\ge\; \frac{1}{u_{\max}}
    \int_0^T \|\dot{\bm q}_{o}^{t}\| \,\mathrm{d}t,
\]
If the target object moves at maximal speed ($\|\bm u_o^{t}\| = u_{\max}$) almost everywhere along an admissible trajectory, then the inequality becomes an equality and we obtain
\[
    T \;=\; \frac{1}{u_{\max}}
    \int_0^T \|\dot{\bm q}_{o}^{t}\| \,\mathrm{d}t
    \;=\;
    \frac{1}{u_{\max}} \,\mathrm{Length}\big(\{\bm q_{o}^t\}_{t=1}^{T-1}\big),
\]
where $\mathrm{Length}(\bm q_{o}(\cdot))$ denotes the arc length of the trajectory in configuration space. Thus, minimizing the time $T$ in~Equation~\eqref{eq:mini_time_caging} is equivalent to minimizing the path length over all admissible collision-free trajectories from $\bm q_{o}^0$ to $\Omega_g$, and therefore equivalent to computing the free-space geodesic distance $d_{\mathcal{C}^{\mathrm{free}}}$ from $\bm q_{o}^0$ to the escape region $\Omega_g$:
\[
    \mathcal{T}(\bm q_{o}^0)
    \;=\;
    \frac{1}{u_{\max}} \,
    d_{\mathcal{C}^{\mathrm{free}}}(\bm q_{o}^0, \Omega_g).
\]
In this sense, $\mathcal{T}(\bm q_{o}^0)$ is simply the geodesic distance in
the free configuration space scaled by $1/u_{max}$.

\paragraph{Eikonal characterization of the escape time.} Let $\mathcal{C}^{\mathrm{free}} \subset \mathbb{R}^n$ denote the free configuration space, and assume that $d_{\mathcal{C}^{\mathrm{free}}}(\cdot,\Omega_g)$ is finite on the connected component of $\mathcal{C}^{\mathrm{free}}$ containing $\Omega_g$. The escape-time field is then defined as
\[
    \mathcal{T}(\bm q)
    =
    \frac{1}{u_{\max}} d_{\mathcal{C}^{\mathrm{free}}}(\bm q, \Omega_g).
\]
It is proven that $\mathcal{T}$ is the unique viscosity solution of the eikonal boundary-value problem~\citep{clawson2014causal}
\begin{equation}
    \|\nabla \mathcal{T}(\bm q)\|
    \;=\;
    \frac{1}{u_{\max}}
    \qquad \text{for } \bm q \in \mathcal{C}^{\mathrm{free}} \setminus \Omega_g,
    \label{eq:eikonal_caging}
\end{equation}
with boundary condition $\mathcal{T}(\bm q)=0$ for all $\bm q \in \Omega_g$, and the convention $\mathcal{T}(\bm q)=+\infty$ for configurations that cannot reach $\Omega_g$ within $\mathcal{C}^{\mathrm{free}}$. 

Eikonal equation~\eqref{eq:eikonal_caging} admits a natural wavefront interpretation: it describes a wave that originates on the escape region $\Omega_g$ and propagates through the collision-free configuration space $\mathcal{C}^{\mathrm{free}}$ with constant speed $u_{\max}$, while propagation is blocked in regions that are not collision-free. The value $\mathcal{T}(\bm q)$ at any configuration $\bm q$ is precisely the arrival time of this wavefront, which is equivalent to the scaled geodesic distance $d_{\mathcal{C}^{\textrm{free}}}(\Omega_g, \bm{q})$ from $\Omega_g$ to $\bm q$ in $\mathcal{C}^{\mathrm{free}}$.

As illustrated by white contours in Figure~\ref{subfig:geodesic_dist}, the wavefront originates from the region outside the escape boundary $\partial \Omega_g$ (the dashed circle) and propagates inward through the collision-free space. Within collision-free regions, the wave advances at constant speed $u_{\max}$ (equals to $1$ in the visualization), while propagation is blocked when the wave encounters environmental obstacles (white area). The white contours indicate equal arrival times of the propagating wavefront, meaning that all points on the same contour share the same reaching time. Regions corresponding to obstacles are assigned zero propagation speed and therefore do not participate in the wave propagation. The corresponding spatial speed map used to compute the escape-time field is shown in Figure~\ref{subfig:speed_map}.

\begin{figure}[!t]
    \centering
    \begin{subfigure}{0.48\linewidth}
        \centering
        \includegraphics[width=\linewidth]{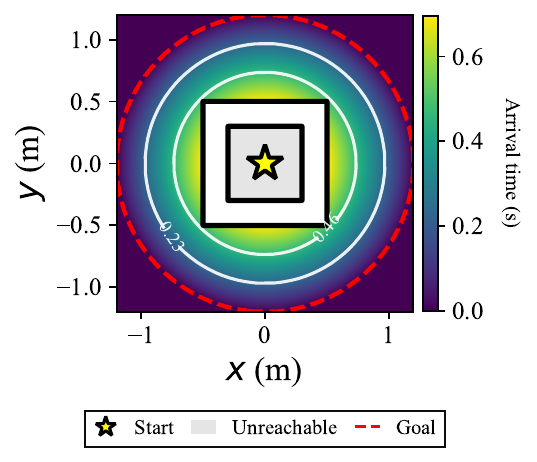}
        \caption{Time field $\mathcal{T}$.}
        \label{subfig:geodesic_dist}
    \end{subfigure}
    \begin{subfigure}{0.48\linewidth}
        \centering
        \includegraphics[width=\linewidth]{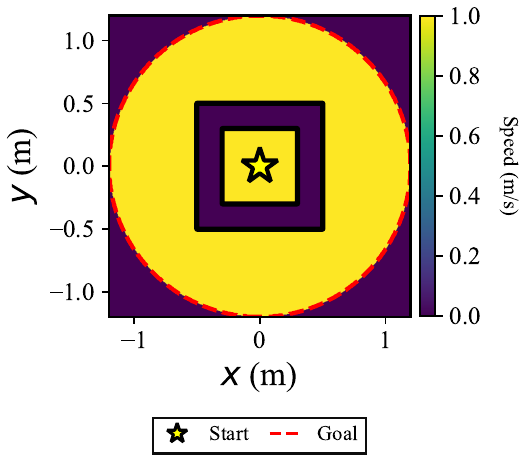}
        \caption{Speed map $u_{\max}$.}
        \label{subfig:speed_map}
    \end{subfigure}
    \caption{(a): Eikonal caging in the configuration space of a point mass (yellow star) interacting with a square-annulus obstacle (white area). The point mass lies inside the inner cavity (grey area), while the goal region is outside the red dashed boundary; No collision-free path exists from the cavity to the goal, so this region forms a cage in configuration space. Contours with values indicate the minimal escape time required for the object on that contour. (b): Speed map used for eikonal caging analysis. The yellow region indicates collision-free configurations with constant speed $u_{max}=1$; the purple square-annulus marks the collision space, where the local speed is zero. We also set the speed to zero outside the workspace boundary, so arrival time there are ignored.}
    \label{fig:eikonal_caging_example}
    \vspace{-0.5cm}
\end{figure}

\paragraph{From Eikonal equation to Eikonal caging.} 
In this article, we assume that the free-space distance $d_{\mathcal{C}^{\mathrm{free}}}$ between two configurations $\bm q_1, \bm q_2$ is symmetric,
\[
    d_{\mathcal{C}^{\mathrm{free}}}(\bm q_1,\bm q_2)
    \;=\;
    d_{\mathcal{C}^{\mathrm{free}}}(\bm q_2,\bm q_1)
    \qquad \forall\, \bm q_1,\bm q_2 \in \mathcal{C}^{\mathrm{free}},
\]
which corresponds physically to assuming that the object’s velocity in configuration space in Equation~\eqref{eq:eikonal_caging} is reversible during the caging analysis. Under this symmetry, the geodesic distance between any configuration $\bm q$ and the escape set $\Omega_g$ can be viewed either as the cost of moving from $\bm q$ to $\Omega_g$ and as the cost of moving from $\Omega_g$ to $\bm q$:
\[
    d_{\mathcal{C}^{\mathrm{free}}}(\bm q,\Omega_g)
    \;=\;
    \inf_{\bm p \in \Omega_g}
    d_{\mathcal{C}^{\mathrm{free}}}(\bm q,\bm p)
    \;=\;
    \inf_{\bm p \in \Omega_g}
    d_{\mathcal{C}^{\mathrm{free}}}(\bm p,\bm q).
\]

Combining this symmetry with the definition of $\mathcal{T}$ and the eikonal characterization in~Equation~\eqref{eq:eikonal_caging}, we obtain a simple interpretation: the solution of the Eikonal equation is simultaneously (i) the arrival time at $\bm q$ of a wavefront propagating from $\Omega_g$ at speed $u_{\max}$, and (ii) the minimal time required for the object, starting from $\bm q$, to reach $\Omega_g$ along a collision-free trajectory under the bounded-velocity model. As a consequence, solving a single eikonal problem in~Equation~\eqref{eq:eikonal_caging} with boundary condition $\mathcal{T}=0$ on $\Omega_g$ yields the escape-time field $\mathcal{T}(\bm q)$ of the object in \emph{all configurations} in $\mathcal{C}^{\mathrm{free}}$. 

This is what we refer to as eikonal caging: we formulate caging and partial caging in terms of the solution of an eikonal equation on the object’s collision-free configuration space, and we exploit the resulting escape-time field both as a theoretical bridge between geometric and escape-based caging and as a practical potential for caging-aware manipulation planning.

\subsection{Illustrative Example on Eikonal Caging}\label{subsec:toy_example_eikonal_caging}

To illustrate the proposed eikonal caging metric, we consider the two-dimensional example in Figure~\ref{fig:eikonal_caging_example}. The object is modeled as a point mass (yellow star) moving in the plane, while the white square-annulus represents an obstacle. The red dashed circle denotes the escape boundary $\partial \Omega_g$, and configurations outside it belong to the escape region $\Omega_g$. The point mass is initially located in the inner cavity (grey region) of the annulus. Since any collision-free path from this cavity to $\Omega_g$ must pass through the obstacle, no feasible escape path exists. Under the classical geometric definition, the inner cavity is therefore a cage region.

Figure~\ref{fig:eikonal_caging_example} (right) shows the speed map used in Equation~\eqref{eq:eikonal_caging}. In collision-free space (yellow), the object moves with constant bounded speed $u_{\max}=1$. Inside the obstacle (purple), the speed is set to zero, preventing propagation through collisions. The speed is also set to zero outside the workspace boundary, excluding those configurations from the computation. Together with the boundary condition $\mathcal{T}(\bm q_o)=0$ on $\Omega_g$, this fully defines the escape-time field.

The resulting solution $\mathcal{T}(\bm q_o)$ assigns to each collision-free configuration the minimum time required to reach the escape region. Configurations outside the annulus have finite values that increase as escape routes become longer or narrower, as shown by the level-set contours in Figure~\ref{subfig:geodesic_dist}. In contrast, the wavefront cannot enter the inner cavity, so the corresponding escape time is effectively $+\infty$. This unreachable region coincides exactly with the classical cage set.

This example highlights two useful properties of the proposed formulation. First, classical caging appears as the special case of infinite escape time. Second, configurations that are not fully caged still receive finite positive values, providing a continuous measure of partial caging and escape difficulty. This graded structure will later allow caging to be incorporated directly into optimization-based planning.

For visualization, the field $\mathcal{T}(\bm q_o)$ is computed on a uniform grid using the Fast Marching Method (FMM) \citep{sethian1996fast, mirebeau2014anisotropic, sanguinetti2015sub, chen2016new}, which provides an efficient numerical approximation of Equation~\eqref{eq:eikonal_caging}.

\section{Whole-Arm Caging Configuration Planning}\label{sec:wacaging_conf}
\subsection{Mathematical Formulation}
Classical caging models often rely on simplified geometries such as point fingers or circular grippers. Here, we instead treat the full articulated robot arm as a reconfigurable caging geometry and seek a joint configuration that maximizes the object's escape difficulty. Given the current object configuration $\bm q_o^0$, the robot is modeled as an articulated obstacle parameterized by its joint configuration $\bm q_r$.

We formulate whole-arm caging configuration planning as
\begin{equation}\label{eq:wm_caging}
\begin{aligned}
    \max_{\bm{q}_{r}} \;
    & \min_{T, \{\bm{u}_o^{t}\}_{t=0}^{T-1}} \quad T \\
    \text{s.t} \quad 
    & \dot{\bm q}_{o}^t=\bm u_o^t, \qquad \|\bm u_o^t\|\le u_{\max}, \\
    & \phi_{r}\!\left(\bm q_{o}^t, \bm{q}_{r}\right)\ge 0 \quad \text{for all } t\in[0,T-1], \\
    & \bm q_{o}^{t=0}=\bm{q}_{o}^0, \qquad \bm q_{o}^{T-1}\in \Omega_g, \\
    & \bm q_{r}^{\text{lw}} \leq \bm q_{r} \leq \bm q_{r}^{\text{up}}, \\
\end{aligned}
\end{equation}
where $\phi_{r}\!\left(\bm q_{o}^{t}, \bm{q}_{r}\right) \geq 0$ encodes the collision-free configurations $\mathcal{C}^{\mathrm{free}}_{o, r}$ (written as $\mathcal{C}^{\textrm{free}}$ for short in the following) in the joint configuration space of the robot and target object; $\Omega_g$ indicates the escape region and corresponds to the region out of the arm's workspace; $[\bm q_{r}^{\text{lw}}, \bm q_{r}^{\text{up}}]$ are the robot's joint limits.

For any fixed robot configuration $\bm q_r$, the inner minimization is exactly the minimum-time escape problem introduced in Section~\ref{subsec:eikonal_caging}. Its optimal value defines the escape-time field $\mathcal T(\bm q_o^0;\bm q_r)$,which measures the shortest time for the object to reach $\Omega_g$ under bounded motion while remaining collision-free. The outer optimization therefore selects the robot configuration that maximizes this escape time:
\begin{equation}\label{eq:wm_caging_max}
\bm q_r^* \in \arg\max_{\bm q_r}\; \mathcal T(\bm q_o^0;\bm q_r).
\end{equation}

\emph{Game-theoretic interpretation:} Equation \eqref{eq:wm_caging} admits a natural min--max interpretation: the object seeks the fastest feasible escape trajectory, while the robot chooses a configuration that maximizes the corresponding worst-case escape time. In this sense, the resulting configuration is geometrically robust to adversarial bounded object motion.

The next two subsections describe how we represent the collision-free boundary $\phi_r$ and how we approximate $\mathcal T(\bm q_o^0;\bm q_r)$ with a differentiable physics-informed model for efficient optimization.

\subsection{Collision-Free Boundary Representation}\label{subsec:rdf}
We represent the robot--object collision boundary $\phi_r(\bm q_o,\bm q_r)$ using the kinematics-aware robot signed distance field (RDF) introduced in \citep{li2024representing}.
The RDF models each robot link with a continuous signed distance field in its local frame, while forward kinematics are used at inference time to evaluate the signed distance $f_r(\bm p,\bm q_r)$ from any workspace query point $\bm p$ to the robot surface at configuration $\bm q_r$.

To handle target objects of arbitrary geometry, we represent the object by a set of $K$ surface samples rigidly attached to its configuration, $\mathcal P(\bm q_o)=\{\bm p_i(\bm q_o)\}_{i=1}^K $. We then define the collision-free boundary function as
\begin{equation}
\phi_r(\bm q_o,\bm q_r)=\min_{\bm p\in\mathcal P(\bm q_o)} f_r(\bm p,\bm q_r).
\end{equation}

Hence, $\phi_r(\bm q_o,\bm q_r)\ge0$ implies that all sampled object points lie outside or on the robot surface, and the corresponding configuration is collision-free. The zero level set $\phi_r=0$ therefore implicitly defines the robot--object contact boundary.

This representation is efficient to evaluate in batch, accommodates arbitrary articulated robot geometries, and naturally supports gradient-based optimization in the following subsections.

\subsection{Physics-Informed Eikonal Caging}\label{subsec:piec}

The escape-time field $\mathcal T(\bm q_o;\bm q_r)$ depends jointly on the robot configuration and object geometry, since both determine the collision-free configuration space. To handle objects of arbitrary shape efficiently, we approximate the object-level escape time by aggregating escape times of a set of key points attached to the object. For key points $\{\bm p_i\}_{i=1}^K$, we define
\begin{equation}
\hat{\mathcal T}(\bm q_o;\bm q_r)
=
\frac{1}{K}\sum_{i=1}^{K}\mathcal T_p(\bm p_i,\bm q_r),
\label{eq:object_time_average}
\end{equation}
where $\mathcal T_p(\bm p_i,\bm q_r)$ denotes the escape time of point $\bm p_i$ under the caging geometry induced by $\bm q_r$. This reduces arbitrary object geometry to batched point queries of a shared field predictor.

We model $\mathcal T_p(\bm p,\bm q_r)$ using a physics-informed neural network (PINN) that directly enforces the eikonal equation. Unlike supervised regression, this learns the escape-time field without requiring precomputed labels or repeated numerical solution of the inner escape problem during training.

Specifically, we parameterize
\begin{equation}
\mathcal T_p(\bm p,\bm q_r)
=
\sigma\!\big(t_\theta(\bm p,\bm q_r)\big),
\label{eq:pinn_factorization}
\end{equation}
where $t_\theta$ is a multilayer perceptron and $\sigma(\cdot)$ is a nonnegative activation ensuring $\mathcal T_p\ge0$.

The network parameters $\theta$ are optimized using the loss
\begin{equation}
\mathcal L_{\mathrm{eik}}
=
\lambda_{\mathrm{pde}}\mathcal L_{\mathrm{pde}}
+
\lambda_{\mathrm{bc}}\mathcal L_{\mathrm{bc}},
\end{equation}
where
\[
\mathcal L_{\mathrm{pde}}
=
\Big(
\|\nabla_{\bm p}\mathcal T_p(\bm p,\bm q_r)\|
-
\frac{1}{u_{\max}}
\Big)^2
\]
enforces the eikonal equation, and
\[
\mathcal L_{\mathrm{bc}}
=
\mathbb E_{\bm p\in\Omega_g}
\big[\mathcal T_p(\bm p,\bm q_r)^2\big]
\]
enforces the boundary condition $\mathcal T_p=0$ on the escape region $\Omega_g$.

The resulting model provides a smooth and differentiable surrogate for the inner minimum-time escape problem in Section~\ref{sec:wacaging_conf}. Consequently, caging objectives can be queried efficiently and optimized directly with respect to robot configurations and trajectories. Additional implementation details are provided in Appendix~\ref{app:pinn}.

\subsection{Theoretical Analysis}
In this subsection, we provide a brief theoretical justification that our point-based approximation $\hat{\mathcal{T}}(\bm{q}_{o}, \bm{q}_{r})$ is a conservative lower bound on the true rigid-body escape time $\mathcal{T}(\bm{q}_{o}, \bm{q}_{r})$.

\begin{proposition}\label{prop:point_cloud_lower_bound}
$\hat{\mathcal{T}}(\bm{q}_{o}, \bm{q}_{r})$ is the lower bound of $\mathcal{T}(\bm{q}_{o}, \bm{q}_{r})$ for an object with arbitrary shape.
\end{proposition}

\begin{proof}
    Take any admissible rigid-body escape trajectory $\bm{q}_{o}^{t}$ with escape time $T$, all collision and dynamics constraints are satisfied, and all sampled points $\bm{p}_{i}^{t}$ induced by the rigid motion satisfy $\bm{p}_i^{T-1} \in \Omega_g$. For each $i$, the trajectory $\bm{p}_i^{t}$ is a feasible point-mass escape trajectory for $\bm{p}_i^0$, so by minimality of $\mathcal{T}_p(\bm{p}_i; \bm{q}_{r})$ we have
    \[
        T \geq \mathcal{T}_{p}(\bm{p}_i; \bm{q}_{r}), \quad \forall \ i.
    \]
    Thus $T \geq \max_i \mathcal{T}_p(\bm{p}_i; \bm{q}_{r})$. Since this holds for any admissible rigid-body escape trajectory and its time $T$, it also holds for the maximum over all such trajectories. Therefore, we have 
    \[
        \mathcal{T}(\bm{q}_{o}; \bm{q}_{r}) \geq \max_{i} \mathcal{T}_p(\bm{p}_i;\bm{q}_{r}).
    \]
    Finally, all point-wise escape times are non-negative, so the maximum dominates the average: 
    \[
        \max_{i} \mathcal{T}_p(\bm{p}_i;\bm{q}_{r}) \geq \frac{1}{K} \sum_{i=1}^{K}
    \mathcal{T}_{p}(\bm p_i, \bm q_{r}) = \hat{\mathcal{T}}(\bm{q}_{o}; \bm{q}_{r}).
    \]
    Combing the two inequalities therefore yields the claimed lower bound: $\mathcal{T}(\bm{q}_{o}; \bm{q}_{r}) \geq \hat{\mathcal{T}}(\bm{q}_{o}; \bm{q}_{r})$. \qed
\end{proof}

\paragraph{Intuitive explanation:} Intuitively, a rigid object
is more constrained than its individual points: every feasible rigid-body escape motion induces feasible point-mass motions, but not conversely. Therefore, the object can never escape faster than its slowest point could in isolation. This yields a clean ordering: 1) the true object escape time is bounded below by the maximum point-wise escape time, 2) which in turn is bounded below by the average point-wise escape time we use in our approximation.

\paragraph{Practical lower bound selection: average over maximum.} From a theoretical perspective, the maximum over point-wise escape times is the tightest point-based lower bound on the true rigid-body escape time. However, in practice we choose to use the average $\hat{\mathcal{T}}$ as lower-bound approximation for two reasons. First, it yields a smoother and more stable quantity for optimization: all points contribute to the objective and its gradient, rather than only the single worst-case point, which improves numerical behavior when optimizing over robot configurations. Second, in multi-object scenes the average encourages the robot to increase the escape time of all objects collectively, rather than focusing exclusively on the single most vulnerable point. At the same time, by the Proposition above, $\hat{\mathcal{T}}$  remains a conservative lower bound on the true escape time, ensuring that large values of $\hat{\mathcal{T}}$ correspond to scenes that are at least as difficult to escape in the rigid-body sense.

\subsection{Whole-Arm Caging Configuration Optimization}
We now instantiate Equation~\eqref{eq:wm_caging} by replacing the exact escape time with its point-based approximation $\hat{\mathcal T}(\bm q_o^0;\bm q_r)$. This yields
\begin{equation}\label{eq:wm_caging_lb}
\begin{aligned}
\max_{\bm q_r}\quad
& \hat{\mathcal T}(\bm q_o^0;\bm q_r) \\
\text{s.t.}\quad
& \phi_r(\bm q_o^0,\bm q_r)\ge0, \\
& \bm q_r^{\mathrm{lw}} \le \bm q_r \le \bm q_r^{\mathrm{up}}.
\end{aligned}
\end{equation}

As shown in the previous subsection, $\hat{\mathcal T}$ is a conservative lower bound of the true rigid-body escape time $\mathcal T_{\mathrm{obj}}$. For planning, however, exact values are less important than ranking robot configurations by caging quality. Maximizing a smooth lower bound therefore remains an effective surrogate for increasing the true escape difficulty.

Because $\hat{\mathcal T}(\bm q_o^0;\bm q_r)$ is efficiently evaluated and differentiable with respect to $\bm q_r$, the problem can be solved using either gradient-free methods such as CMA-ES or gradient-based methods such as sequential quadratic programming (SQP).

The following example illustrates both the fidelity of the learned field and its use for whole-arm caging configuration optimization.

\subsection{Illustrative Example on Whole-Arm Caging Configuration Planning}
\label{subsec:example_wm_caging}

We consider a planar 4-DOF robot arm anchored at the origin with capsule links of lengths $[0.4,0.4,0.4,0.2]$ m and radius $0.05$ m. The workspace boundary, which defines the escape set $\Omega_g$, is modeled as a circle of radius $1.4$ m centered at the robot base. A point-mass object is placed inside the workspace, and the proposed neural eikonal model is used to approximate the escape-time field $\mathcal T(\bm p,\bm q_r)$ conditioned on robot configuration.

\begin{figure}[!t]
    \centering
    \includegraphics[width=\linewidth]{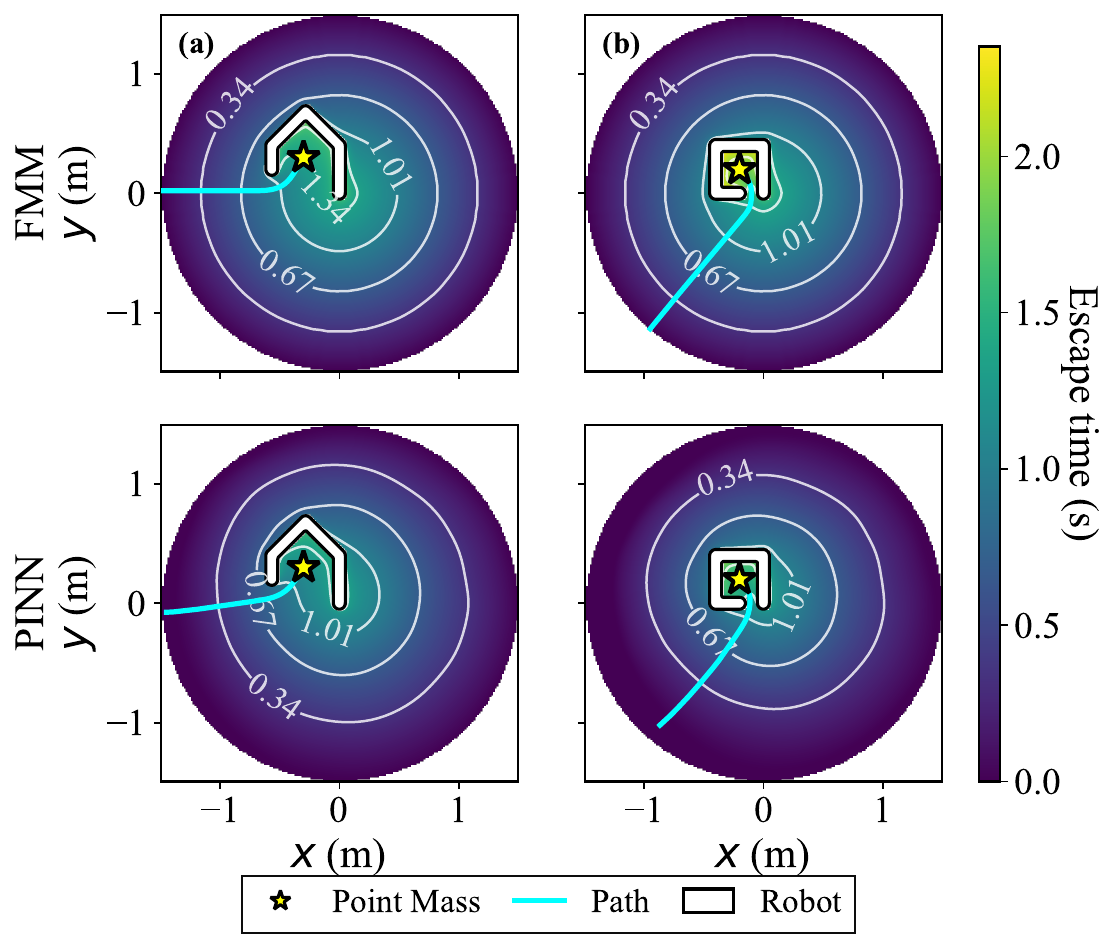}
    \caption{Comparison between the PINN-based eikonal caging time field and the ground-truth solution from FMM. Two representative robot configurations and point-mass positions are shown, illustrating that the learned model closely matches the true escape-time field and recovers the corresponding optimal escape paths.}
    \label{fig:pinn_vs_fmm}
    \vspace{-0.5cm}
\end{figure}
\begin{figure}[!t]
    \centering
    \includegraphics[width=\linewidth]{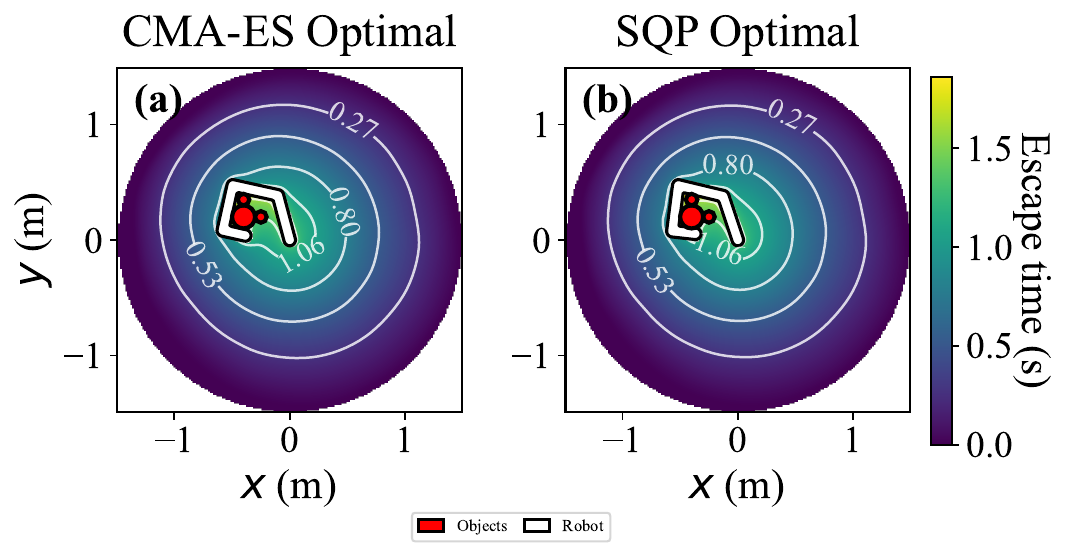}
    \caption{Optimal caging configurations found with CMA-ES and SQP, maximizing the learned eikonal caging metric of an object represented by three circles shown in red.}
    \label{fig:wm_caging_opt_cmaes_sqp}
    \vspace{-0.5cm}
\end{figure}

\paragraph{Field approximation accuracy.}
Figure~\ref{fig:pinn_vs_fmm} compares the learned field\textsuperscript{1} with ground-truth FMM solutions for two representative robot configurations. The predicted contours closely match the FMM solutions, and the induced steepest-descent escape paths are consistent with the corresponding shortest-time trajectories. This indicates that the learned model captures both the scalar field values and their geometric structure.
\footnotetext{Escape time values are scaled by $0.02$ for all plots in the following for rendering.}

\paragraph{Configuration optimization.}
We next optimize Equation~\eqref{eq:wm_caging_lb} for an object represented by three circles, using both CMA-ES and SQP from the same initial robot configuration. As shown in Figure~\ref{fig:wm_caging_opt_cmaes_sqp}, both optimizers converge to enclosing arm postures that wrap around the object and increase the predicted escape time. The final objective values are comparable, suggesting that the learned caging objective is numerically well behaved for both gradient-free and gradient-based optimization.

Together, these results show that the proposed model is sufficiently accurate for planning and can be used directly as an optimization objective for whole-arm caging configuration search. Larger-scale evaluations with realistic geometries are presented in Section~\ref{subsec:exp_wm_caging_conf}.

\section{Whole-Arm Caging Manipulation}\label{sec:wacaging_mani}
We now extend whole-arm caging from static configuration planning to dynamic manipulation. Rather than searching for a single enclosing posture, the robot must execute a trajectory that achieves the task objective while maximizing geometric enclosure of the object. This section incorporates the proposed caging metric into trajectory optimization and uses an illustrative example to compare the behaviors induced by different manipulation objectives. These comparisons motivate the robustness analyses in the following sections.

\subsection{Mathematical Formulation}\label{subsec:wm_cage_mani_form}
We formulate whole-arm caging manipulation as the finite-horizon optimization problem:
\begin{subequations}\label{eq:wm_caging_mani}
\begin{align}
    \min_{\{\bm{u}_{r}^t\}_{t=0}^{N-1}} \quad 
    & \sum_{t=0}^{N-1} 
      \bigl[
        l\big(\bm{q}_{o}^t, \bm{q}_{r}^t, \bm{u}_{r}^t\big)
        - \hat{\mathcal{T}}\big(\bm{q}_{o}^t, \bm{q}_{r}^t\big)
      \bigr],
      \label{eq:wm_caging_mani_obj} \\[2pt]
    \text{s.t.} \quad
    & \bm{q}^{t+1} = f\big(\bm{q}^{t}, \bm{u}_{r}^{t}\big),
      \label{eq:wm_caging_mani_dyn} \\
    & \phi_{r}\!\left(\bm q_{o}^t, \bm{q}_{r}^t\right)\ge 0 \quad t=0,\dots,N, 
    \label{eq:wm_caging_mani_col}\\
    & \bm q_{r}^{\text{lw}} \leq \bm q_{r}^t \leq \bm q_{r}^{\text{up}},
      \quad t=0,\dots,N.
      \label{eq:wm_caging_mani_bounds}
\end{align}
\end{subequations}
Here, $\bm q^t=[\bm q_o^t,\bm q_r^t]$ denotes the joint object--robot state, and $\bm u_r^t$ is the robot control input. The stage cost $l(\cdot)$ captures task objectives such as goal reaching, motion smoothness, or control effort. The caging reward $\hat{\mathcal T}(\bm q_o^t,\bm q_r^t)$ favors states from which the object is difficult to escape, thereby encouraging persistent whole-arm enclosure during manipulation.

The dynamics function $f(\bm q^t,\bm u_r^t)$ models the coupled robot--object evolution under contact. Following prior work \citep{anitescu2006optimization, pang2023global, jin2024complementarity}, we use an optimization-based quasi-dynamic contact model in which object and robot velocities at each step are obtained from a quadratic program. Additional details are provided in Section~\ref{subsec:contact_dynamics}.

\subsection{Two Interpretations of Our Planner}

Equation~\eqref{eq:wm_caging_mani} admits two complementary viewpoints that help explain the behavior of the proposed planner.

\paragraph{(i) Contact-rich trajectory optimization with a learned caging regularizer.}
From a standard planning perspective, Equation~\eqref{eq:wm_caging_mani} is a contact-rich manipulation optimizer augmented with the reward term $-\hat{\mathcal T}(\bm q_o^t,\bm q_r^t)$. This term favors robot states from which the object is difficult to escape, encouraging the arm to use its full geometry rather than relying only on the end-effector or a single link. Under this interpretation, our method is a whole-arm manipulation planner equipped with a geometry-aware cost shaping term.

\paragraph{(ii) Approximate robust planning against fast object escape.}
The same objective can also be interpreted as planning against a bounded adversarial escape model. At each state $(\bm q_o^t,\bm q_r^t)$, the quantity $\mathcal T(\bm q_o^t,\bm q_r^t)$ represents the minimum time required for an object with bounded velocity to reach region outside the robot's workspace. Maximizing this value therefore biases the planner toward more robust configurations that remain safe even under escaping object motion.

Replacing $\mathcal T$ with its learned approximation $\hat{\mathcal T}$ yields the practical optimization problem in Equation~\eqref{eq:wm_caging_mani}. In this sense, the planner implicitly accounts for worst-case escape behavior without explicitly solving an inner adversarial control problem online.

This viewpoint also helps explain the empirical robustness observed in Section~\ref{sec:robust_analyse}: states with large escape time provide additional geometric containment, making the object harder to lose under disturbances or simplified contact models.

\subsection{Contact Dynamics Modeling}\label{subsec:contact_dynamics}

Robot--object interaction is modeled using a standard optimization-based quasi-dynamic contact formulation. At each time step, given the current state $\bm q=[\bm q_o^\top,\bm q_r^\top]^\top$ and robot control input $\bm u_r$, we compute the generalized velocity $\bm v=\dot{\bm q}$ by solving a convex quadratic program (QP), and then update the configuration via $\bm q^{+}=\bm q \oplus h\,\bm v$, where $h$ is the time step and $\oplus$ denotes integration on the configuration manifold.

Specifically, we solve
\begin{equation}\label{eq:qp}
\begin{aligned}
\min_{\bm v}\quad
& \tfrac{1}{2}h^2 \bm v^\top \bm Q \bm v - h\,\bm v^\top \bm b(\bm u_r), \\
\text{s.t.}\quad
& (\bm J_i^n-\mu_i \bm J_{i,j}^t)\bm v + \tfrac{\phi_i}{h} \ge 0, \\
& i=1,\dots,n_c,\quad j=1,\dots,n_d,
\end{aligned}
\end{equation}
where $\bm Q$ and $\bm b(\bm u_r)$ encode inertial and control terms, $\phi_i$ is the signed distance at contact $i$, $\bm J_i^n$ and $\bm J_{i,j}^t$ are normal and tangential contact Jacobians, and $\mu_i$ is the friction coefficient \citep{anitescu2006optimization, pang2023global, jin2024complementarity}.

This first-order model is substantially simpler than full second-order rigid-body dynamics, while retaining the key geometry- and friction-dependent structure of contact interactions. In Equation~\eqref{eq:wm_caging_mani_dyn}, the dynamics function $f(\bm q^t,\bm u_r^t)$ is implemented by one QP solve followed by the integration step above.

Section~\ref{subsec:contact_dyn_simp} later introduces deliberate simplifications of this model to evaluate planner robustness. As shown there, trajectories optimized with the caging objective often remain effective even when the contact model is simplified.

\subsection{Illustrative Example on Whole-Arm Caging Manipulation}\label{subsec:caging_manipulation_example}
We use a controlled planar manipulation task to compare how different objective primitives shape whole-arm robot--object interaction. This example follows the first interpretation of our whole-arm caging manipulation planner and its main purpose is diagnostic rather than statistical: to visualize objective-induced behaviors before the robustness studies in Section~\ref{sec:robust_analyse}.

\begin{figure*}[!t]
    \centering
    \begin{subfigure}[t]{0.25\linewidth}  
        \centering
        \includegraphics[width=\linewidth]{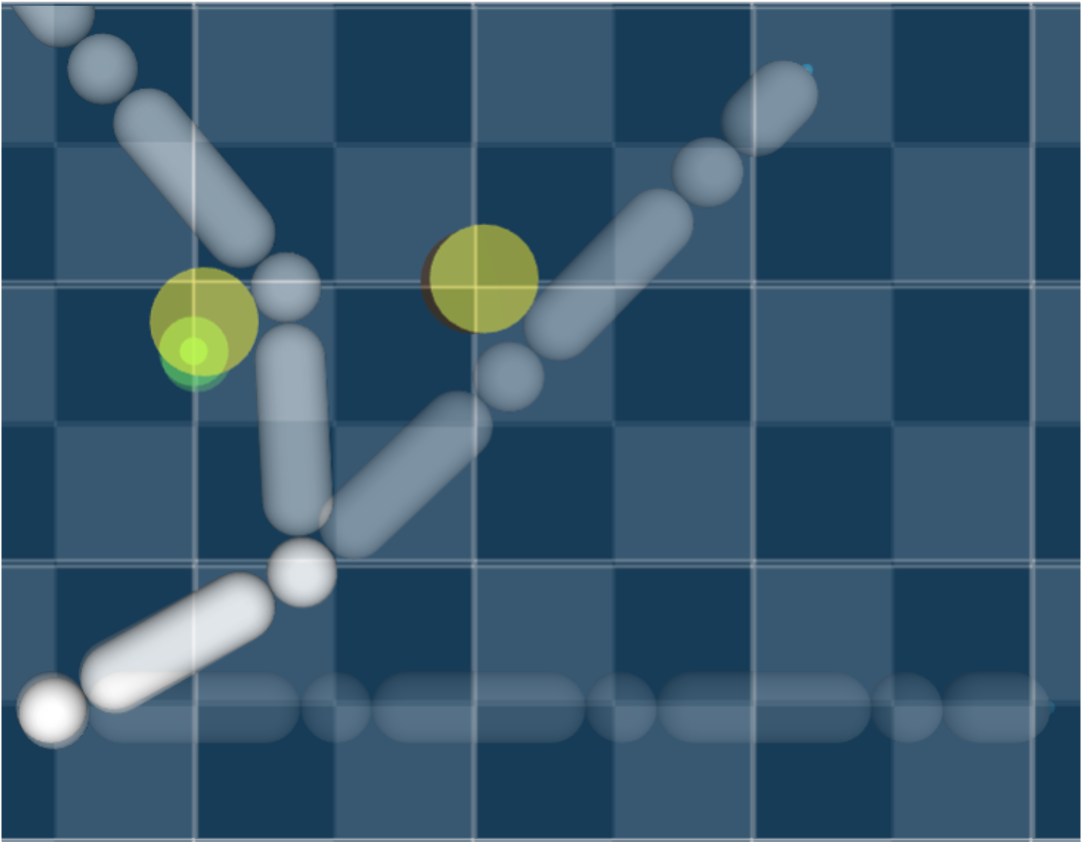}
        \caption{Mode 1 (Baseline 1)}
        \label{subfig:wm_cage_mani_toy_m1}
    \end{subfigure}
    \hspace{0.03\linewidth}
    \begin{subfigure}[t]{0.25\linewidth} 
        \centering
        \includegraphics[width=\linewidth]{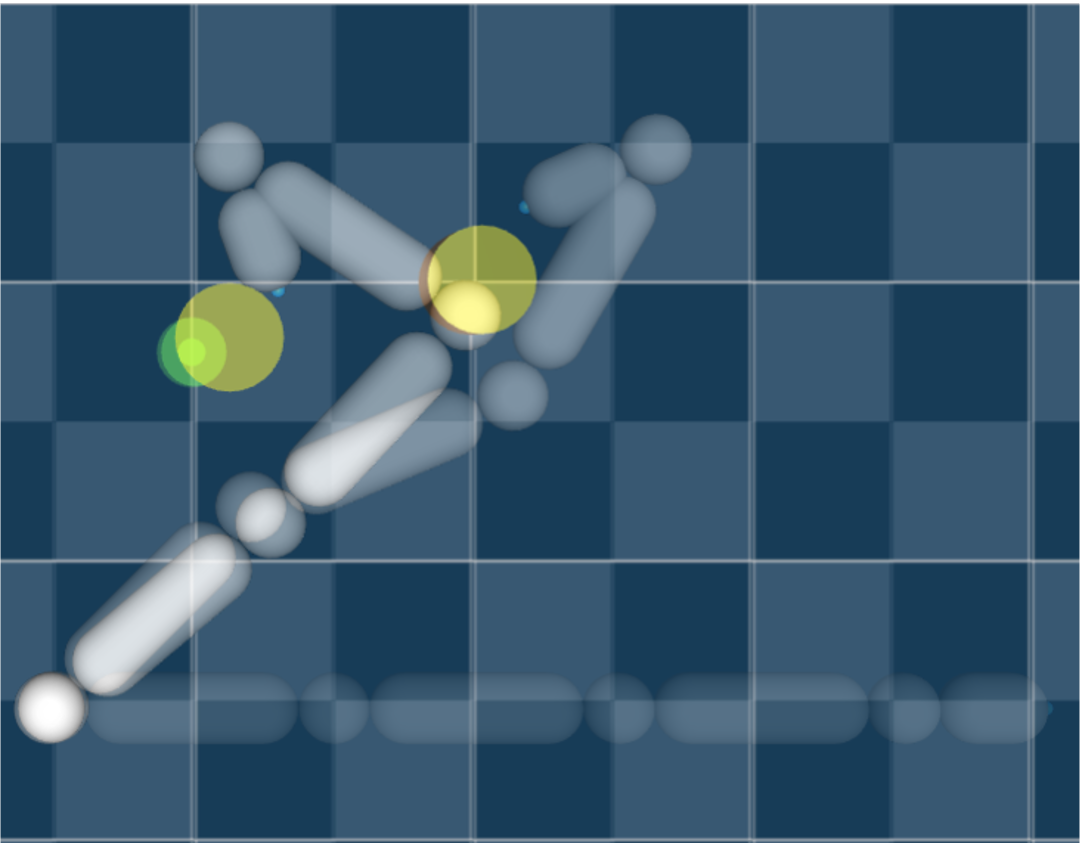}
        \caption{Mode 2 (Baseline 2)}
        \label{subfig:wm_cage_mani_toy_m2}
    \end{subfigure}
    \hspace{0.03\linewidth}
    \begin{subfigure}[t]{0.25\linewidth} 
        \centering
        \includegraphics[width=\linewidth]{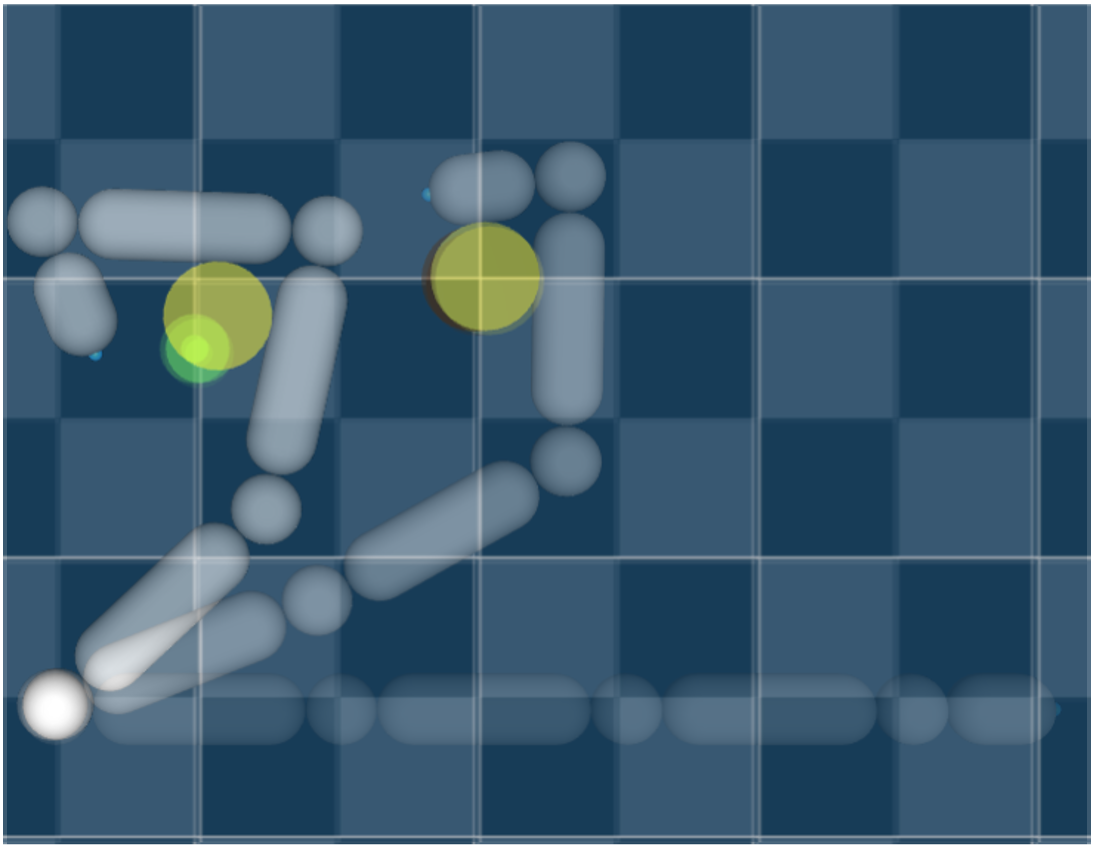}
        \caption{Mode 3 (Ours)}
        \label{subfig:wm_cage_mani_toy_m3}
    \end{subfigure}

    \caption{Whole-arm caging manipulation under three objective modes. Mode 1 minimizes the minimum distance between the object and the \emph{whole} robot arm; Mode 2 minimizes the average distance between the object and \emph{all robot links}; Mode 3 uses the learned escape-time field for whole-arm caging (our method). For each mode, we show the same initial robot configuration, the robot–object configuration at first contact, and the final configuration at the end of the manipulation. Keyframes are overlaid with transparency to illustrate the evolution of robot–object interaction under each objective. The robot (with spherical revolute-joint representation) is initialized at the zero configuration (horizontal in the images), and the object (yellow circle) is pushed toward the green goal region.}
   \label{fig:wm_cage_mani_toy}
   \vspace{-0.5cm}
\end{figure*}

\begin{figure}[!t]
    \centering
    \begin{subfigure}{\linewidth}  
        \centering
        \includegraphics[width=\linewidth]{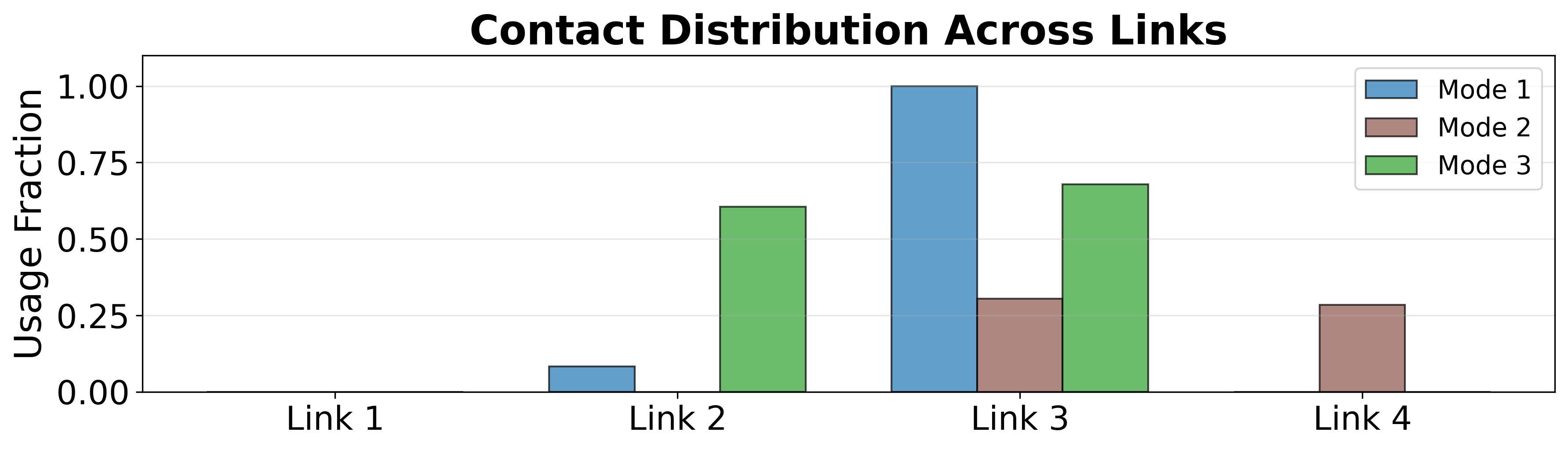}
        \caption{Fraction of time each arm link remains in contact with the object. Mode 1 relies primarily on a single mid-arm link (Link 3). Mode 2 shares proximity mainly between Links 3 and 4, while caging-aware manipulation (Mode 3) maintains sustained, overlapping contact across multiple links.}
        \label{subfig:arm_utilization}
    \end{subfigure}
    \begin{subfigure}[t]{0.48\linewidth} 
        \centering
        \includegraphics[width=\linewidth]{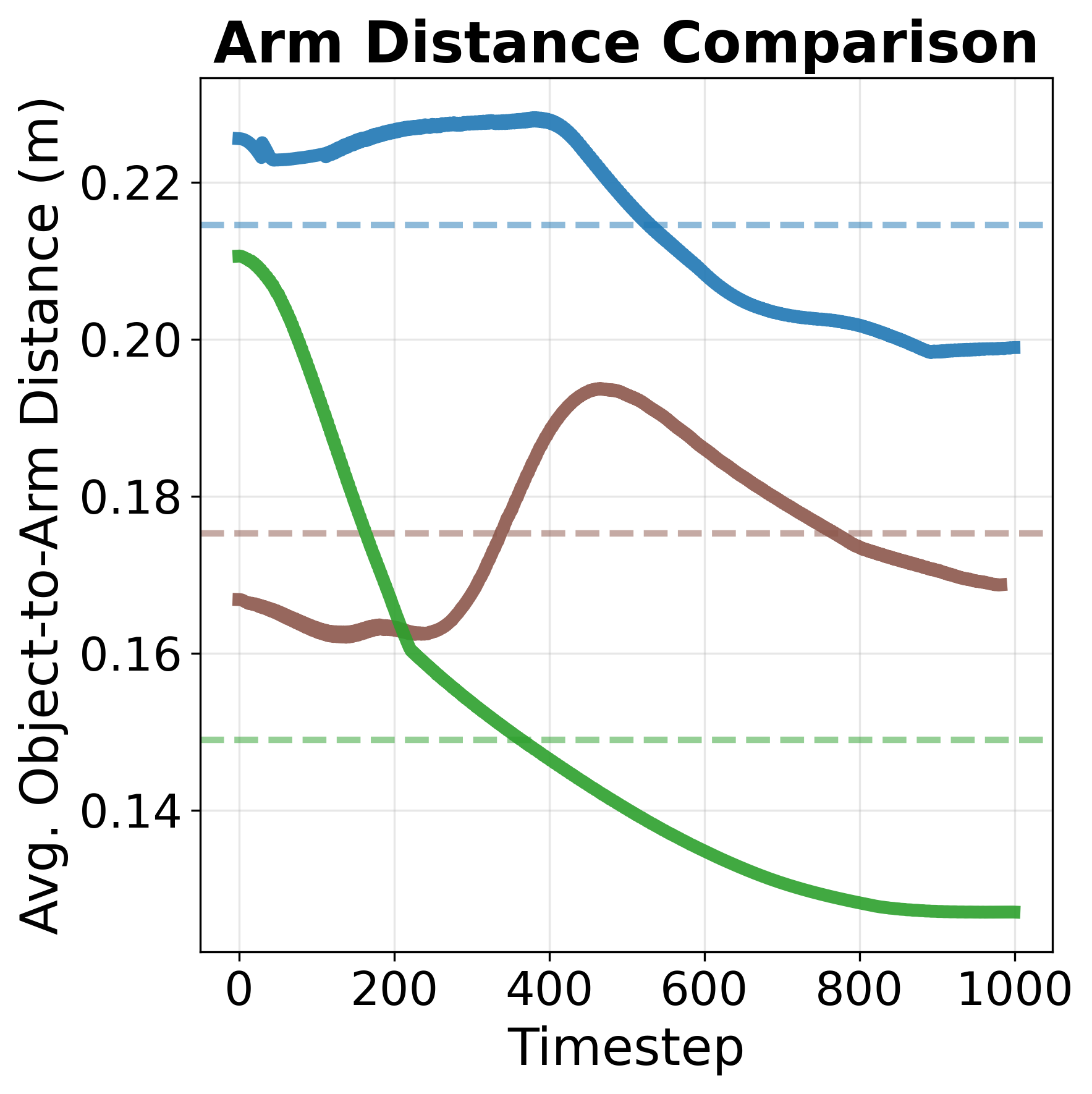}
        \caption{Average object–arm distance over time. Mode~2 rapidly reduces distance by early arm folding, followed by an increase during object transport. Mode~3 maintains larger distance initially and tightens gradually. Dash lines indicate mean values averaged over time.}
        \label{subfig:arm_dist_comp}
    \end{subfigure}
    \hfill
    \begin{subfigure}[t]{0.48\linewidth} 
        \centering
        \includegraphics[width=\linewidth]{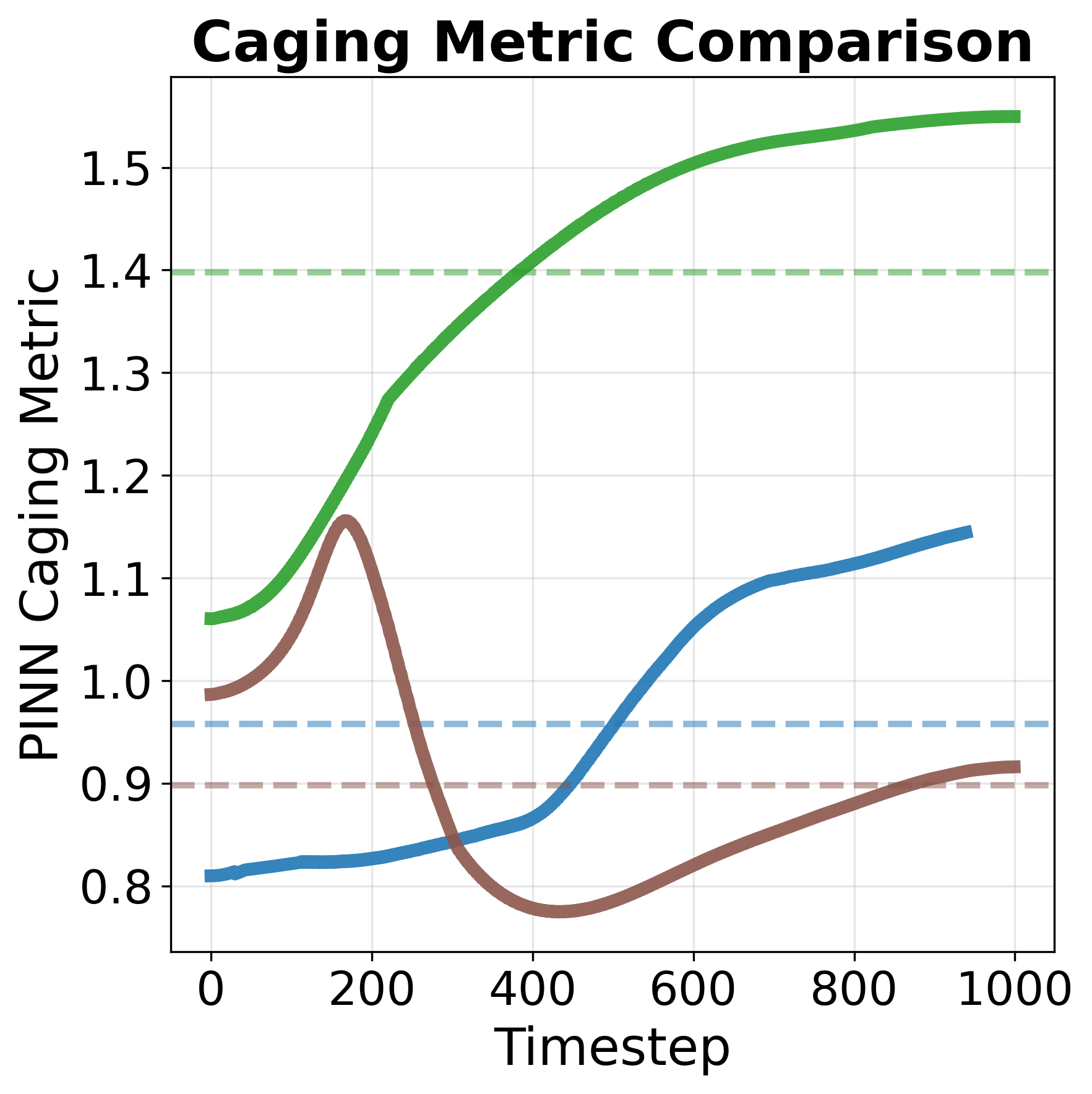}
        \caption{Escape-time (caging) metric over time. Mode~1 remains at low values throughout. Mode~2 exhibits degradation during object transport, while Mode~3 maintains a consistently higher escape-time margin. Dash lines indicate mean values averaged over time.}
        \label{subfig:cage_metric_comp}
    \end{subfigure}

    \caption{Contact distribution, proximity diagnostics, and enclosure quality for the three planning modes.}
   \label{fig:contact_distr_cage_metric}
   \vspace{-0.5cm}
\end{figure}

We use the same planar 4-DOF robot arm and learned eikonal caging model as in Section~\ref{subsec:example_wm_caging}. A circular object is transported from its initial pose to a goal region under the quasi-dynamic contact model of Section~\ref{subsec:contact_dynamics}. Across all experiments, the robot model, optimizer, dynamics, and constraints are identical; only the manipulation objective is changed. Figure~\ref{fig:wm_cage_mani_toy} illustrates the setup and representative trajectories.

\paragraph{Mode 1: Contact-prior baseline.}
We replace the caging reward with the minimum signed distance between the object and the \emph{whole robot arm}. This favors maintaining contact, but does not explicitly account for global escape geometry.

\paragraph{Mode 2: Whole-arm proximity baseline.}
We minimize the average distance between the object and \emph{all robot links}. This encourages distributed proximity, but remains a local objective that does not reason about escape routes.

\paragraph{Mode 3: Whole-arm caging objective (ours).}
We use the proposed formulation in Equation~\eqref{eq:wm_caging_mani}, which maximizes the learned escape-time field $\hat{\mathcal T}(\bm q_o^t,\bm q_r^t)$. This explicitly favors robot configurations that impede fast object escape.

\subsubsection{Result Comparison Across Three Objective Modes}

\paragraph{Coordinated whole-arm usage, proximity, and enclosure quality.}

We compare the optimized trajectories using three diagnostics:  
(i) link contact usage, defined as the fraction of time each link remains in contact with the object; (ii) average object--arm distance (metric used in Mode~2); and (iii) the escape-time caging metric. Results are shown in Figure~\ref{fig:contact_distr_cage_metric}.

Mode~1 relies primarily on a single mid-arm link, with limited participation from the remaining links (Figure~\ref{subfig:arm_utilization}). Because maintaining one active contact is often sufficient to reduce the objective, the planner has little incentive to form coordinated whole-arm enclosure. Accordingly, the average object--arm distance remains relatively large, and the resulting caging values are only moderate.

Mode~2 increases multi-link engagement and rapidly reduces object--arm distance early in the motion. However, proximity alone does not preserve enclosure during transport: as the arm reconfigures to move the object, the average distance rises and the caging metric drops substantially (Figure~\ref{subfig:cage_metric_comp}). This indicates that local closeness does not necessarily imply resistance to escape.

Mode~3 (ours) exhibits a different strategy. Rather than minimizing distance aggressively at the outset, the arm first arranges its geometry to obstruct global escape routes and then progressively tightens during transport. This produces sustained overlapping contact across multiple links, decreasing proximity over time, and consistently larger escape-time values than the two baselines.

Overall, both proximity and caging objectives encourage greater whole-arm involvement than the contact-prior baseline. However, only the caging objective consistently promotes coordinated enclosure over long horizons. These qualitative differences motivate the controlled robustness evaluations in the next section.

\section{Controlled Robustness Analysis Under Contact-Model Mismatch}\label{sec:robust_analyse}
This section uses the controlled toy whole-arm manipulation problem introduced earlier to isolate how different planning objectives affect robustness under disturbances and contact-model mismatch. The simplified geometry enables interpretable comparisons while preserving the key enclosure behaviors studied in this paper.

Contact-dynamics mismatches generally induce discrepancies between the predicted and actual object motions. When the contact model used during planning deviates from the true system dynamics, the resulting prediction error can be interpreted as an object velocity disturbance acting on the quasi-dynamic contact model introduced in Section~\ref{subsec:contact_dynamics}. Under mild assumptions, such disturbances can be viewed as bounded perturbations to the object motion. Section~\ref{subsec:game_theoretic_connection} formalizes this interpretation and shows how the escape-time objective corresponds to maximizing robustness against worst-case disturbances within this bound.

Using the same task setup and planned trajectories from the controlled toy whole-arm manipulation example in Section~\ref{subsec:caging_manipulation_example}, we evaluate robustness in two complementary ways. Section~\ref{subsec:robustness2motion_disturb} studies random object-motion disturbances, representing bounded but non-adversarial velocity mismatch at each time step. Although the planner is motivated by a worst-case disturbance interpretation, evaluating only adversarial disturbances would favor the proposed method. We therefore sample disturbances with the same magnitude bound but independent directions, yielding a more neutral comparison with baseline objectives. Section~\ref{subsec:contact_dyn_simp} then evaluates three representative forms of contact-model mismatch: simplifications in contact dynamics modeling, object shape, and friction parameters.

\subsection{Contact Dynamic Mismatch \& Object Escape Velocity}\label{subsec:game_theoretic_connection}

In practice, the quasi-dynamic model used during planning differs from the true second-order contact dynamics of the physical system. Let
\[
\bm v_o^{\text{pred}} = f_o(\bm q,\bm u_r)
\]
denote the object velocity predicted by the QP-based quasi-dynamic model, and let
\[
\bm v_o^{\text{true}} = f_o^{\text{true}}(\bm q,\bm u_r)
\]
denote the velocity produced by the true contact dynamics. The discrepancy between the two can be represented as a velocity mismatch
\[
\bm v_o^{\text{true}} = f_o(\bm q,\bm u_r) + \delta \bm v_o ,
\]
where $\delta \bm v_o$ captures modeling errors arising from quasi-dynamic approximations, simplified object geometry, or inaccurate friction parameters that alter the contact Jacobians and constraints in the QP formulation.

For a range of such modeling mismatches, we assume the resulting velocity error is bounded,
\[
\|\delta \bm v_o\| \le \bar v .
\]
In our framework we set $\bar v = u_{\max}$, corresponding to the maximal escape velocity assumed in the escape-time formulation. The worst-case disturbance direction is then taken to align with the steepest descent direction of the escape-time field,
\[
\delta \bm v_o
=
- u_{\max}
\frac{\nabla_{\bm q_o}\mathcal T(\bm q_o,\bm q_r)}
{\|\nabla_{\bm q_o}\mathcal T(\bm q_o,\bm q_r)\|}.
\]

Under this interpretation, whole-arm caging manipulation can be viewed as approximately maximizing resistance to bounded contact dynamics mismatches that drive the object toward escape from robot control region in the worst case.

\begin{figure}[!t]
    \centering
    \includegraphics[width=\linewidth]{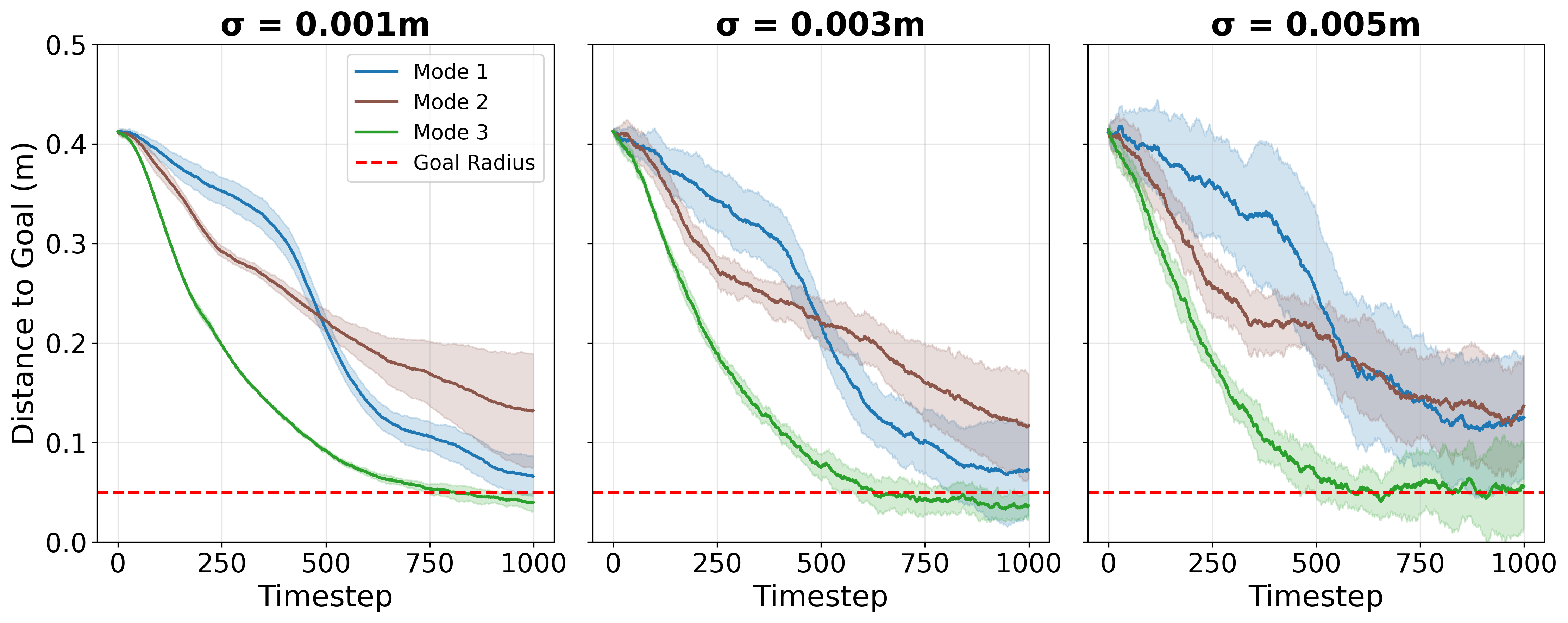}
    \caption{Object--to--goal distance over time during open-loop execution of planned trajectories under three levels of random disturbances across the three planning modes. Lower values indicate better task progress. Mode~3 shows the most concentrated and stable trajectories across disturbance levels.}
    \label{fig:robustness_detailed}
    \vspace{-0.2cm}
\end{figure}
\begin{figure}[!t]
    \centering
    \includegraphics[width=\linewidth]{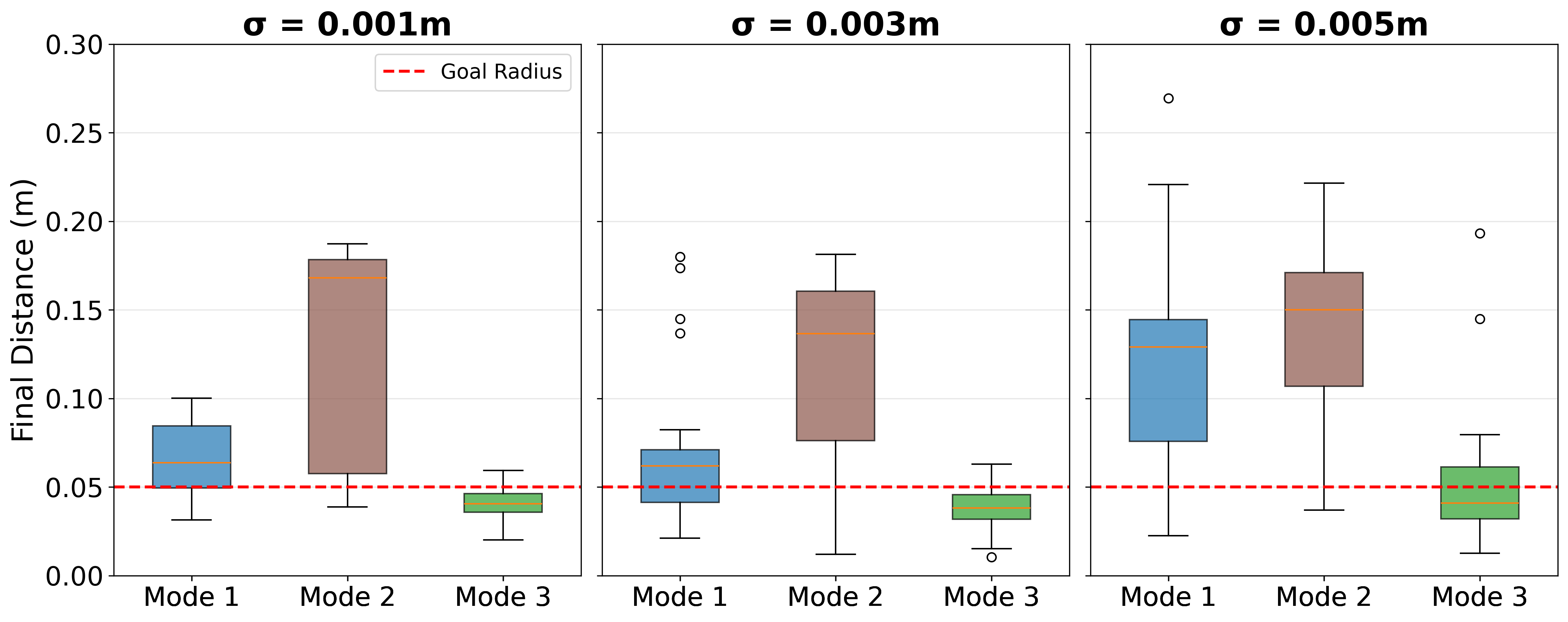}
    \caption{Final object-to-goal distance at the terminal time step under random disturbances. Smaller values indicate better task completion. Mode~3 consistently achieves the lowest final error, while Mode~2 exhibits the largest degradation under disturbance.}
    \label{fig:robustness_detailed_final_dist}
    \vspace{-0.5cm}
\end{figure}

\subsection{Robustness to Random Object Motion Disturbances}\label{subsec:robustness2motion_disturb}
Using the same task setup and planned trajectories from the illustrative example in Section~\ref{subsec:caging_manipulation_example}, we first evaluate how different objectives tolerate unmodeled object motion. For each mode, we fix the optimized open-loop robot trajectory $\{\bm{q}_r^t\}$ and execute the quasi-dynamic model while injecting perturbations into the object state at every time step. 

\begin{figure*}[!t]
    \centering
    \includegraphics[width=\linewidth]{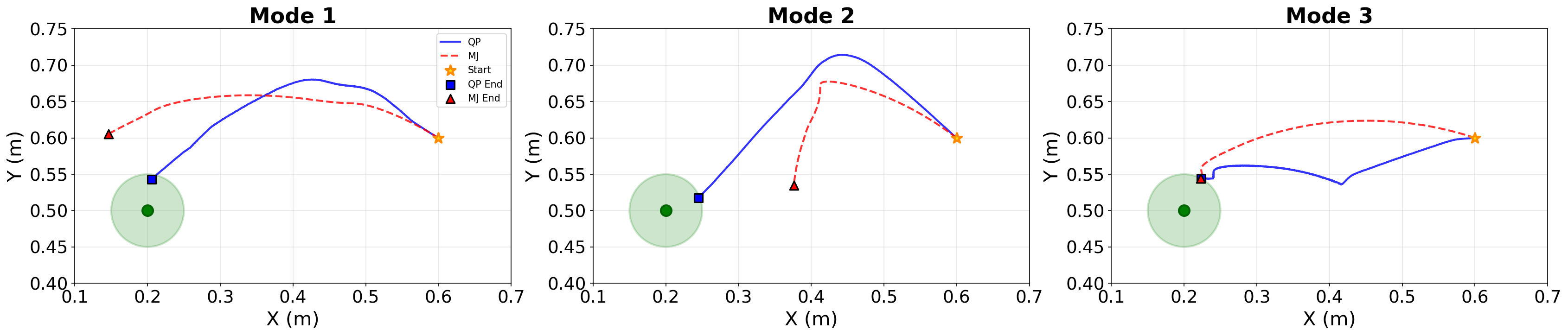}
    \caption{Object trajectories under dynamics-order mismatch. Robot motions are planned using the first-order quasi-dynamic model (blue solid) and executed under second-order MuJoCo dynamics (red dashed). The goal region is shown in green. Mode~3 preserves similar overall behavior under mismatch, whereas Modes~1 and~2 exhibit larger deviations or task failure.}
    \label{fig:qp_vs_mj}
    \vspace{-0.5cm}
\end{figure*}

Specifically, we add zero-mean positional perturbations with magnitudes $\sigma \in \{0.001, 0.003, 0.005\}\,\mathrm{m}$. The escape velocity bound used in the planner is $u_{\max}=0.05\,\mathrm{m/s}$, and the time step of the manipulation planner is $0.1\,\mathrm{s}$. The largest perturbation $\sigma=0.005\,\mathrm{m}$ therefore corresponds to the same maximal velocity mismatch bound ($u_{\max}\cdot h$). For each disturbance level, we perform 100 rollouts and record both the object-to-goal distance over time and the final distance (Figures~\ref{fig:robustness_detailed} and~\ref{fig:robustness_detailed_final_dist}).

The three objectives produce clearly different disturbance responses. Mode~1 exhibits relatively large trajectory variance during the early phase of execution, consistent with its reliance on limited enclosure from a dominant contact link. Mode~2 remains sensitive throughout the task: although it keeps the object closer to the arm on average, perturbations readily redirect the object into configurations with reduced escape resistance, leading to broad trajectory dispersion and frequent loss of progress toward the goal.

An instructive comparison arises between Modes~1 and~2 in the later phase of execution. After approximately 400 time steps, the variance of Mode~1 decreases and its final object-to-goal distance becomes smaller than that of Mode~2, despite maintaining a larger average object--arm distance throughout the trajectory. This behavior is consistent with the coordinated whole-arm usage analysis. Although Mode~1 does not globally minimize proximity, its evolving arm geometry gradually blocks a subset of escape routes. As shown in Figure~\ref{subfig:cage_metric_comp}, this increases the escape-time value after roughly 300 time steps, which in turn improves disturbance tolerance later in the manipulation.

Mode~2, in contrast, demonstrates that maintaining proximity alone is insufficient for robust task completion. Even though the object remains close to the arm, the planned geometry does not consistently suppress global escape opportunities. Consequently, small perturbations accumulate into larger final task errors than those of Mode~1 and Mode~3 across all disturbance levels.

Mode~3 exhibits the strongest robustness overall. Object trajectories remain tightly concentrated throughout execution, and final positions consistently lie within or near the goal region. Because execution is entirely open-loop, this performance cannot be attributed to feedback correction; rather, it reflects the larger escape-time margins induced during planning.

Taken together, these results show that average proximity is not a reliable proxy for robustness. Disturbance tolerance depends on how the arm geometry constrains escape over time, not only on instantaneous closeness. Explicitly optimizing escape resistance therefore yields the most stable whole-arm manipulation behavior under random object-motion disturbances.

\subsection{Exploiting Robustness for Contact Dynamic Simplification}\label{subsec:contact_dyn_simp}

The previous subsection modeled planning mismatch abstractly through random object-motion disturbances. We now examine whether similar trends arise when disturbances are generated by concrete simplifications of the contact model itself. Specifically, we evaluate three representative sources of mismatch: 1) the use of \emph{quasi-dynamics} contact dynamics in Equation~\eqref{eq:qp}, 2) shape under-approximation, 3) friction mismatch. The central question is whether geometric enclosure can compensate for reduced model fidelity. If so, simplified planning models may remain effective when combined with an escape-time objective.

\subsubsection{Simplification 1: First-order Contact Dynamics}\label{subsec:qp_vs_mj}
We first test whether trajectories planned under the quasi-dynamic first-order model remain effective when executed under MuJoCo’s second-order contact dynamics, which introduce inertia and richer transient effects. All trajectories are optimized with the same quasi-dynamic model and then executed open-loop in MuJoCo. Figure \ref{fig:qp_vs_mj} compares the object trajectories observed with the quasi-dynamic model and with those observed during open-loop MuJoCo execution for the three planning modes. 

The two baseline objectives exhibit distinct failure modes. Mode~1 relies heavily on a dominant pushing contact and is sensitive to unmodeled rotational or tangential effects, causing substantial trajectory deviation and failure to reach the goal. Mode~2 maintains proximity to the object, but does not consistently block escape routes; small discrepancies in sliding or rotation therefore redirect the object away from the intended path. 

Mode~3 is less sensitive to this modeling-order mismatch. Although detailed trajectories differ between planning and execution, the overall manipulation behavior remains similar and the final object state stays near the planned target. This suggests that enclosure geometry can reduce sensitivity to higher-order effects ignored during planning.

\subsubsection{Simplification 2: Object Shape Under-approximation.}
\begin{figure*}[!t]
    \centering
    \includegraphics[width=\linewidth]{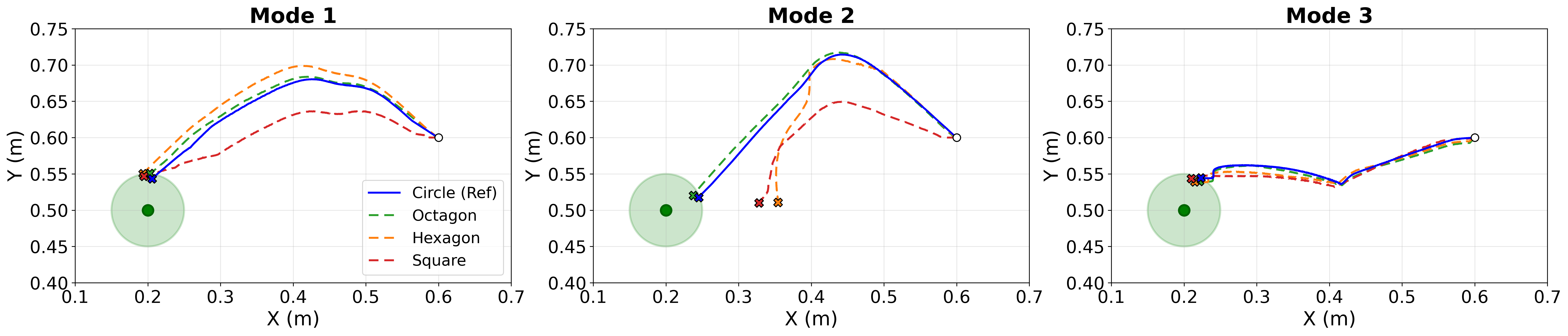}
    \caption{Object trajectories under shape under-approximation. Robot motions planned for a circular object are executed on polygonal objects (octagon, hexagon, square) with increasing geometric mismatch. The circle case is the no-mismatch reference. Mode~3 remains comparatively stable as mismatch increases, while the baselines accumulate larger trajectory errors.}
    \label{fig:shape_under_app}
    \vspace{-0.5cm}
\end{figure*}
\begin{figure}[!t]
    \centering
    \includegraphics[width=\linewidth]{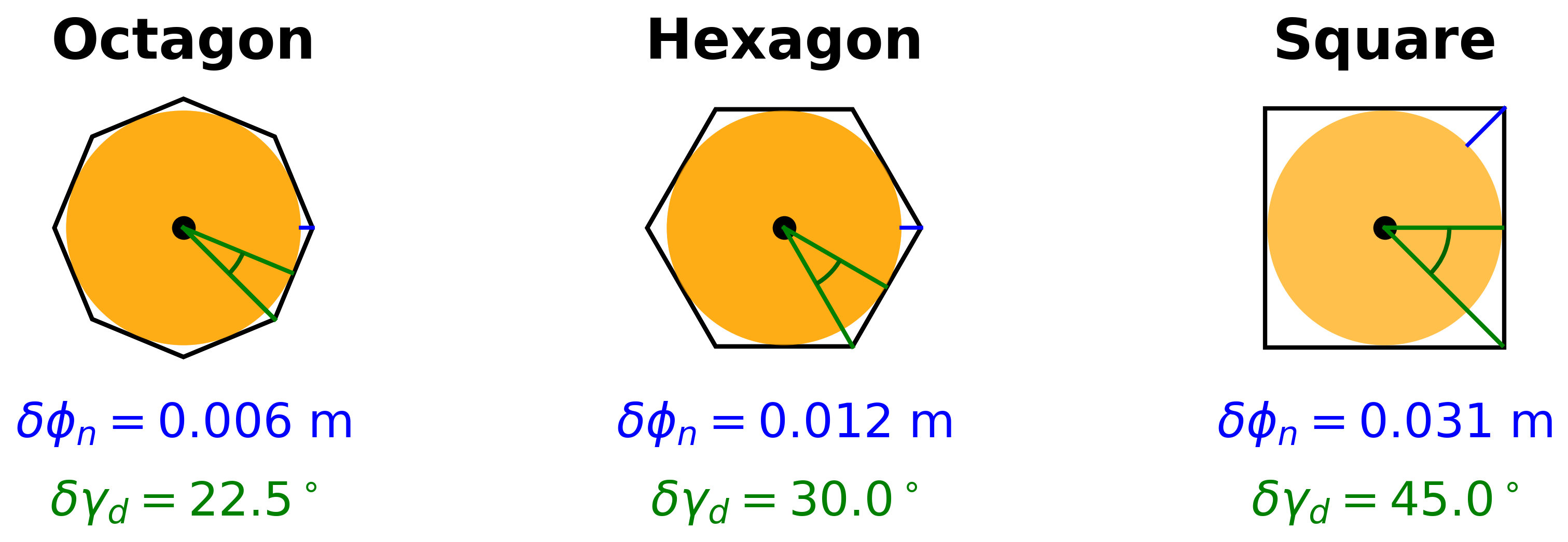}
    \caption{Illustration of shape under-approximation used during planning. The orange circle denotes the simplified object model used for optimization, while black outlines indicate the true polygonal shapes used during execution. Blue segments show maximum contact-point displacement $\delta \phi_n$, and green arcs show corresponding surface-normal mismatch $\delta \gamma_d$. Larger mismatch values imply greater geometric modeling error.}
    \label{fig:shape_app_visual}
    \vspace{-0.5cm}
\end{figure}
We next evaluate geometric simplification, where planning assumes a circular object of radius $0.075\,\mathrm{m}$ while execution uses polygonal objects (octagon, hexagon, and square) of increasing mismatch. The mismatch is quantified in Figure~\ref{fig:shape_app_visual} by the maximum contact-point displacement (\emph{contact mismatch} $\delta \phi_n$) and the corresponding surface normal orientation error (\emph{angle mismatch} $\delta \gamma_d$). All trajectories are optimized for the same circle model and then executed open-loop under the quasi-dynamic model with the true polygonal shapes.

Shape under-approximation affects not only the apparent object boundary but also the quantities used by the contact model. For a fixed robot configuration, the true polygonal shape shifts the closest surface point and alters surface normals and tangential directions. Since these quantities determine the signed distances and contact Jacobians $\bm J_i^n$ and $\bm J_{i,j}^d$ used by the quasi-dynamic formulation, geometric mismatch introduces systematic errors in the predicted object motion.

Figure \ref{fig:shape_under_app} shows the resulting object trajectories under increasing shape mismatch. For the octagonal object, all methods remain close to the reference behavior. As mismatch increases for the hexagon and square, the baselines degrade more noticeably. Mode~1 often achieves partial completion but accumulates trajectory error. Mode~2 is more sensitive, since errors in local contact geometry directly affect the directions favored by the proximity objective.

Mode~3 remains comparatively stable across all tested shapes. Even for the square object, trajectories stay close to the reference and continue to reach the goal region. These results suggest that whole-arm caging manipulation is less sensitive to systematic geometric modeling error than baseline objectives relying primarily on precise local contact information.

\subsubsection{Simplification 3: Nominal Friction Parameter Assumption.}

We evaluate robustness to friction mismatch under a simplified contact model in which a single nominal friction parameter is assumed during planning. Specifically, trajectories are optimized using a fixed friction coefficient ($\mu = 0.5$) and then executed open-loop under varying execution-time friction values $\mu \in \{0.0, 0.25, 0.5, 0.75, 1.0\}$. 

From a physical standpoint, friction primarily affects tangential contact interactions, influencing sliding behavior and rotational controllability. In the quasi-dynamic contact model, reducing friction weakens tangential constraints and increases susceptibility to slip, while higher friction strengthens tangential coupling. As a result, object orientation is inherently more sensitive to friction variation than object position. For the pushing task considered here, task success depends primarily on accurate object positioning rather than precise orientation control. We therefore evaluate robustness at the task level by examining object position trajectories and goal-reaching behavior. 

Figure \ref{fig:frictionless_contact_pos} shows the resulting object trajectories for the three planning modes. Mode~3 shows consistently small positional deviations across the tested friction values and continues to reach the goal region. In contrast, Modes~1 and~2 exhibit larger deviations as friction departs from the nominal planning value. This suggests that enclosure geometry can reduce task-level sensitivity to friction mis-specification even when tangential force modeling is approximate.

\begin{figure*}[!t]
    \centering
    \includegraphics[width=\linewidth]{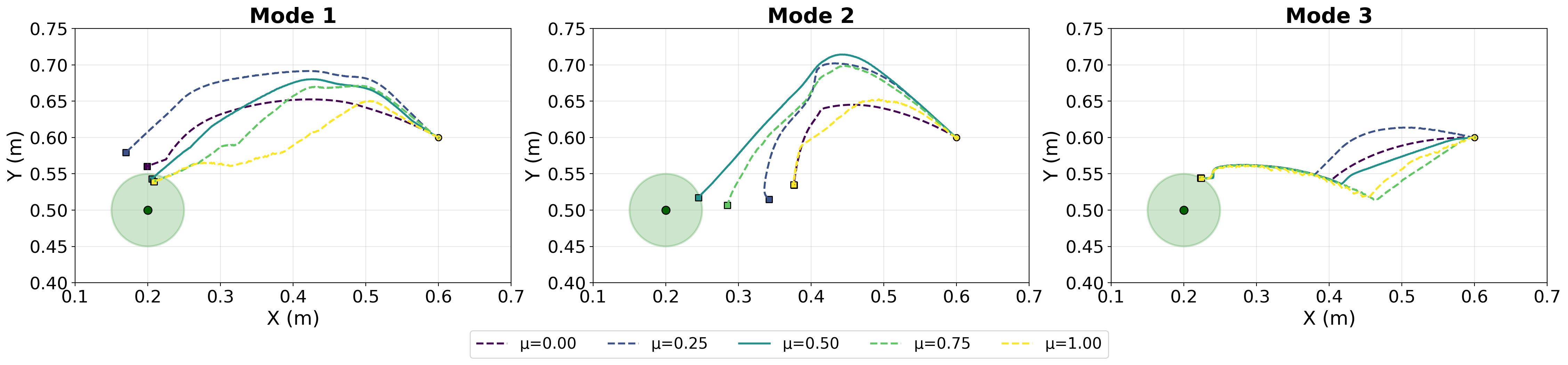}
    \caption{Object trajectories under friction mismatch. Robot motions planned with nominal friction coefficient $\mu=0.5$ are executed using different friction values. Mode~3 exhibits smaller task-level trajectory deviations across the tested range, indicating reduced sensitivity to friction mis-specification.}
    \label{fig:frictionless_contact_pos}
    \vspace{-0.5cm}
\end{figure*}

\subsubsection{Summary}
Across all three simplification scenarios: dynamics order, object shape, and friction parameters, a consistent pattern emerges. Objectives based on local contact heuristics (Mode~1) or average proximity (Mode~2) may maintain contact or closeness, but they do not explicitly preserve global escape resistance. Consequently, modeling errors can redirect the object into configurations with weak enclosure, leading to drift or task failure.

In contrast, the caging-aware objective directly optimizes an escape-time margin, biasing planning toward robot configurations that continue to restrict object escape under mismatch. This yields substantially greater tolerance to simplified contact models across all tested scenarios.

Together with the random-disturbance study, these experiments support the broader claim of this paper: geometric enclosure can be exploited to reduce dependence on highly accurate contact dynamics in whole-arm manipulation planning.

\section{Validation on Realistic Geometry and Real-World Systems}\label{sec:experiments}

Since the proposed method is fundamentally geometry-based, this section thus evaluates whether the conclusions obtained from the illustrative examples extend to realistic robot and object geometries.

In Section~\ref{subsec:example_wm_caging}, we showed that the physics-informed eikonal caging metric accurately approximates the eikonal equation solution and enables whole-arm caging configuration optimization for robot and object geometries composed of simple shape primitives (e.g., capsules and circles). Here, we examine whether these properties transfer to realistic scenarios.

Specifically, we investigate the following questions:

\termdef{def:quest1}{\textbf{Q1}}: Does the PINN-based escape-time field approximation transfer to realistic robot geometries? \label{q:one}

\termdef{def:quest2}{\textbf{Q2}}: Does the key-point-based escape-time field approximation transfer to realistic object geometries?

In the robustness analysis (Section~\ref{sec:robust_analyse}), we showed that whole-arm caging manipulation improves robustness to object motion disturbances and contact-dynamics simplifications in a goal-reaching task. We further examine:

\termdef{def:quest3}{\textbf{Q3}}: Does this robustness transfer to realistic robot and object geometries?

\termdef{def:quest4}{\textbf{Q4}}: How much do whole-arm contact geometry and whole-arm caging contribute to the observed robustness?

\termdef{def:quest5}{\textbf{Q5}}: How does robustness degrade when the caging objective competes with task objectives in a \emph{trajectory-tracking} task?

\termdef{def:quest6}{\textbf{Q6}}: How robust is whole-arm caging manipulation under multiple simultaneous contact-model mismatches that arise during sim-to-real transfer?

We conduct experiments in three stages. First, Section~\ref{subsec:exp_wm_caging_conf} evaluates the learned physics-informed eikonal caging metric in a configuration-planning setting with realistic robot and object geometries, answering \termref{def:quest1}{\textbf{Q1}}--\termref{def:quest2}{\textbf{Q2}}. Second, Section~\ref{subsec:mobile_franka_cage_mani} integrates the escape-time objective into manipulation planning with contact dynamics to evaluate manipulation robustness under representative modeling simplifications, answering \termref{def:quest3}{\textbf{Q3}}--\termref{def:quest5}{\textbf{Q5}}. Finally, we conducted real-world experiments to answer \termref{def:quest6}{\textbf{Q6}} in Section~\ref{subsec:sim2real}.

\begin{figure}[t]
    \centering
    \includegraphics[width=0.6\columnwidth]{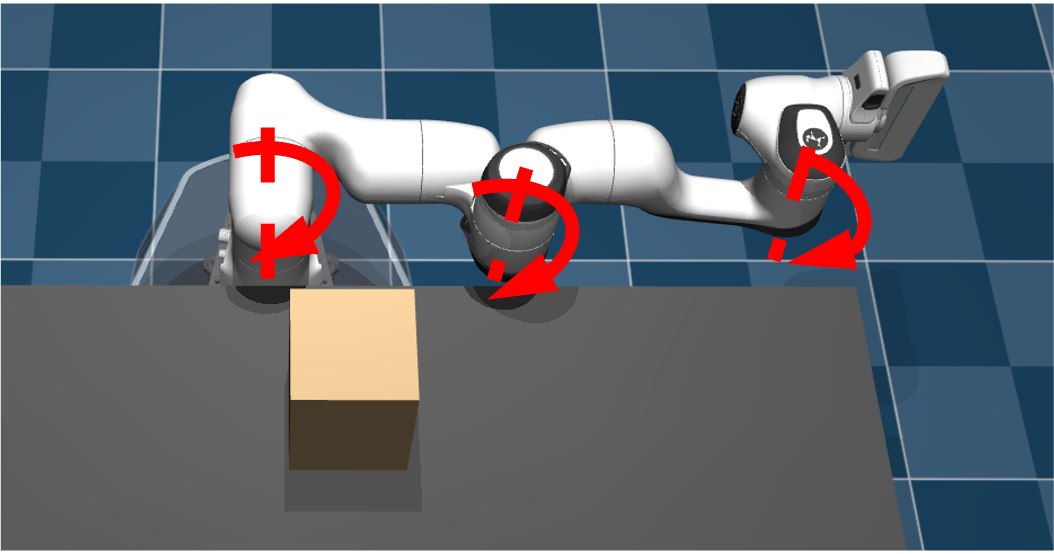}
    \caption{Experiment setup with a Franka arm for tabletop whole-arm manipulation planning. Red dashed lines and arrows indicate the active joints.}
    \label{fig:wm_franka_exp_setup}
\end{figure}

\paragraph{Experiment setups.} All simulation experiments consider planar tabletop settings using a Franka robot arm. To isolate the role of whole-arm geometric engagement while retaining realistic articulated structure, the arm is constrained to operate in a fixed-height plane parallel to the table. A subset of joints is fixed and only joints $[1,4,6]$ are actively controlled (Figure~\ref{fig:wm_franka_exp_setup}), yielding an effective planar whole-arm manipulator that preserves the true link geometry and kinematic coupling of the Franka arm. The robot is initialized from the configuration shown in Figure~\ref{fig:wm_franka_exp_setup}, which is used in all experiments in this section.

\subsection{Whole-Arm Caging Configuration Planning}
\label{subsec:exp_wm_caging_conf}
We first evaluate whole-arm caging configuration planning using a fixed-base Franka arm. Given a fixed object pose on the table, the goal is to find a \textit{static robot configuration} that maximizes the object’s escape-time value while respecting joint limits and robot--object collision constraints. We solve the resulting configuration optimization problem using CMA-ES, treating the learned escape-time field as a black-box objective. Implementation details are provided in Appendix~\ref{app:wm_caging_opt}. A qualitative comparison between the learned escape-time field and ground-truth solutions obtained with a classical FMM solver is detailed in Appendix~\ref{app:piec_franka}.

\begin{figure}[!t]
    \centering
    \begin{subfigure}[t]{\linewidth}
        \centering
        \includegraphics[width=\linewidth]{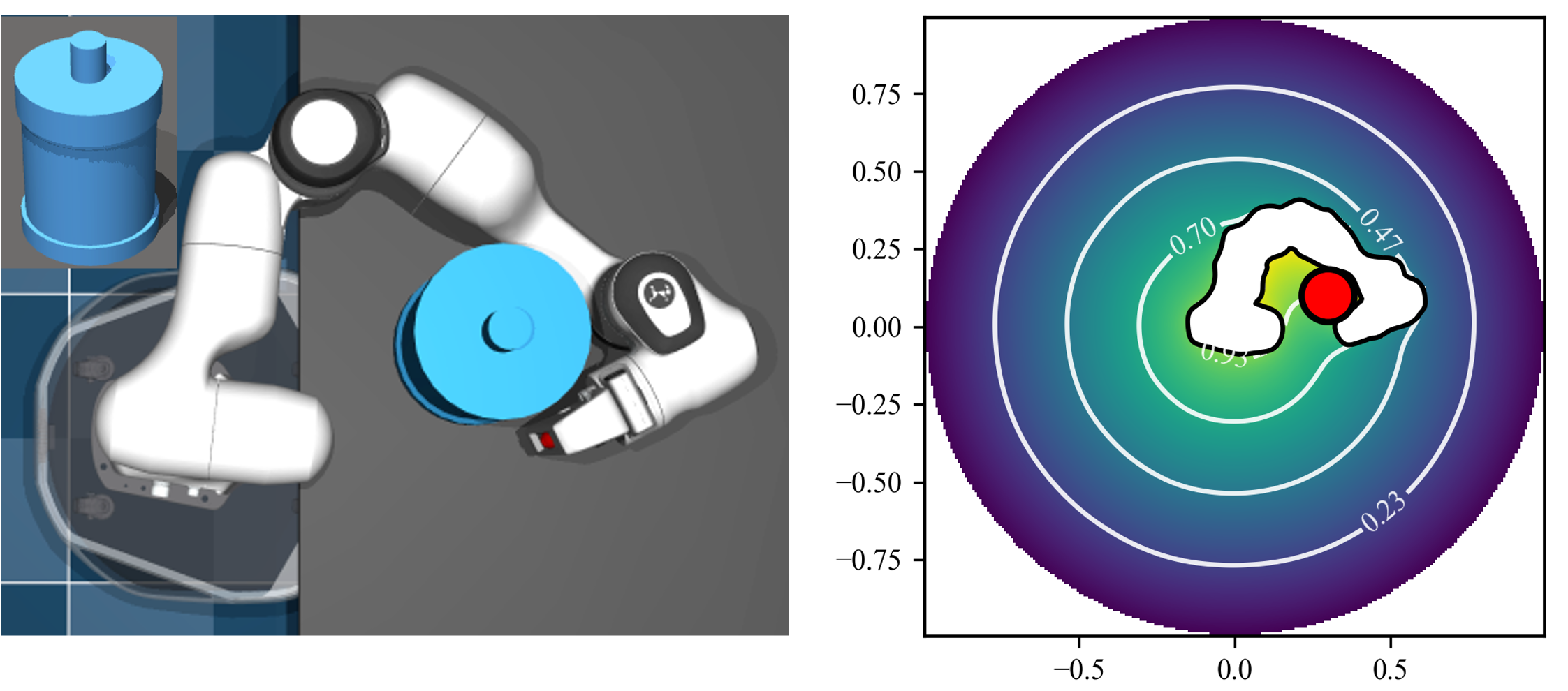}
        \caption{Water jug approximated by a single circle whose center servers as the key point for escaping-time query.}
        \label{fig:wm_caging_franka_water_jug}
    \end{subfigure}
    \hfill
    \begin{subfigure}[t]{\linewidth}
        \centering
        \includegraphics[width=\linewidth]{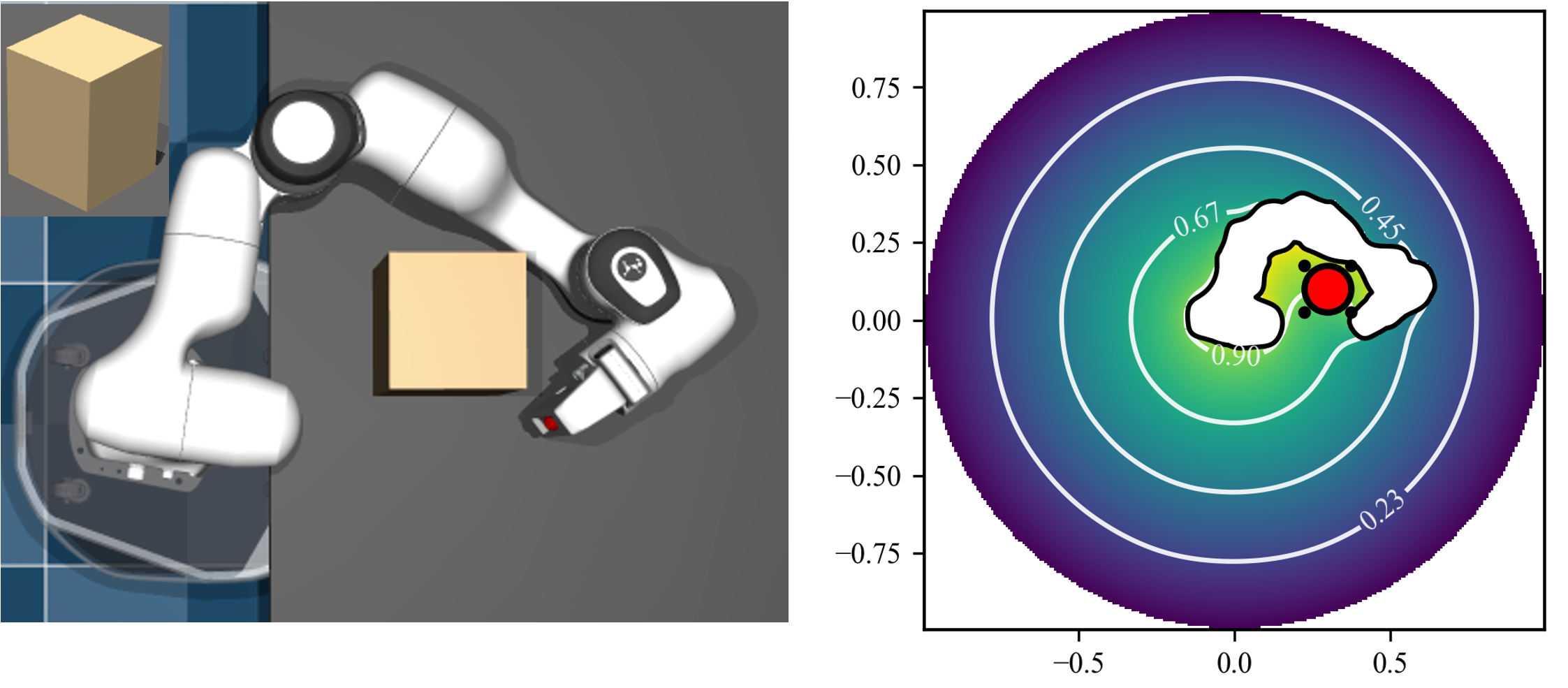}
        \caption{Box approximated by five circular key points: one at the center and four at the corners.}
        \label{fig:wm_caging_franka_box}
    \end{subfigure}
    \hfill
    \begin{subfigure}[t]{\linewidth}
        \centering
        \includegraphics[width=\linewidth]{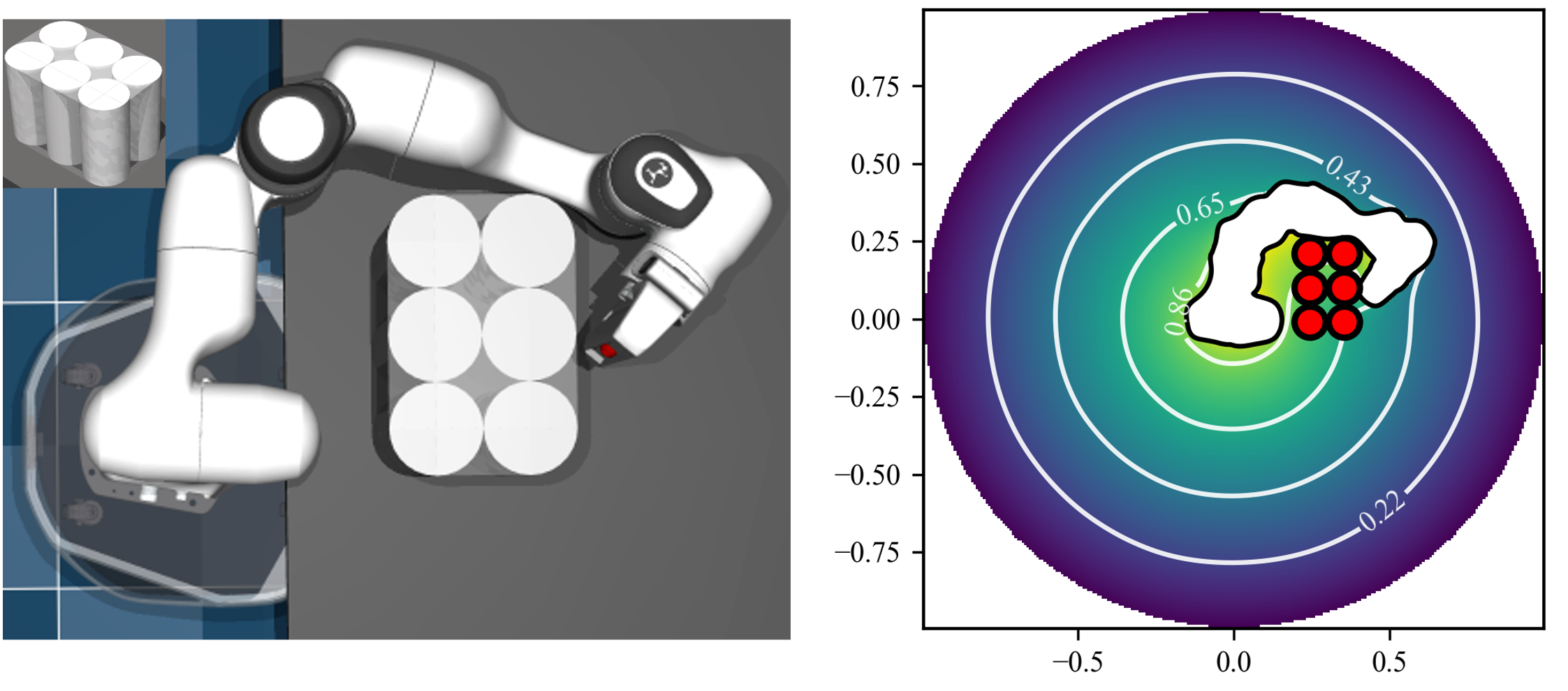}
        \caption{Kitchen paper pack approximated by six circular key points in 2D.}
        \label{fig:wm_caging_franka_paper_pack}
    \end{subfigure}
    \caption{Whole-arm caging configurations for a Franka arm and three everyday object geometries. For each object, the left panel shows the optimized caging configuration obtained with CMA-ES, with a 3D rendering of the object shown in the inset. The right panel visualizes the learned escape-time field together with the object’s 2D point-based approximation. Each object is represented by a small set of circles whose centers serve as key points for querying and aggregating escape-time values during optimization. These examples illustrate that the learned caging metric extends naturally from the illustrative examples to realistic robot and object geometries.}
    \label{fig:wm_caging_franka_objects}
    \vspace{-0.5cm}
\end{figure}

Figure~\ref{fig:wm_caging_franka_objects} visualizes the resulting optimal whole-arm caging configurations together with the corresponding learned escape-time fields for three daily objects: a \texttt{water jug}, a \texttt{box}, and a kitchen \texttt{paper pack}. 
All objects are placed at the same fixed pose $[0.3,\,0.1]\,\mathrm{m}$ on the table, and the robot is initialized from the same joint configuration shown in Figure~\ref{fig:wm_franka_exp_setup}. For each object, the left panel shows the optimized whole-arm configuration found by CMA-ES, while the right panel shows the learned escape-time field overlaid with the object’s 2D key-point representation. Across all objects, the learned escape-time field exhibits smooth and globally structured level sets that reflect the escape difficulty induced by the articulated arm geometry. The optimizer consistently finds whole-arm configurations that wrap around the object to prevent them from escaping, rather than relying on a single link or end-effector contact. 

\begin{table}[!t]
    \centering
    \caption{Whole-arm caging configuration planning results on realistic object geometries. Each object is approximated by a small set of circular key points whose centers query the learned escape-time field. We report the mean and minimum escape-time values $\mathcal{T}$ across key points, where larger values indicate stronger geometric enclosure. The minimum value reflects the weakest protected region of the object.}
    \label{tab:wm_caging_config_results}
        \begin{tabular}{lccc}
            \toprule
            Object &   \# Key Points & Mean $\mathcal{T} (\hat{\mathcal{T}})$ & Min $\mathcal{T}$ \\
            \midrule
            Water jug   & 1 & 0.95 & 0.95 \\
            Box         & 5 & 0.93 & 0.84 \\
            Paper pack  & 6 & 0.80 & 0.66 \\
            \bottomrule
        \end{tabular}
    \vspace{-0.5cm}
\end{table}

Quantitative results are summarized in Table~\ref{tab:wm_caging_config_results}, which reports the mean and minimum escape-time values across object key points at the optimized configurations. 
The single-point water jug achieves the highest escape-time value ($\mathcal{T}=0.95$), while multi-point objects exhibit lower mean escape times due to their increased spatial extent and directional exposure. For the box and paper pack, escape-time values vary across key points, revealing anisotropic escape difficulty induced by object shape. The optimizer implicitly balances these variations by shaping the arm to improve enclosure across \emph{all} key points, consistent with the point-based lower-bound formulation described in Section~\ref{subsec:piec}.

Overall, these results answer \termref{def:quest1}{\textbf{Q1}}, \termref{def:quest2}{\textbf{Q2}}. The proposed physics-informed eikonal caging metric can be optimized effectively for whole-arm caging configuration planning with realistic robot kinematics and object geometries using the key-point-based escape-time approximation. The learned escape-time field provides a smooth, geometry-aware objective that scales beyond the illustrative examples. This objective forms the basis for the whole-arm caging manipulation experiments presented in the following section.

\subsection{Whole-Arm Caging Manipulation}
\label{subsec:mobile_franka_cage_mani}

\begin{figure}[!t]
    \vspace{-10pt}
    \centering
    \begin{minipage}[t]{0.24\linewidth}
        \centering
        \includegraphics[width=\linewidth]{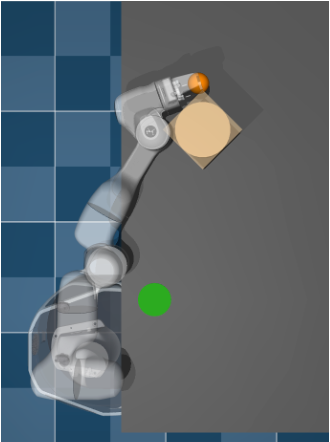}
        \caption*{(a) Point}
    \end{minipage}
    \hfill
    \begin{minipage}[t]{0.24\linewidth}
        \centering
        \includegraphics[width=\linewidth]{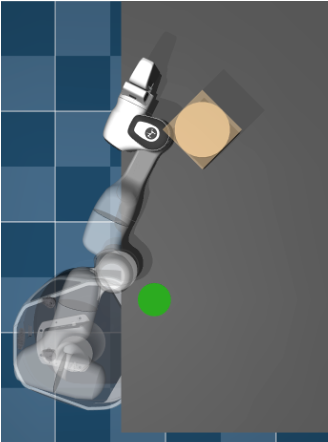}
        \caption*{(b) Link}
    \end{minipage}
    \hfill
    \begin{minipage}[t]{0.24\linewidth}
        \centering
        \includegraphics[width=\linewidth]{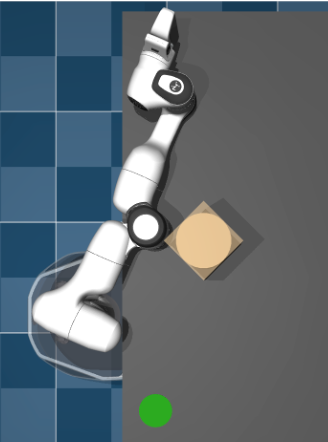}
        \caption*{(c) WAContact}
    \end{minipage}
    \hfill
    \begin{minipage}[t]{0.24\linewidth}
        \centering
        \includegraphics[width=\linewidth]{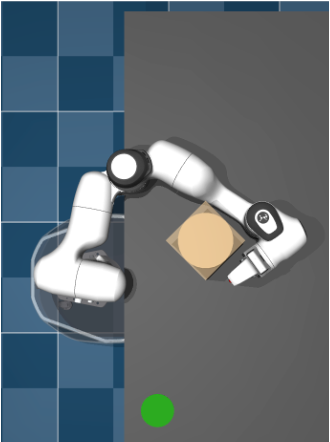}
        \caption*{(d) WACaging}
    \end{minipage}
    \caption{Initial contact configurations for whole-arm caging manipulation (\emph{WACaging}) and baseline contact geometries (\emph{Point}, \emph{Link}, \emph{WAContact}). The nominal object representation used for planning (orange circle) is an under-approximation of the true box geometry (light orange). Transparent robot parts are only used for visualization and are incapable of making contact with the target object. The green marker indicates the goal object region used in subsequent manipulation tasks.}
    \label{fig:wm_mani_exp_setup}
    \vspace{-0.5cm}
\end{figure}

We next evaluate whole-arm caging manipulation on a Franka robot to examine whether the robustness observed in the illustrative examples transfers to realistic robot geometries (\termref{def:quest3}{\textbf{Q3}}). In contrast to the earlier analysis, which compared multiple objective formulations, here we fix a common manipulation objective and instead vary the robot contact geometry available during manipulation (\termref{def:quest4}{\textbf{Q4}}).

Specifically, we consider four variants that differ in how the robot arm can interact with the object:
\begin{itemize}
    \item \textbf{Point:} Contact is restricted to a yellow spherical end-effector point, as shown in Figure~\ref{fig:wm_mani_exp_setup} (a).
    \item \textbf{Link:} Contact is restricted to a short distal arm segment consisting of the last arm link and the gripper, as shown in Figure~\ref{fig:wm_mani_exp_setup} (b).
    \item \textbf{WAContact:} The full arm geometry is available for contact under the same minimal-distance contact prior used by the previous two baselines. This prior is the same metric used in Mode~1 of the illustrative whole-arm caging manipulation example in Section~\ref{subsec:caging_manipulation_example}.
    \item \textbf{WACaging:} The full arm geometry is used together with the learned escape-time caging objective (our whole-arm caging manipulation planner).
\end{itemize}
This comparison separates robustness arising from increased whole-arm contact engagement from robustness induced by the escape-time caging objective, allowing us to isolate their respective contributions to the observed robustness (\termref{def:quest3}{\textbf{Q3}}, \termref{def:quest4}{\textbf{Q4}}).

Besides, unlike the fixed-base Franka configuration optimization experiment in Section~\ref{subsec:exp_wm_caging_conf}, we enable the mobile base in this section to enlarge the reachable workspace of the arm. The base does \emph{not} participate in object contact or caging; all interactions with the object occur exclusively through the arm. The mobile platform ensures feasible arm configurations for objects at different table locations while preserving the geometric interpretation of whole-arm engagement.

For each variant, we first solve the whole-arm caging configuration optimization to obtain the initial robot-object contact configurations shown in Figure~\ref{fig:wm_mani_exp_setup}. This optimization follows the same formulation as in Section~\ref{subsec:exp_wm_caging_conf}, with the mobile base enabled. From these initial configurations, we then solve the manipulation trajectory optimization problem for two tasks: (i) \emph{Whole-Arm Trajectory Optimization}, in which the object is pushed into a specified goal region, and (ii) \emph{Whole-Arm Trajectory Tracking}, in which the object is required to follow a prescribed reference trajectory. In the trajectory-tracking task, we deliberately select the reference trajectory in which the task objective competes with the caging objective in order to study how this trade-off affects robustness (\termref{def:quest5}{\textbf{Q5}}).

To ensure fair comparisons, all methods share the same state and action spaces. In both tasks, the mobile base position and the active arm joints are jointly optimized. Additional planner parameters and cost function details are provided in Appendix~\ref{app:wm_caging_mani}.

\subsubsection{Whole-Arm Trajectory Optimization:} The circular object is initially placed at $[0.35, 0.1]\,\mathrm{m}$ and must be pushed into a circular goal region centered at $[0.2, -0.4]\,\mathrm{m}$ with radius $0.05\,\mathrm{m}$. We used the same planner, planner parameters and robot velocity limits as in the planar illustrative example. Figure~\ref{fig:robot_obj_traj} illustrates the keyframe robot-object configurations with the considered four contact geometries. The blue curves in Figure~\ref{fig:traj_opt_random_dist} show the corresponding object trajectories obtained by open-loop execution of the planned mobile Franka trajectories.  

\paragraph{Manipulation capability and link engagement.}

We observe that the \emph{Point} baseline, which restricts contact to a spherical end-effector, fails to consistently push the object into the goal region, whereas the other variants succeed. With single-point contact, successful manipulation requires continuously regulating the pushing direction so that the object is steered toward the goal while maintaining stable contact. Even small deviations in the relative robot–object configuration can induce slip or unintended rotation, forcing the robot to break and re-establish contact before continuing the push.

In our setup, the planner jointly optimizes the mobile base pose and three active arm joints, resulting in a relatively high-dimensional search space. For single-point pushing, feasible solutions occupy a narrow region of this space, as they require precise coordination between base motion and arm configuration to maintain contact while steering the object. Although such behaviors can be manually designed, discovering them through open-loop trajectory optimization under simplified contact dynamics is challenging. In contrast, allowing larger portions of the arm to engage the object provides passive geometric guidance. Multi-link contact reduces sensitivity to small configuration variations and enables the object to be steered toward the goal without requiring precise directional corrections at each time step.

\paragraph{Robustness to object motion disturbances.}
Figure~\ref{fig:wm_franka_caging_comp} compares the evolution of caging values over time for the four contact geometries. With the spherical end-effector, the robot quickly loses control of the object, and the caging value drops to zero after approximately 350 steps. The \emph{Link} baseline maintains slightly higher enclosure due to increased contact area, but remains sensitive to slip and gradual escape. Allowing the full arm geometry (\emph{WAContact}) leads to noticeably higher initial caging values and improved geometric coverage. However, because escape directions are not explicitly suppressed, the object progressively slides along the arm during manipulation (as shown in Figure~\ref{fig:WAContact_nominal}), causing the caging value to decrease over time and eventually approach the link segment of the \emph{Link} baseline (blue and brown lines in Figure~\ref{fig:wm_franka_caging_comp}, correspondingly). In contrast, \emph{WACaging} explicitly maximizes the learned escape-time field, maintaining substantially higher caging values and reinforcing enclosure as manipulation progresses.

These differences translate directly into robustness under random object motion disturbances. In Figure~\ref{fig:traj_opt_random_dist}, we execute the planned trajectories with additive disturbances of magnitude $\bm{\sigma}=0.005\,\mathrm{m}$. The disturbed object trajectories (thin grey lines) and their mean trajectories (red curves) reveal a clear progression in robustness as link engagement increases. Compared to \emph{Point} and \emph{Link}, \emph{WAContact} produces more concentrated disturbed trajectories and smaller deviations from the nominal path. Moreover, \emph{WACaging} exhibits a markedly stronger effect: disturbed trajectories are tightly funneled toward the goal region, with significantly reduced dispersion relative to \emph{WAContact}. This is reflected quantitatively in Table~\ref{tab:franka_d2g_sr}, where \emph{WAContact} achieves modest improvements in final distance and success rate over baselines, while \emph{WACaging} yields substantially lower final distance-to-goal errors and consistently higher success rates. 

Together, these results indicate a clear hierarchy: increasing geometric engagement improves disturbance tolerance, and explicitly maximizing the escape-time value further amplifies this robustness by actively suppressing escape modes rather than merely increasing contact area (\termref{def:quest3}{\textbf{Q3}}, \termref{def:quest4}{\textbf{Q4}}).

\paragraph{Robustness to model mismatches.}

We next evaluate robustness under systematic model mismatches, following the simplification analysis in the planar illustrative examples. The nominal planning model assumes a friction coefficient $\mu = 0.5$, circular object geometry, and QP-based quasi-dynamics. We then open-loop execute the planned trajectories under four mismatch conditions: higher friction ($\mu = 0.8$), lower friction ($\mu = 0.2$), box object geometry, and second-order MuJoCo contact dynamics, separately. 

The resulting object trajectories and final distance-to-goal values are summarized in Figure~\ref{fig:traj_opt_mismatch_comp} and Table~\ref{tab:franka_mismatch}. Across all mismatch conditions, a clear hierarchy emerges. From \emph{Point} to \emph{Link} to \emph{WAContact}, the object trajectories under mismatch (dashed curves in Figure~\ref{fig:traj_opt_mismatch_comp}) exhibit progressively smaller deviation from the nominal trajectory (solid blue curve) generated using the simplified model. 
\clearpage
\begin{figure*}[p]
    \centering

    % ---------- Row 1: robot & object keyframes ----------
    \begin{subfigure}{0.48\linewidth}
        \centering
        \includegraphics[width=\linewidth]{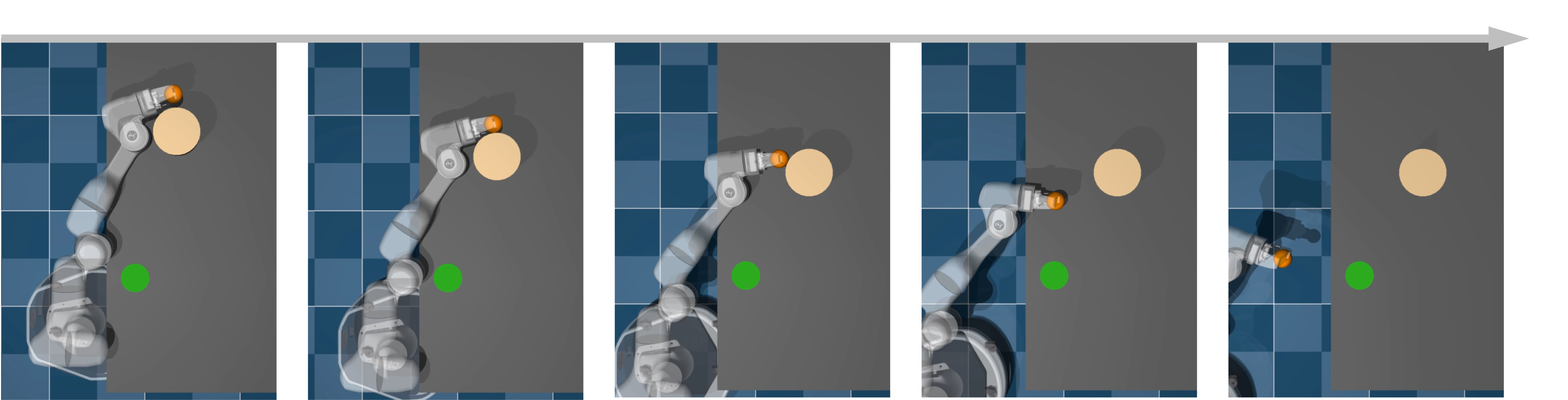}
        \caption{Point: yellow spherical end-effector as contact geometry.}
        \label{fig:point_nominal}
    \end{subfigure}
    \hfill
    \begin{subfigure}{0.48\linewidth}
        \centering
        \includegraphics[width=\linewidth]{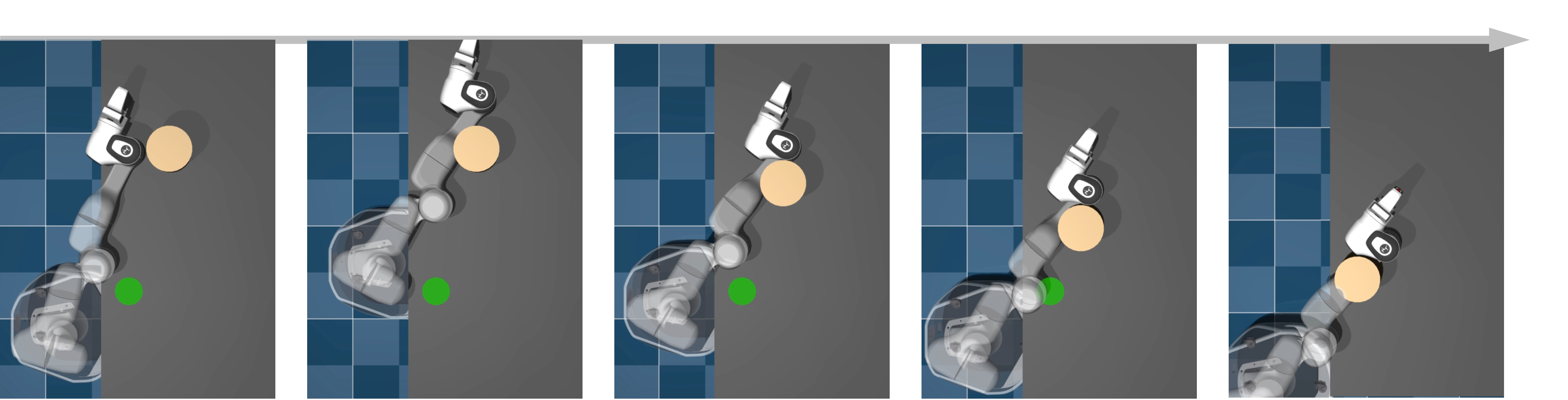}
        \caption{Link: using the last Franka arm link and the gripper.}
        \label{fig:link_nominal}
    \end{subfigure}

    \vfill

    \begin{subfigure}{0.48\linewidth}
        \centering
        \includegraphics[width=\linewidth]{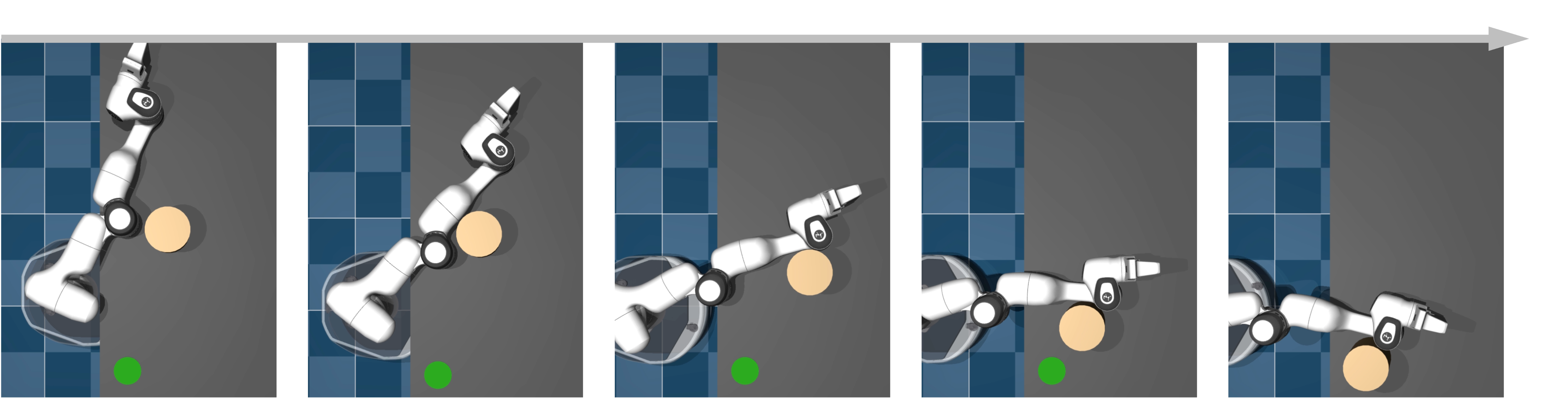}
        \captionsetup[subfigure]{justification=centering}
        \caption{WAContact: using the whole robot arm.}
        \label{fig:WAContact_nominal}
    \end{subfigure}
    \hfill
    \begin{subfigure}{0.48\linewidth}
        \centering
        \includegraphics[width=\linewidth]{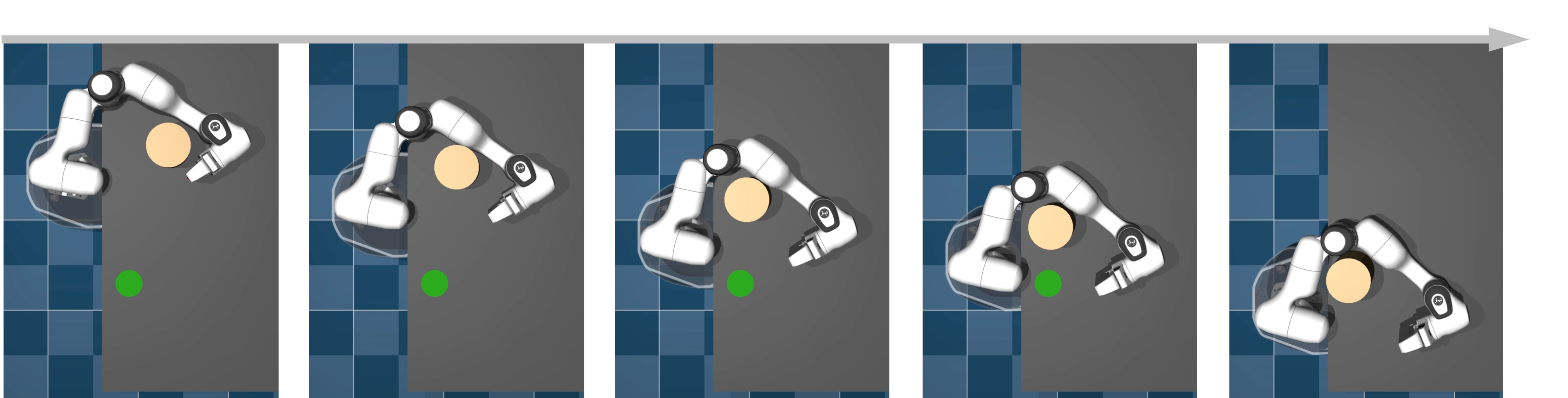}
        \caption{WACaging (Ours)}
        \label{fig:WACaging_nominal}
    \end{subfigure}

    \caption{Robot and object configurations at keyframes with four different contact geometries.}
    \label{fig:robot_obj_traj}

    \vspace{1.5\baselineskip}

    % ---------- Row 2: random disturbance robustness ----------
    \includegraphics[width=\linewidth]{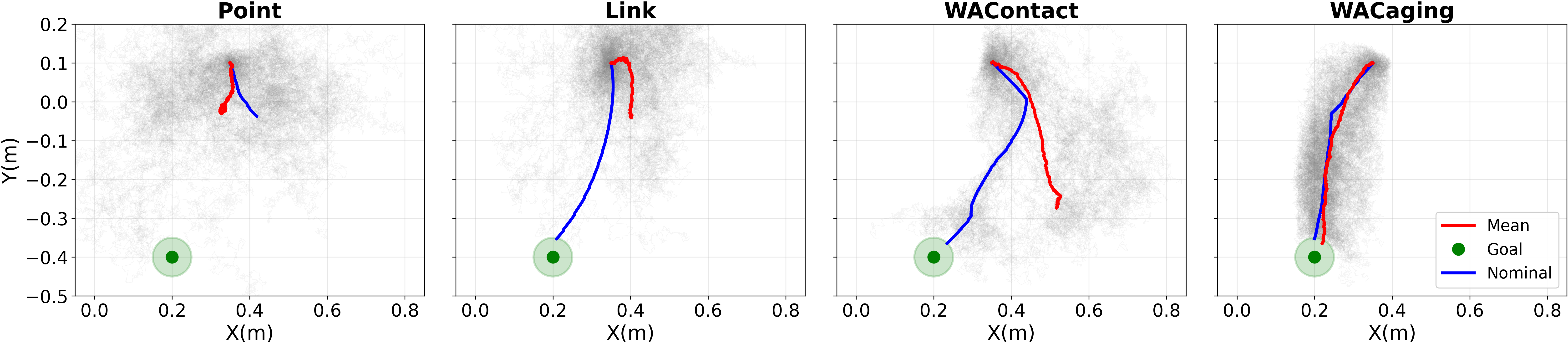}
    \captionof{figure}{Robustness evaluation under random object motion disturbances. The object trajectories without any disturbances are shown as \emph{Nominal} trajectories with blue lines. \emph{Red} trajectories show the mean disturbed object trajectories (shown with grey thin lines) over 100 trials.}
    \label{fig:traj_opt_random_dist}

    \vspace{1.5\baselineskip}

    % ---------- Row 3: caging metric & mismatch comparison ----------
    \begin{minipage}[t]{0.3\textwidth}
        \centering
        \includegraphics[width=\linewidth]{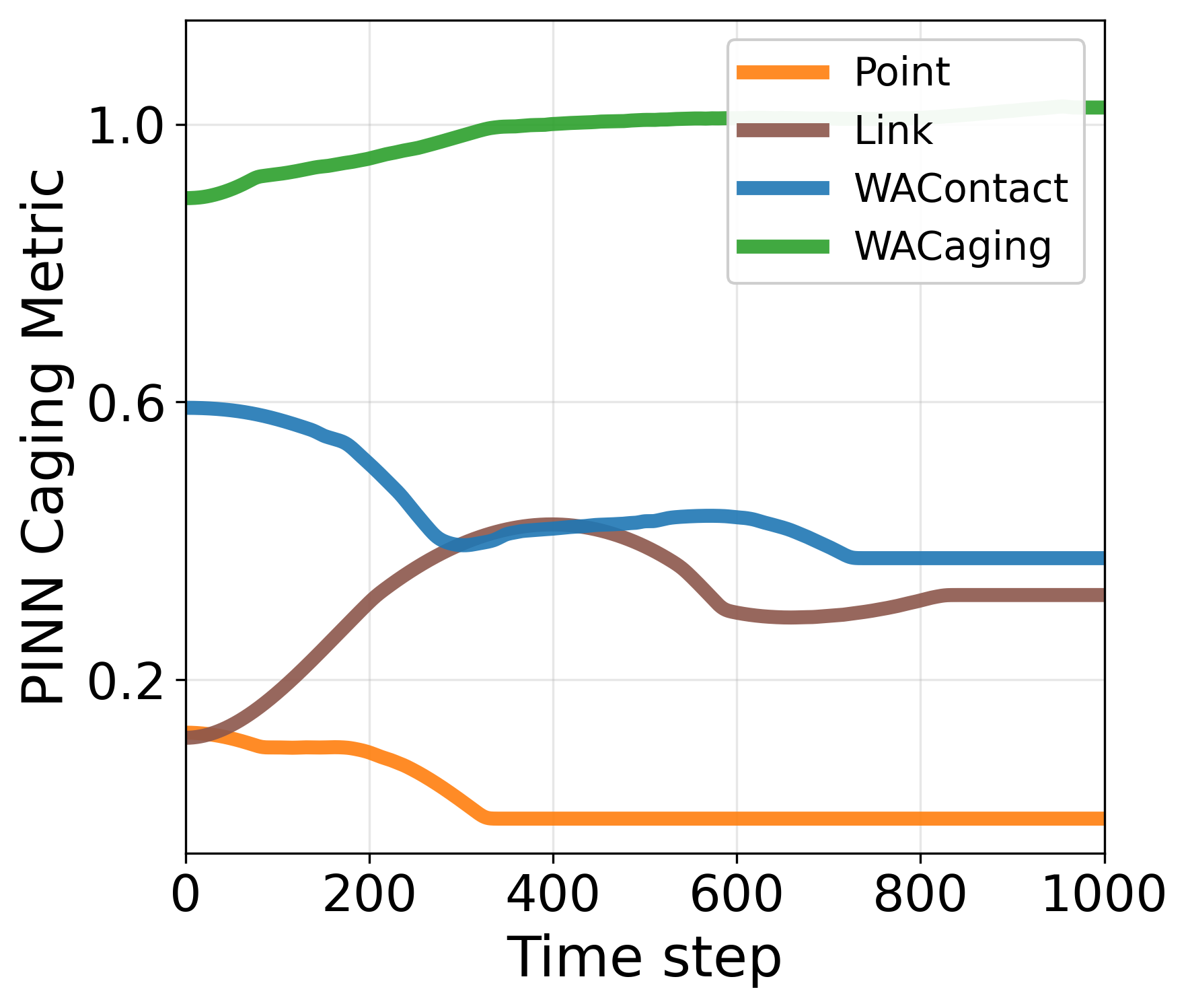}
        \captionof{figure}{Caging metric comparison.}
        \label{fig:wm_franka_caging_comp}
    \end{minipage}
    \hfill
    \begin{minipage}[t]{0.68\textwidth}
        \centering
        \includegraphics[width=\linewidth]{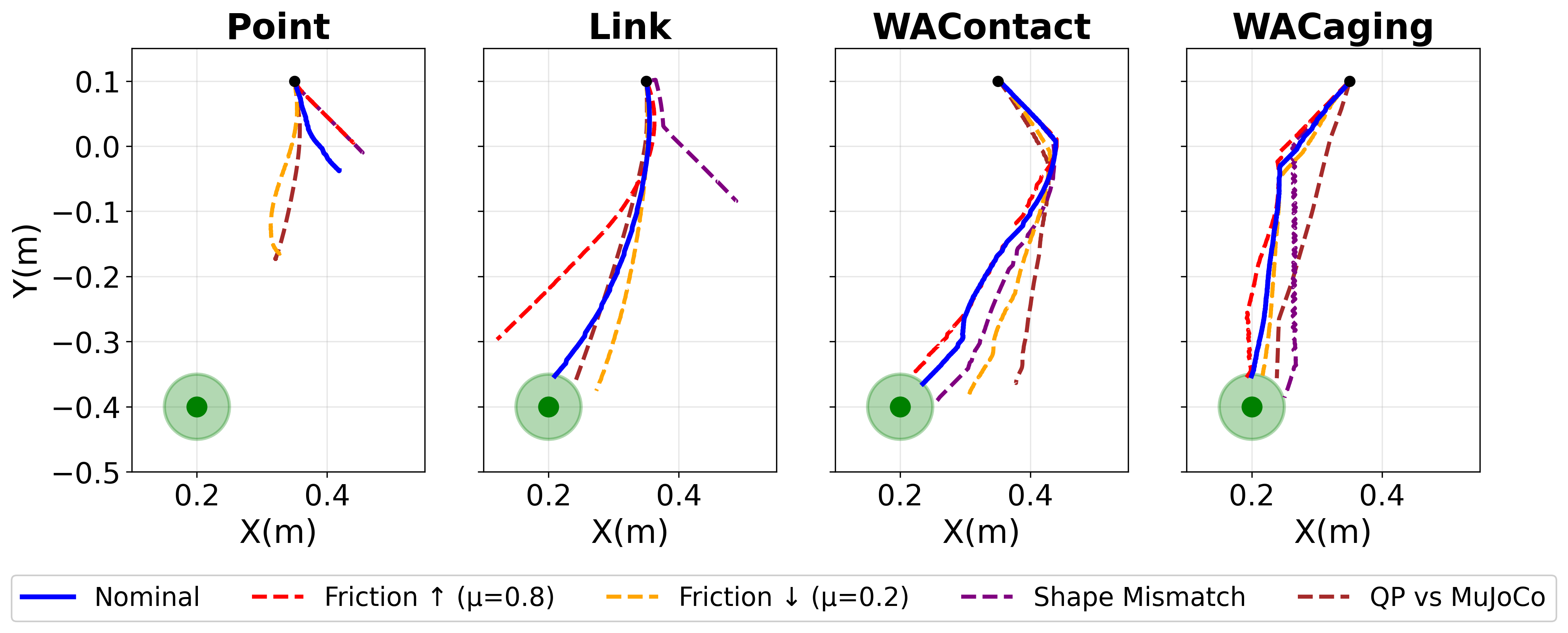}
        \captionof{figure}{Object trajectories under model mismatches.}
        \label{fig:traj_opt_mismatch_comp}
    \end{minipage}

    \vspace{1.5\baselineskip}

    % ---------- Row 4: tables (as tables, via captionof{table}) ----------
    \begin{minipage}[t]{0.3\textwidth}
        \centering
        \captionof{table}{Final distance to goal (D2G) and success rate (SR) under random object motion disturbances, complementary to Figure~\ref{fig:traj_opt_random_dist}.}
        \vspace{\baselineskip}
        \label{tab:franka_d2g_sr}
        \setlength{\tabcolsep}{5pt}
        \renewcommand{\arraystretch}{1.2}
        \begin{tabular}{lcc}
            \toprule
            Method & D2G & SR \\
            \midrule
            Point     & $0.383 \pm 0.165$  & 4\%  \\
            Link      & $0.399 \pm 0.201$  & 4\%  \\
            WAContact & $0.308 \pm 0.208$  & 15\% \\
            WACaging  & $\textbf{0.003}\pm \textbf{0.018}$ & \textbf{87\%} \\
            \bottomrule
        \end{tabular}
    \end{minipage}
    \hfill
    \begin{minipage}[t]{0.68\textwidth}
        \centering
        \captionof{table}{Goal-reaching performance on the Franka robot under four contact-dynamics mismatches. Friction $\uparrow$ and $\downarrow$ indicate friction mismatches where the coefficient changes from $0.5$ to $0.8$ and $0.2$, respectively. Shape and MuJoCo correspond to shape and contact dynamics order mismatches. Final distance to goal values are reported, complementary to Figure~\ref{fig:traj_opt_mismatch_comp}.}
        \label{tab:franka_mismatch}
        \setlength{\tabcolsep}{5pt}
        \renewcommand{\arraystretch}{1.2}
        \begin{tabular}{lccccc}
            \toprule
            Method & Nominal & Friction$\uparrow$ & Friction$\downarrow$ & Shape & MuJoCo \\
            \midrule
            Point     & 0.373 & 0.421 & 0.217 & 0.416 & 0.207 \\
            Link      & \textbf{0.000} & 0.080 & 0.027 & 0.378 & 0.008 \\
            WAContact & \textbf{0.000} & 0.001 & 0.058 & 0.008 & 0.132 \\
            WACaging  & \textbf{0.000} & \textbf{0.000} & \textbf{0.001} & \textbf{0.002} & \textbf{0.002} \\
            \bottomrule
        \end{tabular}
    \end{minipage}

\end{figure*}
\clearpage
The improvement from \emph{Point} to \emph{Link} arises from increased contact area, while \emph{WAContact} further benefits from whole-arm geometric coverage. However, because \emph{WAContact} does not explicitly suppress escape directions, its advantage diminishes as the object gradually slides along the arm (as shown in Figure~\ref{fig:WAContact_nominal}), leading to final distance-to-goal values that are comparable to those of \emph{Link}. In contrast, \emph{WACaging} maintains consistently higher caging values throughout execution and demonstrates significantly stronger trajectory consistency under all mismatch scenarios. The executed trajectories remain closely aligned with the nominal plan, resulting in substantially lower final distance-to-goal errors compared to all baselines. 

These results indicate that increasing geometric engagement alone improves robustness to model inaccuracies, but explicitly maximizing the escape-time value provides an additional layer of protection by actively suppressing escape modes (\termref{def:quest3}{\textbf{Q3}}, \termref{def:quest4}{\textbf{Q4}}). 

The comparison between \emph{WAContact} and \emph{WACaging} is especially important: both variants allow full-arm contact under the same planner, contact model, and state--action space, but only \emph{WACaging} optimizes the escape-time objective. The performance gap therefore isolates the contribution of escape-time shaping beyond the robustness obtained from whole-arm contact geometry alone (\termref{def:quest4}{\textbf{Q4}}). This geometric stabilization improves open-loop execution robustness to contact dynamic mismatches, supporting the use of contact-dynamics simplifications in whole-arm manipulation planning.

\subsubsection{Whole-Arm Trajectory Tracking:}
\label{subsec:traj_tracking}
The goal-reaching experiments demonstrate that escape-time shaping improves robustness when task progress is compatible with increasing enclosure. In many manipulation tasks, however, the object must follow a prescribed reference trajectory. In such cases, maximizing the escape-time objective may conflict with minimizing tracking error (\termref{def:quest5}{\textbf{Q5}}). This experiment therefore investigates how escape-time shaping interacts with reference trajectory tracking and evaluates whether robust manipulation can still be achieved when enclosure must be partially relaxed at the beginning of manipulation.

\paragraph{Reference trajectories and objective formulation.}
We consider two reference object trajectories connecting the same initial and final configurations as in the goal-reaching experiment. The first is a left-side circular arc (yellow curve in Figure~\ref{fig:wm_tt_ref_obj_traj_ratios}), along which enclosure can increase monotonically during motion. The second is a right-side arc (green curve), which requires initially reducing enclosure to reposition the robot before pushing the object back toward the goal, as shown in Figure~\ref{fig:wm_tt_traj_snapshots_ratio100}. Unless otherwise stated, we focus on the right-side trajectory, as it induces a geometric conflict between enclosure growth and reference tracking. During trajectory optimization, we minimize the distance between the object and the reference trajectory while maximizing the learned escape-time value. We vary the weight ratio 
\[
\lambda = \frac{w_{track}}{w_{cage}},
\]
so that small $\lambda$ corresponds to enclosure-dominant planning and large $\lambda$ emphasizes tracking accuracy.

\begin{figure}[!t]
    \centering
    \begin{subfigure}{\linewidth}
        \centering
        \includegraphics[width=\linewidth]{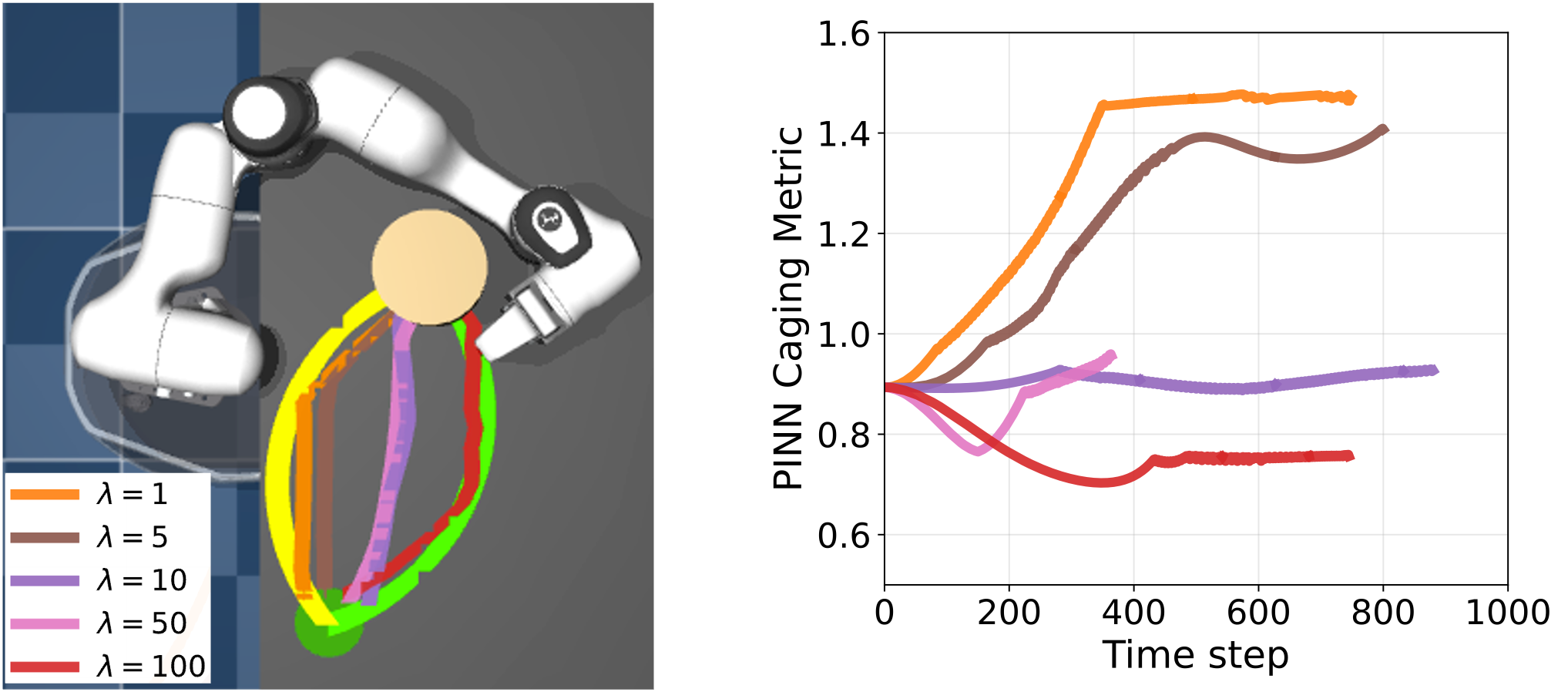}
        \caption{\emph{Left}: Two reference trajectories that connect the initial object position and the goal region center are shown with bright yellow and purple wider lines. Object's trajectories with different weight ratios $\lambda$ are shown with different colors. \emph{Right}: The figure shows the corresponding caging values over time for different ratios.}
        \label{fig:wm_tt_ref_obj_traj_ratios}
    \end{subfigure}
    \hfill
    \begin{subfigure}{\linewidth}
        \centering
        \includegraphics[width=\linewidth]{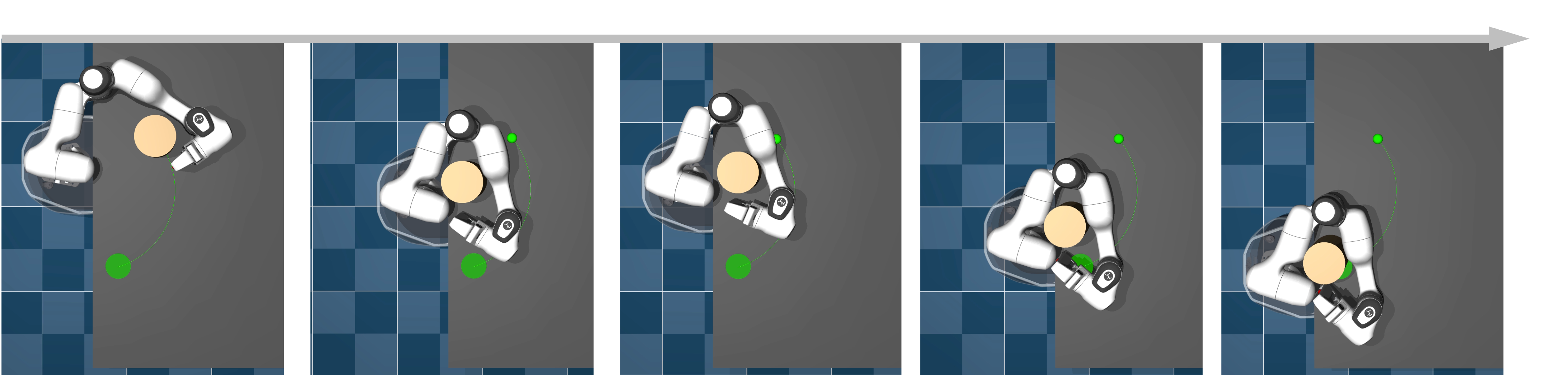}
        \caption{Robot and object configuration at keyframes when $\lambda = 1$}
        \label{fig:wm_tt_traj_snapshots_ratio1}
    \end{subfigure}
    \hfill
    \begin{subfigure}{\linewidth}
        \centering
        \includegraphics[width=\linewidth]{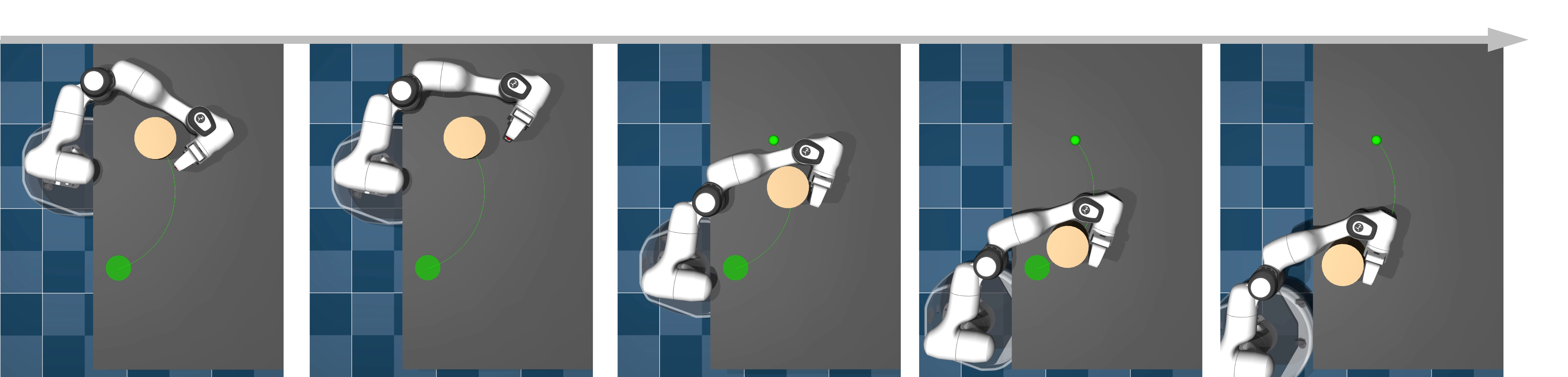}
        \caption{Robot and object configuration at keyframes when $\lambda = 100$}
        \label{fig:wm_tt_traj_snapshots_ratio100}
    \end{subfigure}
    \caption{Whole-arm trajectory tracking and caging performance analysis}
    \label{fig:wa_tt_performance_ana}
    \vspace{-0.5cm}
\end{figure}

\paragraph{Overall tracking and enclosure behavior.}
Figure~\ref{fig:wm_tt_ref_obj_traj_ratios} illustrates the resulting object trajectories and corresponding escape-time values for $\lambda \in \{1,5,10,50,100\}$.  For small $\lambda$ (e.g., $\lambda=1$ or $5$), the planner prioritizes enclosure and produces object paths that remain closer to the left-side arc as observed in the previous goal-reach task, despite the right-side arc being used as reference. The corresponding keyframes in Figure~\ref{fig:wm_tt_traj_snapshots_ratio1} show that the robot maintains a more wrapped configuration, preserving enclosure at the expense of reference fidelity.

As $\lambda$ increases, the object trajectory progressively aligns with the prescribed right-side path. For $\lambda=100$, the object closely follows the reference trajectory. The keyframes in Figure~\ref{fig:wm_tt_traj_snapshots_ratio100} reveal that the robot temporarily relaxes the enclosure early in the motion to reconfigure its pushing direction, thereby enabling accurate tracking. This improved tracking performance comes at the expense of enclosure margin, as shown in the right panel of Figure~\ref{fig:wm_tt_ref_obj_traj_ratios}. Increasing $\lambda$ systematically reduces the escape-time value along the trajectory. For large $\lambda$, the escape-time temporarily drops below its initial value, reflecting the need to reduce the enclosure to achieve the manipulability required for pushing along the conflicting reference path. 

\begin{figure}[!t]
    \centering
    \includegraphics[width=\linewidth]{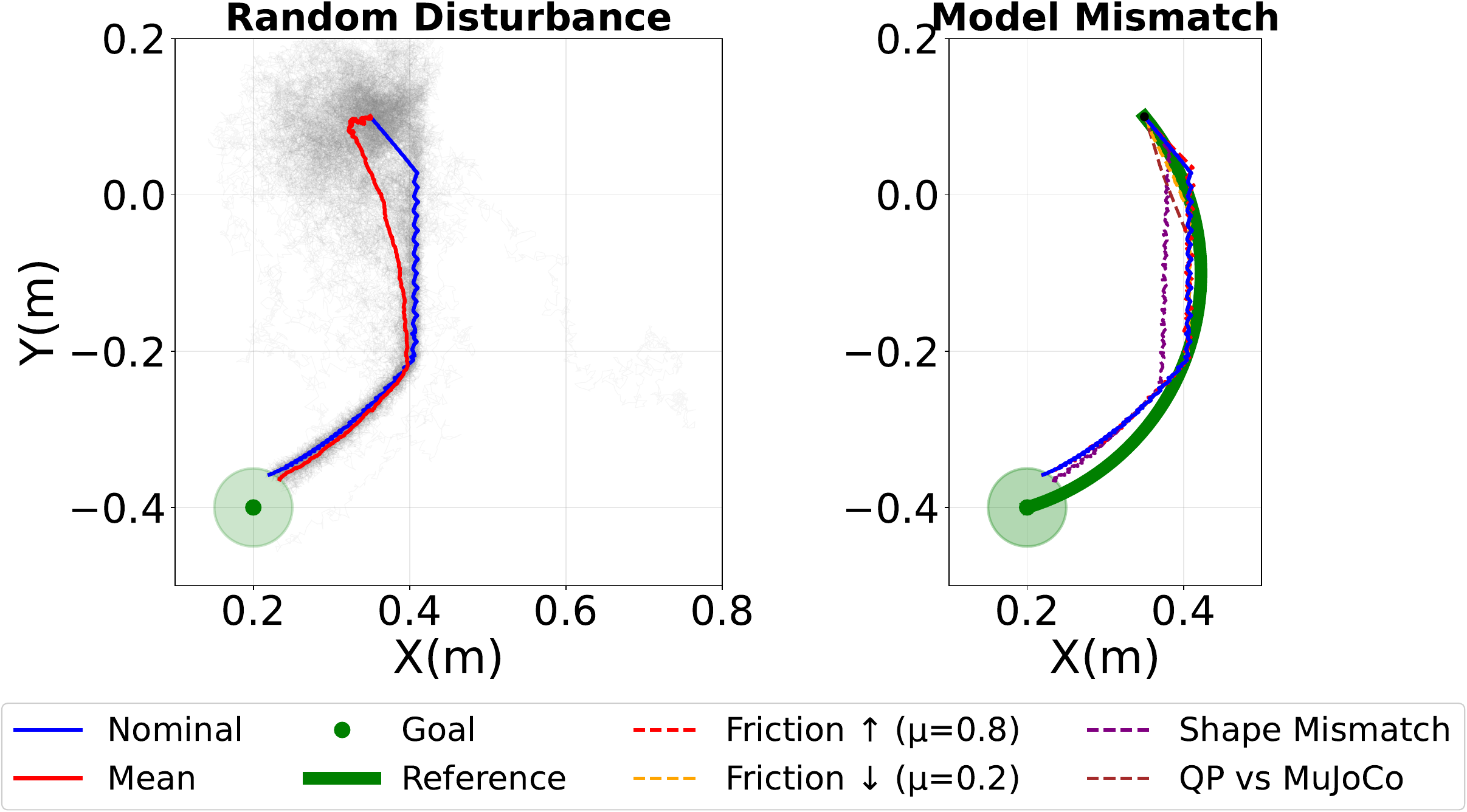}
    \caption{Evaluation of robustness with solution generated under $\lambda=100$. \emph{Left}: Robustness to random motion disturbance with magnitude $\sigma=0.005 \textrm{m}$. \emph{Right}: Robustness to four representative forms of contact model mismatches.}
    \label{fig:trajectory_level_robustness}
    \vspace{-0.5cm}
\end{figure}

\paragraph{Robustness analysis.}

We evaluate the robustness of the trajectory-tracking solution ($\lambda=100$) under both random motion disturbances and contact model mismatches, following the same experimental settings used in the goal-reaching experiments. The resulting trajectory distributions are shown in Figure~\ref{fig:trajectory_level_robustness}, using the same visualization conventions as the \textbf{WACaging} results reported in Figures~\ref{fig:traj_opt_random_dist} and~\ref{fig:traj_opt_mismatch_comp} for the goal-reaching experiments.

Compared with the goal-reaching task, the escape-time profile differs because the robot must temporarily reduce enclosure in order to adjust its pushing configuration. As shown in Figure~\ref{fig:wm_tt_ref_obj_traj_ratios}, the escape-time value decreases from approximately $0.9$ to $0.75$ during the initial phase of motion, whereas in the goal-reaching task the escape-time value gradually increases to around $1.0$. This reduced enclosure margin leads to a more dispersed distribution of object trajectories during the early phase under random disturbances, as seen in Figure~\ref{fig:trajectory_level_robustness} (Left). In several trials the object temporarily leaves the immediate manipulation region before the robot re-establishes contact.

Nevertheless, the overall escape-time value remains relatively high throughout the trajectory. Once the robot re-establishes contact with the object, the arm geometry again constrains the object motion and pushes it back toward the manipulation region. As a result, the disturbed trajectories gradually converge toward the nominal trajectory and follow the reference path toward the goal region. The final object positions therefore remain close to the goal, with a slightly larger variance compared with the goal-reaching experiment. 

A similar behavior is observed under contact-model mismatches. As shown in Figure~\ref{fig:trajectory_level_robustness} (Right), the executed trajectories remain bounded across all mismatch conditions and do not exhibit progressive drift away from the reference path. The largest deviation occurs under the shape mismatch condition, where the circular object is replaced with a box of edge length $0.212\,\mathrm{m}$. This introduces a geometric contact-distance discrepancy of approximately $0.031\,\mathrm{m}$ at the box corners, which is consistent with the maximal trajectory deviation observed in the experiments.

Taken together, these results indicate that when tracking and caging objectives compete, the robot temporarily relaxes enclosure in order to adjust its pushing configuration and satisfy the tracking objective. Although this reduces the escape-time margin compared with the goal-reaching task, the remaining enclosure remains sufficiently large to preserve robust manipulation behavior. As a result, object trajectories remain bounded under both disturbances and model mismatches, and the object can still be reliably pushed along the reference trajectory (\termref{def:quest5}{\textbf{Q5}}). 

These observations also suggest that incorporating caging-aware objectives during reference trajectory generation may further improve robustness in tasks where large configuration adjustments are required.

\subsection{Simulation-to-Real Transfer}\label{subsec:sim2real}

We then evaluate whether manipulation trajectories optimized using the simplified QP-based quasi-dynamic contact model transfer to real-world execution. The planned whole-arm caging trajectories are executed open-loop on the physical robot without replanning or feedback correction, allowing us to assess robustness to several simultaneous contact-model mismatches presented in the real system (\termref{def:quest5}{\textbf{Q6}}).

\begin{figure*}
    \centering
    \includegraphics[width=\linewidth, trim=0.3cm 6.9cm 0.3cm 2.5cm, clip]{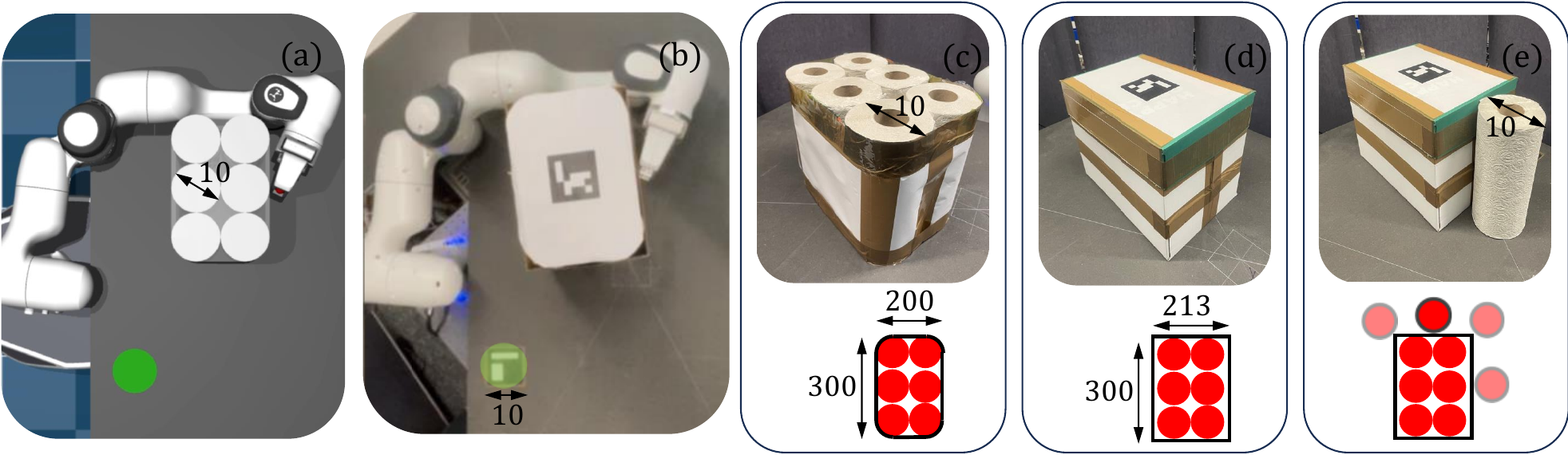}
    \caption{Real-world experiment setup with increasing levels of object-shape mismatch relative to the planning model. (a,b) Initial contact configuration in simulation and on the Franka robot. During planning, the object is approximated as six cylindrical primitives corresponding to the kitchen paper rolls (white cylinders in (a)). We considered three levels of object shape mismatch \textbf{(SM)} relative to the simulation model. \textbf{SM1}: packaged kitchen paper pack. \textbf{SM2}: the same pack inside an A4 paper box. \textbf{SM3}: the boxed object with an additional paper roll attached near the initial contact edges, producing the largest mismatch. The added roll is highlighted in (e). In all trials, the object is placed near the nominal simulated pose with small perception/calibration offsets. Units are $\mathrm{mm}$.}
    \label{fig:real_world_exp_setup}
    \vspace{-0.5cm}
\end{figure*}

Experiments are conducted using a fixed-base Franka Panda arm in a planar tabletop manipulation setup, as shown in Figure~\ref{fig:real_world_exp_setup} (a--b). The manipulated object is placed on the table at position $[0.4,\,0.2]\,\mathrm{m}$ with yaw angle $\frac{\pi}{2}$ and pushed toward a circular goal region (green area) centered at $[0.2,\,-0.2]\,\mathrm{m}$ using whole-arm caging manipulation. Object pose is estimated using an ArUco marker attached to the object and observed by an overhead RealSense D435 camera. The perception system provides an approximate position accuracy of about $6\,\mathrm{mm}$, while additional uncertainty may arise from marker placement on the object, resulting in small pose-estimation errors. The object is manually placed close to the nominal simulated initial pose but not perfectly aligned, introducing a small initial pose mismatch between simulation and reality.

The manipulation trajectory is generated in simulation using the whole-arm caging planner described in Section~\ref{subsec:wm_cage_mani_form}. During planning, robot–object interaction is modeled using the QP-based quasi-dynamic contact formulation in Equation~\eqref{eq:qp}. In the real world, however, the interaction follows the true second-order contact dynamics of the physical system with uncalibrated friction parameters. 

Moreover, we execute the same planned trajectory under three object shape mismatch scenarios to evaluate the robustness under increasing levels of geometric discrepancy. In simulation, the object is represented as the union of six cylindrical primitives corresponding to the individual paper rolls. In the real world, the rolls are packaged together as a single pack as shown in Figure~\ref{fig:real_world_exp_setup}(c), resulting in an initial shape mismatch (\textbf{SM1}). Next, the kitchen paper pack is placed inside a slightly larger A4 packaging box, introducing additional shape mismatch at the four corners of the box (\textbf{SM2}). Finally, an additional paper roll is placed near the A4 box, introducing larger shape mismatch and additional object–object contact that is not modeled during planning (\textbf{SM3}). These shape mismatches are visualized in Figure~\ref{fig:real_world_exp_setup}(c--e).

\begin{figure}[!t]
    \centering
    \includegraphics[width=\linewidth]{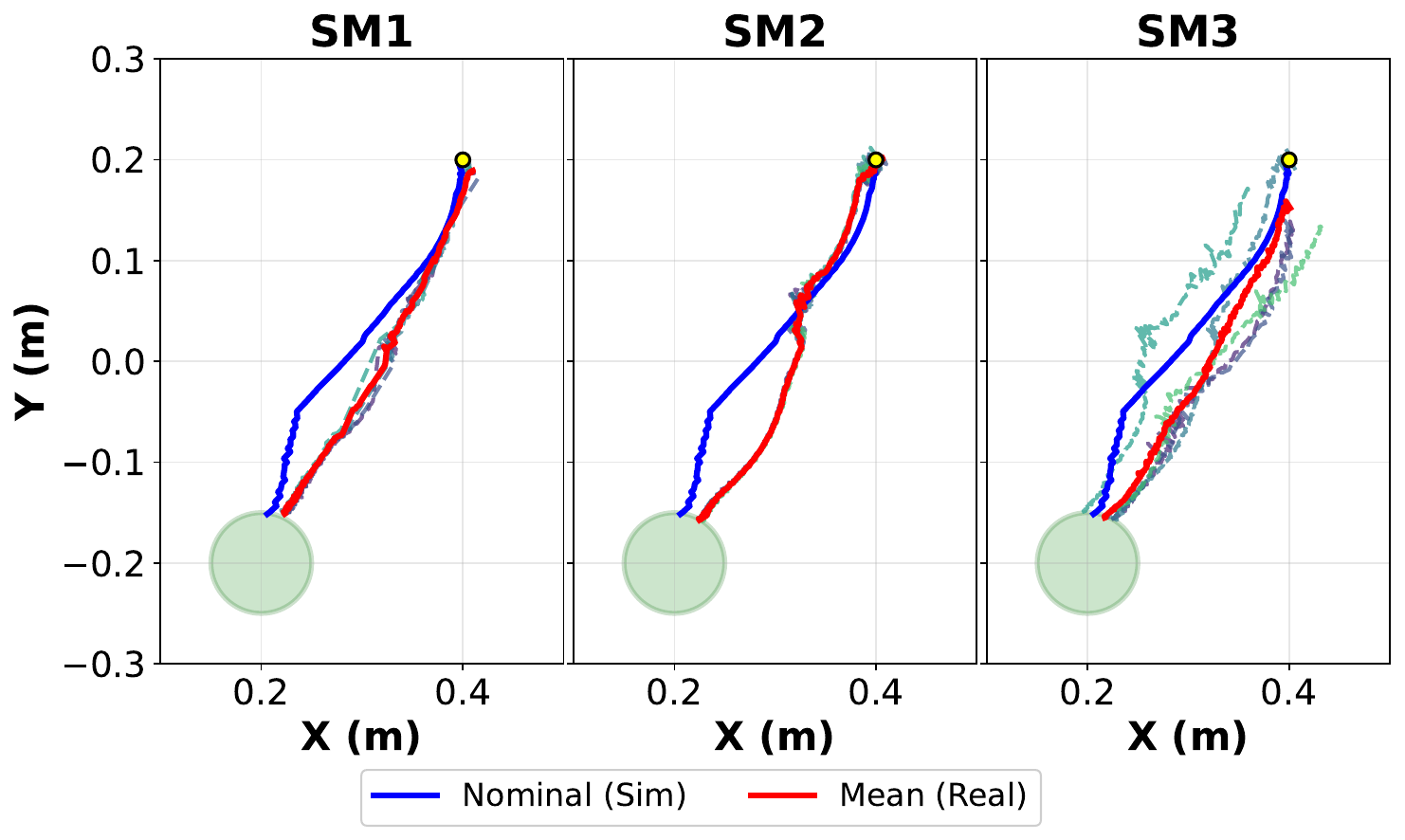}
    \caption{Object trajectories during sim-to-real transfer under three increasing levels of object-shape mismatch. Other execution conditions are kept unchanged across the three scenarios. The yellow dot marks the nominal initial object position used during simulation planning. For each case, the planned simulation trajectory is shown in blue, together with five real-world trials (thin dashed lines) and their mean trajectory (red). As mismatch increases, deviation from the nominal plan becomes more pronounced, enabling direct comparison of transfer robustness across scenarios.}
    \label{fig:real_world_obj_traj}
    \vspace{-0.5cm}
\end{figure}

Consequently, executing the planned trajectory in \textbf{SM1} already involves multiple contact-model mismatches, including initial object placement mismatch, quasi-dynamic contact-model approximation, and friction and shape mismatch. From \textbf{SM1} to \textbf{SM2} and \textbf{SM3}, other mismatches are kept and the level of shape mismatch is progressively increased to further evaluate the robustness of the planner. For each shape mismatch scenario (\textbf{SM1--SM3}), the same planned manipulation trajectory is executed open-loop on the real robot for five trials. No replanning or feedback correction is applied during execution. The resulting object trajectories are compared with the nominal trajectory observed using the QP-based quasi-dynamics contact model and shown in Figure~\ref{fig:real_world_obj_traj}.

For \textbf{SM1}, the object is pushed close to the goal boundary with a final distance of $4.1\,\mathrm{mm} \pm 2.1\,\mathrm{mm}$. Considering the perception uncertainty of approximately $6\,\mathrm{mm}$, we treat these trials as successful executions. When switching to the A4 paper box (\textbf{SM2}), the object is consistently pushed into the goal region across all trials. For \textbf{SM3}, we introduce an additional kitchen paper roll at the four configurations shown in Figure~\ref{fig:real_world_exp_setup}. The configuration highlighted in the figure is evaluated twice, while the remaining three trials correspond to the other configurations. The resulting trajectories are shown in Figure~\ref{fig:real_world_obj_traj}. Compared with \textbf{SM1} and \textbf{SM2}, larger deviations are observed in the initial object positions, and individual trials exhibit greater trajectory variation. Nevertheless, the mean object trajectory remains close to the nominal trajectory predicted by the planning model, and the object is consistently pushed into the goal region, similar to the results observed in \textbf{SM1} and \textbf{SM2}.

\begin{figure*}[!t]
    \centering
    \includegraphics[width=\linewidth, trim=5.9cm 4.4cm 5.9cm 4.4cm, clip]{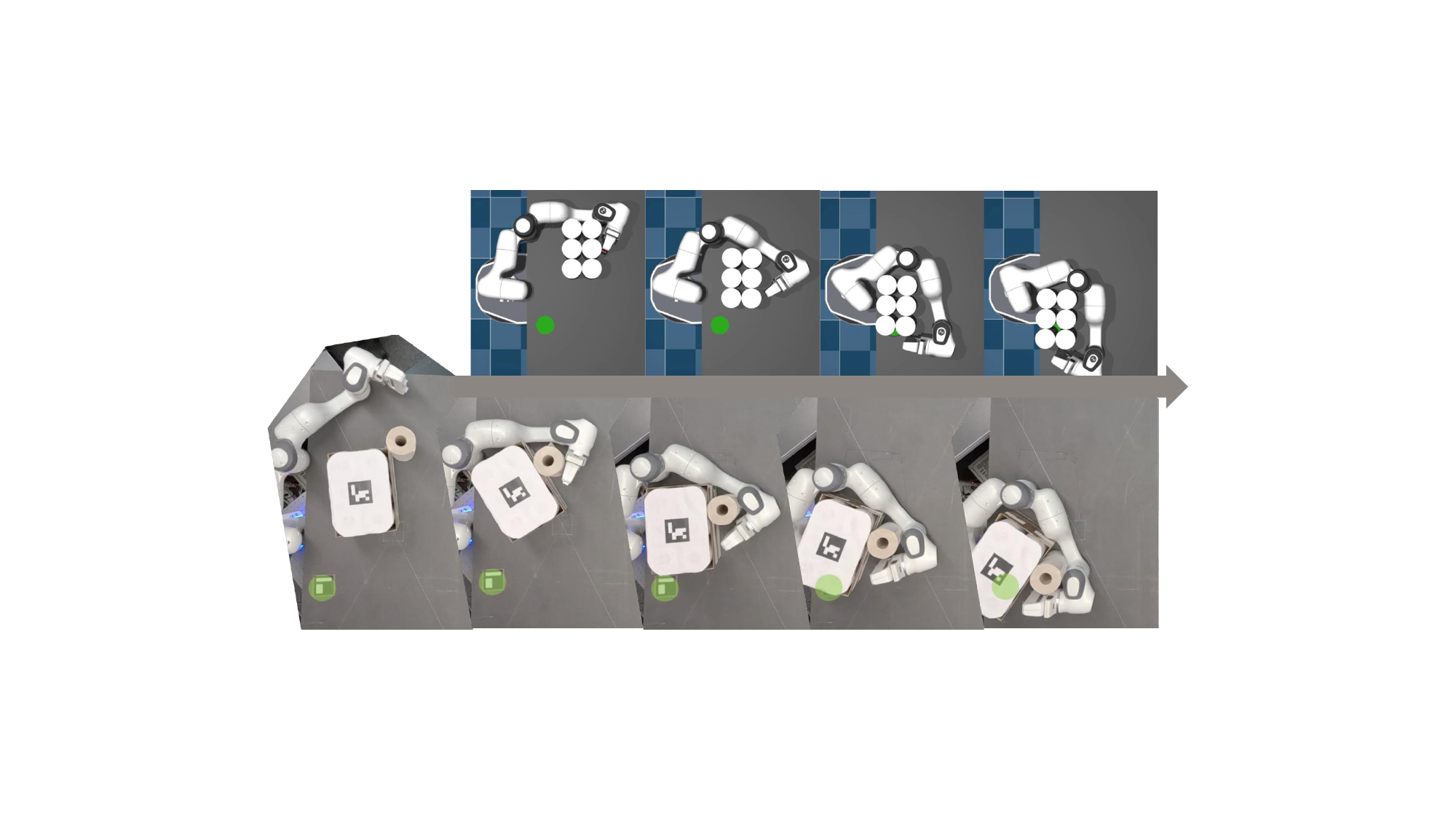}
    \caption{Snapshots of robot--object configurations in simulation and real-world execution under the most challenging mismatch setting (\textbf{SM3}: A4 paper box with an additional paper roll). The upper row shows the planning model used for whole-arm caging manipulation, while the lower row shows corresponding real-world execution. Despite multiple simultaneous contact-geometry mismatches, the transferred trajectory remains effective and successfully pushes the object into the goal region (green area), answering \textbf{Q6}.}
    \label{fig:real_world_snapshots}
    \vspace{-0.5cm}
\end{figure*}

To further illustrate the robot–object interaction during manipulation, Figure~\ref{fig:real_world_snapshots} shows representative snapshots of the planned trajectory in simulation with the six-cylinder contact geometry and the corresponding real-world execution under the \textbf{SM3} mismatch scenario. Additional results and videos can be found on the \href{https://sites.google.com/view/piec4wam?usp=sharing}{project webpage}.

Overall, these results indicate that whole-arm caging manipulation successfully handles the multiple contact-model mismatches encountered during sim-to-real transfer. The manipulation also remains robust under more severe geometric discrepancies and even when an additional object is introduced, which alters the contact geometry and creates object--object interactions not modeled during planning (\termref{def:quest6}{\textbf{Q6}}).

\section{Discussion}
The proposed escape-time formulation introduces a geometric objective for improving robustness in whole-arm manipulation. To make this framework practical for articulated robots and high-dimensional planning, we adopted several modeling and algorithmic simplifications. This section discusses these design choices, their implications, current limitations, and promising directions for future work.

\subsection{Physics-Informed Eikonal Caging}
\paragraph{Point-based Object Representation.} In this work, escape-time reasoning uses a point-based object approximation to obtain a stable and tractable formulation. Richer object models, such as bounding boxes or parametric shape descriptors, could be incorporated by lifting collision constraints into the joint robot--object configuration space. The same eikonal framework would then apply to the augmented state representation.

\subsection{Whole-Arm Caging Manipulation}\label{subsec:franka_cage_mani}
\paragraph{Escape-Time Metric versus Heuristics.}
The escape-time objective provides a principled alternative to local proximity-based metrics for whole-arm manipulation. Our experiments compare it with representative heuristic objectives to illustrate the limitations of purely local formulations. We do not claim that escape-time is the only effective metric; future work may develop simpler objectives that approximate enclosure reasoning while retaining similar benefits.

\paragraph{Optimizer Consideration.} The learned escape-time field is differentiable and can in principle be optimized with gradient-based methods. We demonstrated this in simplified configuration planning. However, full manipulation planning additionally couples the objective with hybrid contact dynamics. For the realistic Franka experiments, CMA-ES provided more reliable solutions under the available computational budget. Developing scalable differentiable contact planners that more fully exploit escape-time gradients remains an important direction for future work.

\paragraph{Min–max game for robust control.}
In the illustrative example in Section~\ref{subsec:caging_manipulation_example}, we also explored a stage-wise adversarial variant in which the object was modeled as an agent with bounded escape velocity. At each time step, the object’s action followed the gradient of the learned escape-time field, and the resulting disturbance was explicitly injected into the contact dynamics. This formulation is conceptually closer to differential game–based approaches, while remaining a single-player trajectory optimization under a fixed disturbance policy. Empirically, we found that explicitly simulating such adversarial object dynamics significantly increased planning time and numerical sensitivity, while producing manipulation behaviors qualitatively similar to those obtained using the escape-time objective alone (Appendix~\ref{app:gt_wm_cage_mani}). These results suggest that, for the evaluated tasks, most of the robustness benefits arise from the geometry-based escape-time field itself, rather than from online adversarial rollout. Motivated by this trade-off, we adopt the simpler formulation in this work and leave the development of more stable and scalable game-theoretic extensions to future work.

\paragraph{Relation to Other Robust Planning Methods.}
Belief-space planning, robust model predictive control, and reinforcement learning typically improve robustness through explicit uncertainty modeling, feedback correction, or policy adaptation. These approaches are complementary to our method. Rather than representing uncertainty as a belief distribution or optimizing over a prescribed set of parameter variations, our approach introduces a geometric robustness primitive: it shapes the robot configuration so that object escape becomes difficult under bounded worst-case motion. This is particularly relevant for whole-arm manipulation, where contact-model mismatch can arise from multi-contact mode transitions, geometry simplification, friction errors, and unmodeled object motion, which are difficult to parameterize exhaustively.

Accordingly, the baselines in this work are designed as objective and contact-geometry ablations rather than exhaustive comparisons to all robust-planning paradigms. The contact prior, proximity, link-contact, and whole-arm contact baselines isolate whether the observed robustness comes from the proposed escape-time objective, rather than merely from maintaining contact, reducing local distance, increasing contact area, or allowing whole-arm interaction. In particular, the \emph{WAContact} baseline uses the same full-arm contact geometry as our method but removes the escape-time objective, directly testing whether whole-arm geometry alone is sufficient. The results show that explicit escape-time shaping provides additional robustness beyond whole-arm contact alone. Combining the proposed escape-time objective with feedback control or belief-space planning is a promising direction for more complex settings such as occlusion, clutter, multi-object manipulation, and larger disturbances.

\subsection{Limitations}
\paragraph{Primarily Open-Loop Execution.}
The experiments in this work rely mainly on open-loop trajectory execution. This design isolates the geometric robustness induced by the escape-time objective: the improved execution performance cannot be attributed to online replanning or feedback correction. However, the proposed formulation does not provide a formal closed-loop robustness guarantee for arbitrary disturbances, sensing errors, or contact-model mismatch. The escape-time field should instead be interpreted as a geometric robustness margin under the bounded escape model, which empirically improves tolerance to the representative mismatches evaluated in this paper. Belief-space planning and feedback control remain important under perception uncertainty or large disturbances. Combining escape-time shaping with model predictive control, differentiable simulation for closed-loop policies is a promising direction for further improving robustness in more complex manipulation settings.

\paragraph{Planar Task Scope.}
The evaluated tasks are primarily planar pushing scenarios. Extending the framework to three-dimensional manipulation is conceptually natural for geometry-based caging, but requires richer escape dynamics than the isotropic eikonal formulation used in this work. In 3D settings, object escape depends on gravity, resulting in direction-dependent mobility. One promising direction is to replace the isotropic escape model with an anisotropic eikonal formulation \citep{mirebeau2014anisotropic, chen2016new}, where the local wave propagation speed depends on both configuration and direction. For example, the escape speed could vary with the angle between the object motion direction and gravity, assigning different costs to motions along, against, and perpendicular to gravity. This would allow the escape-time field to encode gravity-aware resistance to object escape. Recent work on neural eikonal fields suggests that such metric-dependent time fields can be approximated with PINNs \citep{Li26IJRR}, making this a promising path toward gravity-aware 3D whole-arm caging manipulation. We leave the systematic design and evaluation of such models to future work.

\section{Conclusion}
This work introduced escape-time shaping as a geometric robustness primitive for whole-arm manipulation planning. By formulating caging as a minimum-time escape problem and leveraging its eikonal characterization, we obtained a continuous and differentiable escape-time field that quantifies resistance to worst-case object escape and can be optimized directly in manipulation planning. We demonstrated that incorporating this escape-time objective into whole-arm manipulation planning systematically biases the robot toward configurations that suppress escape directions. In simulation and real-world experiments, this geometric bias translated into improved robustness under representative contact-model mismatches and disturbances, which can be further exploited to simplify the whole-arm contact dynamics. Overall, the results suggest that reasoning about geometric escape in configuration space provides a principled and practical way to induce robustness in contact-rich whole-arm manipulation.

{\small 
\section*{Acknowledgements}
This work is supported by the State Secretariat for Education, Research and Innovation (SERI), Switzerland, for participation in the European Commission’s Horizon Europe Programme through the INTELLIMAN project (HORIZON-CL4-Digital-Emerging Grant 101070136). 
The work is also partly funded by the European Commission under the Horizon Europe Framework Program project SoftEnable (Grant 101070600).
The authors thank Ruiqi Ni for his valuable discussion on eikonal equations and physics-informed neural networks.
}

{\small 
\section*{Declaration of conflicting interests} 
The authors declared no potential conflicts of interest with respect to the research,
authorship, and/or publication of this article.
}

\bibliographystyle{SageH}
\bibliography{Sage}

\appendix

\section{Implementation Details}
\subsection{Whole-Arm Caging Configuration Planning}\label{app:wm_caging_opt}

\subsubsection{Physics-informed Eikonal Caging}\label{app:pinn}
We use the same PINN training pipeline for both planar-arm and Franka experiments; the two setups differ mainly in input dimensionality, boundary geometry, and sampling parameters.

\paragraph{Network architecture.}
The escape-time field is represented by a multi-layer perceptron (MLP) that takes as input the concatenated vector $(\bm{p}, \bm{q}_{r})$ of object key points $\bm{p}$ and robot configuration $\bm{q}_r$, with total input dimension $d_{\bm{p}} + d_{\bm{q}_{r}}=6$ for the planar toy arm and typically $d_{\bm{p}} + d_{\bm{q}_{r}}=5$ for the Franka arm. The MLP has depth $6$ and width $256$ and outputs a scalar value passed through a Softplus \citep{dugas2000incorporating} activation to enforce non-negativity. 

\paragraph{Speed field from robot SDF.}
In the main text, the eikonal equation is written with a constant maximal
escape speed \(u_{\max}\). In implementation, we use a spatially varying
maximal speed field \(u_{\max}(\bm p,\bm q_r)\) induced by the robot signed
distance function
\(d=\phi_r(\bm p,\bm q_r)\). This field slows down wave propagation near
the robot surface and blocks propagation inside collision regions.

We define a normalized speed attenuation factor
\begin{equation}
    \alpha(d)
    =
    \alpha_{\min}
    +
    (1-\alpha_{\min})
    \sigma\!\big(s(d-\tau)\big),
\end{equation}
where \(\sigma(\cdot)\) is the logistic sigmoid. In our experiments, we set
\(\alpha_{\min}=0.01\), \(\tau=0.05\), and \(s=40\). The local maximal escape
speed is then given by
\begin{equation}
u_{\max}(\bm p,\bm q_r)
=
\bar u_{\max}
\begin{cases}
0, & d \le d_0,\\[2pt]
\max\!\big(\alpha(d),\,\alpha_{\mathrm{floor}}\big), & d>d_0,
\end{cases}
\end{equation}
where \(\bar u_{\max}\) is the nominal free-space escape speed,
\(d_0=0.03\) is the near-collision cutoff, and
\(\alpha_{\mathrm{floor}}=0.001\) prevents numerically vanishing speeds in
near-contact but collision-free regions.

Accordingly, the physics-informed eikonal residual used during training is
\begin{equation}
    \mathcal L_{\mathrm{pde}}
    =
    \left(
    \left\|\nabla_{\bm p} T_p(\bm p,\bm q_r)\right\|
    -
    \frac{1}{u_{\max}(\bm p,\bm q_r)}
    \right)^2 .
\end{equation}
When \(u_{\max}(\bm p,\bm q_r)=\bar u_{\max}\) everywhere in free space,
this reduces to the constant-speed eikonal equation used in the main text.
During training, samples are restricted to collision-free points with a
small clearance by enforcing
\(\phi_r(\bm p,\bm q_r)\ge 0.01\).

\paragraph{Sampling strategy.}
At each training iteration, we sample a mini-batch of robot configurations and then draw training points from two streams: two from the interior of the collision-free space and one from the workspace boundary. The first interior stream samples collision-free points in a thin band near the robot surface, where the geometry is most restrictive and accurate gradients of the escape-time field are most important. The second interior stream samples from the entire collision-free workspace to encourage consistency of the PDE away from obstacles and to reduce bias toward only near-robot regions. In our experiments, we use a \(0.99/0.01\) split between the near-band and full-workspace streams. Additionally, boundary points are sampled on the outer workspace boundary and assigned \(\mathcal{T}_p=0\). This fixes the value of the time field on the boundary, providing a consistent zero-time reference.

\paragraph{Training details.}
We train the neural network with two loss terms: a PDE residual loss and a boundary loss, weighted by \((\lambda_{\mathrm{pde}},\lambda_{\mathrm{bc}})=(10,1)\).
Optimization is performed using Adam with a learning rate \(0.001\), and we apply a step scheduler that multiplies the learning rate by \(0.5\) every 2000 iterations.
Training runs for \(2\times10^4\) iterations.

\subsubsection{Optimizers: CMA-ES and SQP}\label{app:optimizers_implement}

\paragraph{CMA-ES.}
In Sections~\ref{subsec:example_wm_caging} and~\ref{subsec:exp_wm_caging_conf}, we use CMA-ES to solve a static whole-arm caging configuration problem. The objective is to find a robot configuration that maximizes the learned escape-time field while penalizing collisions between the robot arm and the target object. The optimization objective is given by 
\[c = -\hat{\mathcal{T}}(\bm{q}_o^{0}, \bm{q}_r) + \omega_{\mathrm{col}} \sum_{i=1}^{K} \phi_r(\bm{p}_i, \bm{q}_r),
\]
where $\bm{q}_o^{0}$ denotes the fixed initial object configuration, $\bm{q}_r$ is the robot joint configuration, $\phi_r$ is the robot signed distance field evaluated at object key points $\{\bm{p}_i\}_{i=1}^{K}$, and $\omega_{\mathrm{col}} = 10^{3}$ is a large penalty weight encouraging collision-free configurations. We use a population size of $100$ and run CMA-ES for $100$ iterations for both planar and Franka-arm caging configuration optimization.

\paragraph{SQP.}
In Section~\ref{subsec:example_wm_caging}, we additionally solve the static whole-arm caging configuration problem using sequential quadratic programming (SQP). The problem is formulated as the minimization of the negative escape-time objective subject to collision-free constraints and joint-limit box constraints, following Equation~Equation~\eqref{eq:wm_caging_lb}. Collision avoidance is enforced using the robot's SDF evaluated at each object key point. At each SQP iteration, constraints are linearized and a quadratic program of the form
\[
    \min_{\Delta}\ \tfrac{1}{2}\Delta^\top \bm{H}\,\Delta + \bm{g}^\top \Delta,
\quad \text{s.t. } \bm{G}\Delta \le \bm{h},
\]
is solved, where $\bm{g}$ is the gradient of the objective and $\bm{H}$ is a Gauss--Newton approximation given by $\bm{H} = \bm{J}^\top \bm{J} + 10^{-2}\mathbf{I}$. For the planar arm, constraint gradients are obtained analytically from the capsule-based SDF representation of each link, yielding one constraint per link--object pair. The robot configuration $\bm{q}_r$ is updated using an adaptive step size: if the maximum constraint violation exceeds $10^{-3}$, the step size is reduced by a factor of $0.8$ (down to a minimum of $0.1$); if it is below $10^{-6}$, the step size is increased by a factor of $1.1$ (up to a maximum of $1.0$). Optimization terminates when $\|\Delta\| < 10^{-2}$ or after $100$ iterations.

\subsection{Whole-Arm Caging Manipulation}
\label{app:wm_caging_mani}

\subsubsection{Trajectory Parameterization and Optimization}

For the caging manipulation tasks in Sections~\ref{subsec:caging_manipulation_example}  and~\ref{subsec:mobile_franka_cage_mani}, robot trajectories are parameterized using $5$ key points over a planning horizon of $1000$ time steps, following the via-point representation in~\cite{Jankowski2022vp}. The full trajectory is obtained by interpolating between key points while enforcing joint-level kinematic constraints. Maximum velocity limits are set to $0.05\,\mathrm{rad/s}$ for arm joints (and $0.05\,\mathrm{m/s}$ for the mobile base).

Trajectory optimization is performed using CMA-ES. For each planning task, CMA-ES is run for $100$ iterations, and each iteration evaluates $50$ sampled trajectories under the same quasi-dynamic contact model. The best-performing trajectory across all iterations is selected for execution and evaluation $c_{prior}$. All methods share the same trajectory parameterization, optimizer, contact model, and planning horizon. The only difference between baselines and the proposed method lies in the contact-related objective term. All whole-arm caging manipulation experiments optimize the objective
\[
c = \sum_{t=0}^{1000} \Big(
    w_{\text{task}}\, c_{\text{task}}(t)
    + w_{\text{cmd}}\, c_{\text{cmd}}(t)
    + w_{\text{prior}}\, c_{\text{prior}}(t)
    \Big),
\]
where:
\begin{itemize}
    \item $c_{\text{task}}(t)$ encodes the task objective, such as goal reaching or reference trajectory tracking;
    \item $c_{\text{cmd}}(t) = \|\bm{u}_r^t\|^2$ penalizes control effort;
    \item $c_{\text{prior}}(t)$ denotes either the contact-prior term used by the baselines or the caging-related term $-\hat{\mathcal{T}}(\bm{q}_o^t, \bm{q}_r^t)$ derived from the learned escape-time metric (proposed method).
\end{itemize}
Unless otherwise specified, $w_{\text{task}}=10^2$, $w_{\text{cmd}}=10^3$, and $w_{\text{prior}}=10^2$ 
are used for trajectory optimization experiments.

For the trajectory tracking task in Section~\ref{subsec:franka_cage_mani}, the same trajectory parameterization, optimizer, and hyper-parameters are used. The only modification concerns the task term $c_{\text{task}}$. In trajectory optimization (goal-reaching), we define the task objective as
\[
c_{\text{task}}(t) = \|\bm{q}_o^t - \bm{q}_o^g\|^2,
\]
the squared distance from the current object position to the center of the goal region. In trajectory tracking, instead of penalizing distance to a fixed goal, we define
\[
c_{\text{task}}(t) 
    = d\big(\bm{q}_o^t, \bm{q}_o^{\mathrm{ref}}\big),
\]
where $d(\cdot,\cdot)$ denotes the shortest Euclidean distance from the current object position $\bm{q}_o^t$ to the reference trajectory $\bm{q}_o^{\mathrm{ref}}$. The reference trajectory is discretized into $50$ interpolated points, and the minimum distance to these points is used as the tracking error.

\subsubsection{Contact Dynamics}
\paragraph{QP-based quasi-dynamic contact model.}
For the caging manipulation tasks in Sections~\ref{subsec:caging_manipulation_example} and~\ref{subsec:mobile_franka_cage_mani}, we employ a QP-based quasi-dynamic contact model during planning. Unless otherwise specified, the model uses a Coulomb friction coefficient of $0.5$, a discrete planning time step of $0.1\,\mathrm{s}$, and a diagonal mass matrix $\mathbf{M} = \mathrm{diag}([0.1,\,0.1,\,0.1,\,0.1,\,0.1,\,0.042])$. Mass matrix regularizer $\epsilon$ is set to be $10$ for stable simulation of multi-contact interaction. The primary difference between the planar illustrative example and the mobile Franka experiments lies in the joint-space stiffness parameters. In the planar illustrative example (Section~\ref{subsec:caging_manipulation_example}), we use a stiffness matrix $\mathbf{K} = \mathrm{diag}([500,\,500,\,500,\,500])$, whereas in the mobile Franka experiments (Section~\ref{subsec:mobile_franka_cage_mani}) we use higher stiffness values $\mathbf{K} = \mathrm{diag}([10^{4},\,\ldots,\,10^{4}])$ for the actively controlled joints. Contact Jacobians are computed using MuJoCo's collision detection and contact geometry interfaces, ensuring consistent contact point locations and surface normals between planning and simulation.

\paragraph{MuJoCo forward simulation.}
For forward simulation and execution, we use MuJoCo with its default simulation time step of $0.002\,\mathrm{s}$. The Franka arm is controlled using a position controller whose stiffness parameters match those used in the corresponding QP-based planning model. Each planned trajectory waypoint generated by the QP model is executed over $50$ MuJoCo simulation steps, allowing the system to converge to the target configuration before advancing to the next waypoint.

\section{Additional Experiment Results}

\subsection{Stage-wise Adversarial Whole-Arm Caging Manipulation}
\label{app:gt_wm_cage_mani}

For the toy whole-arm caging manipulation example in Section~\ref{subsec:caging_manipulation_example}, we additionally evaluated a stage-wise adversarial variant that is closer in spirit to differential game formulations, by introducing an explicit object escape velocity into the robot--object contact dynamics. In this variant, the object is modeled as an adversarial agent that attempts to escape at each time step, while the robot optimizes the same task-level objective as in the main planner.

Specifically, we modify the contact dynamics in~Equation~\eqref{eq:wm_caging_mani_dyn} by injecting a bounded escape velocity
\[
g(\bm{q}_r^t, \bm{q}_o^t) = -u_{\max} \, \nabla_{\bm{q}_o^t} \hat{\mathcal{T}}(\bm{q}_o^t, \bm{q}_r^t),
\]
which corresponds to the steepest descent direction of the learned escape-time field with respect to the object configuration conditioning on the current robot configuration. This induces a closed-loop, stage-wise adversarial interaction: after each robot action, the object reacts according to a fixed escape policy derived from the inner value function, while the outer planner remains a single-player trajectory optimization under this disturbance model. All other planner parameters are kept identical to those used in the main paper.

\begin{figure}[!t]
    \centering
    \begin{subfigure}{\linewidth}  
        \centering
        \includegraphics[width=\linewidth]{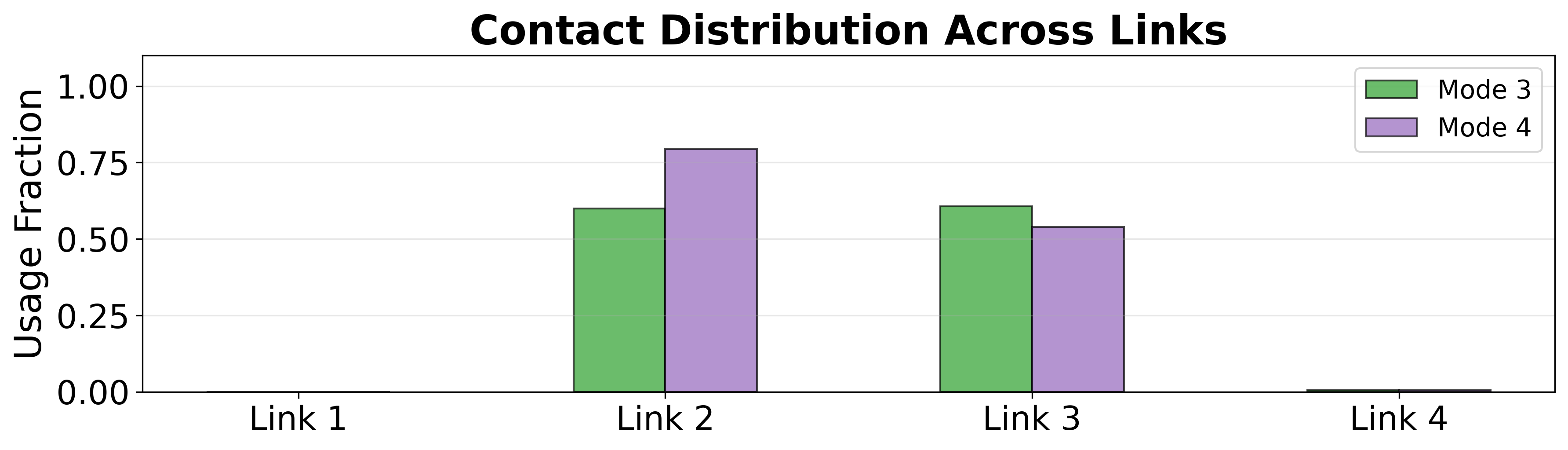}
        \caption{Fraction of time each arm link remains in contact with the object. }
        \label{subfig:arm_utilization_game}
    \end{subfigure}
    \begin{subfigure}[t]{0.48\linewidth} 
        \centering
        \includegraphics[width=\linewidth]{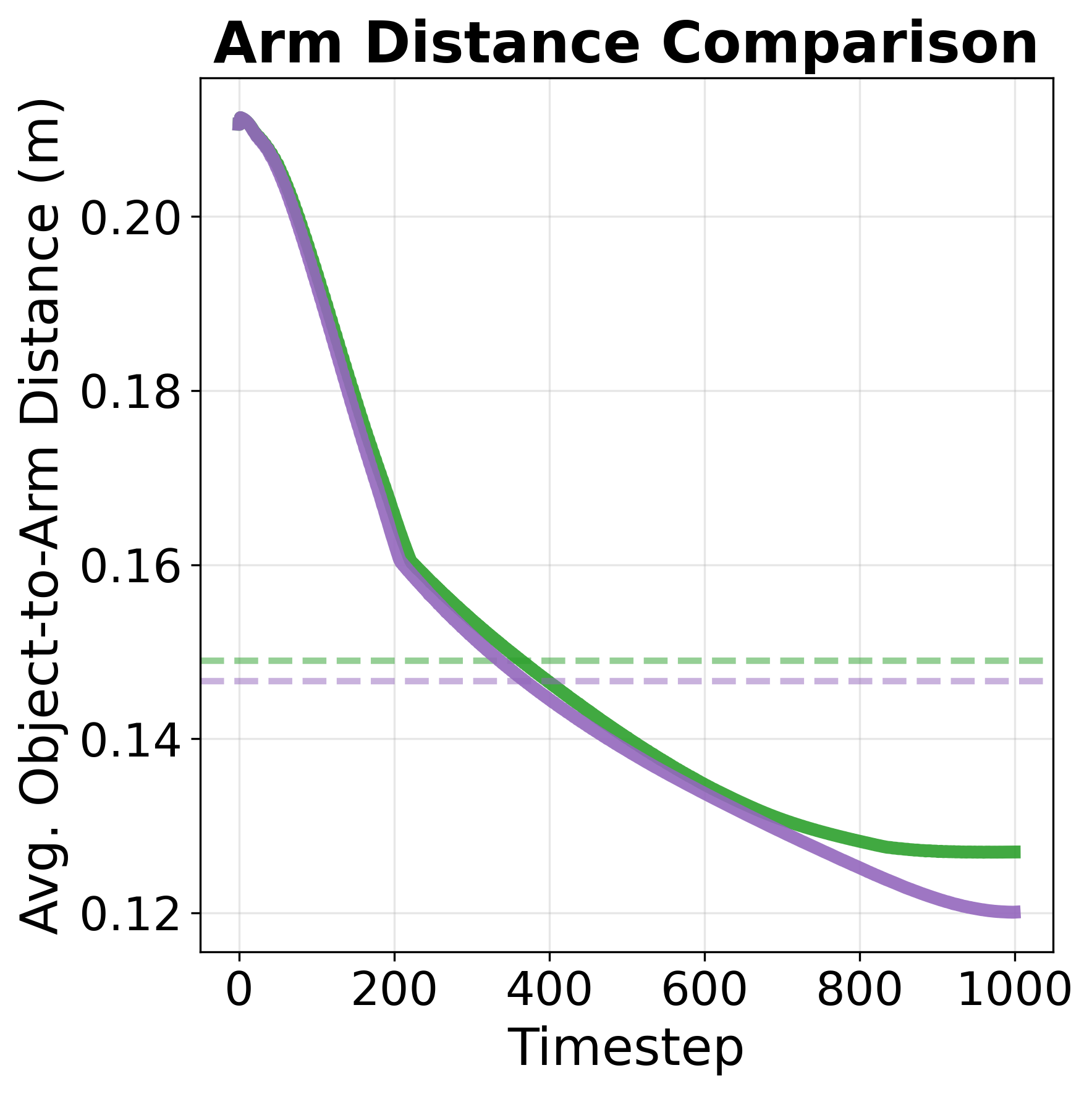}
        \caption{Average object–arm distance over time. }
        \label{subfig:arm_dist_comp_game}
    \end{subfigure}
    \hfill
    \begin{subfigure}[t]{0.48\linewidth} 
        \centering
        \includegraphics[width=\linewidth]{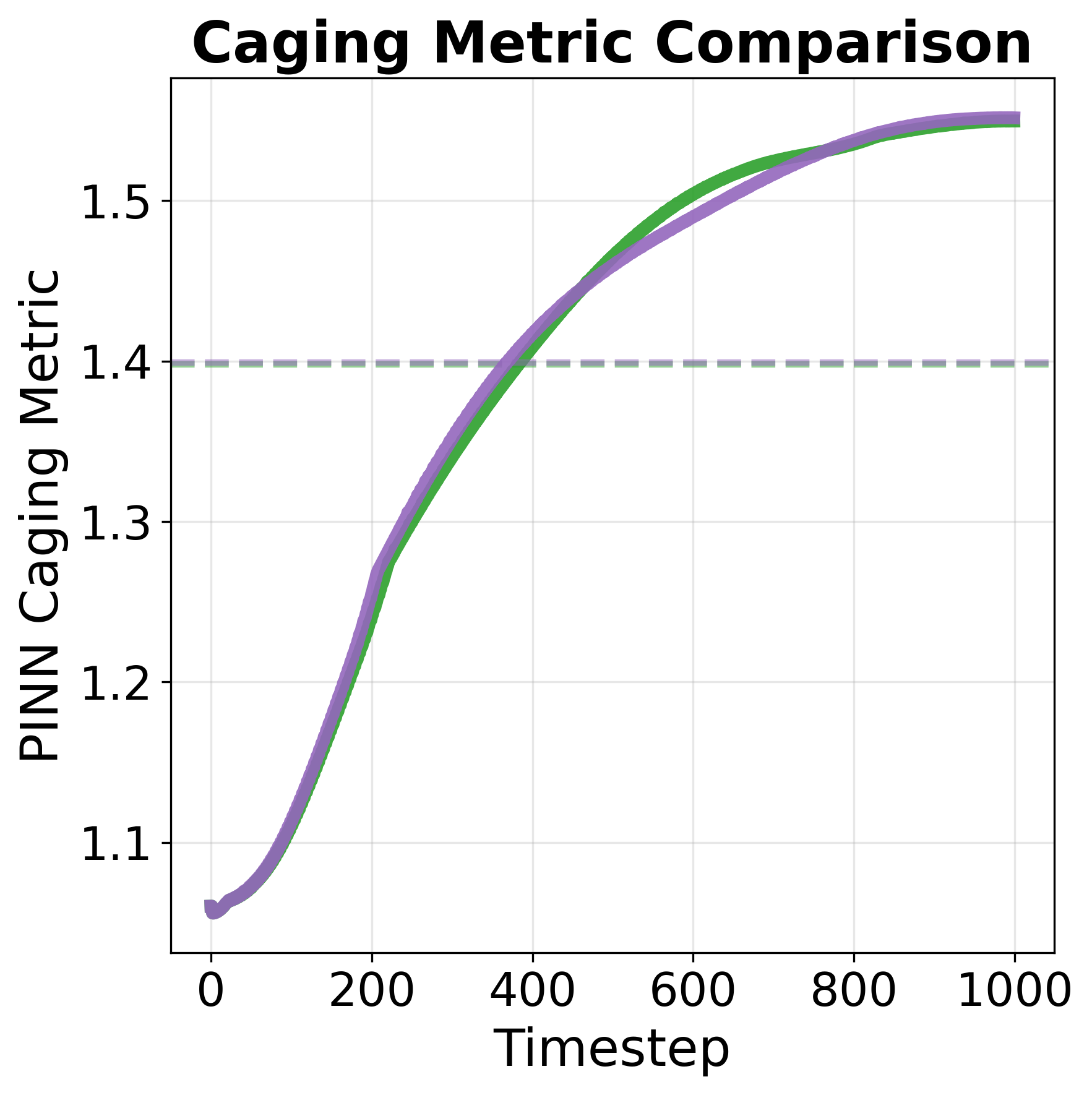}
        \caption{Escape-time (caging) metric over time.}
        \label{subfig:cage_metric_comp_game}
    \end{subfigure}

    \caption{Contact distribution, proximity diagnostics, and enclosure quality for the three planning modes.}
   \label{fig:contact_distr_cage_metric_game}
   \vspace{-0.5cm}
\end{figure}

Figure~\ref{fig:contact_distr_cage_metric_game} compares arm utilization, average object--arm distance, and the learned caging metric over time between the planner used in the main paper (Mode~3) and the stage-wise adversarial variant (Mode~4). We observe that explicitly injecting the escape velocity biases the robot toward increased use of the second link and reduced reliance on the third link, resulting in slightly smaller average object--arm distance and marginally higher escape-time values. These effects indicate a tighter geometric enclosure, consistent with the adversarial objective.

\begin{figure}[!t]
    \centering
    \includegraphics[width=\linewidth]{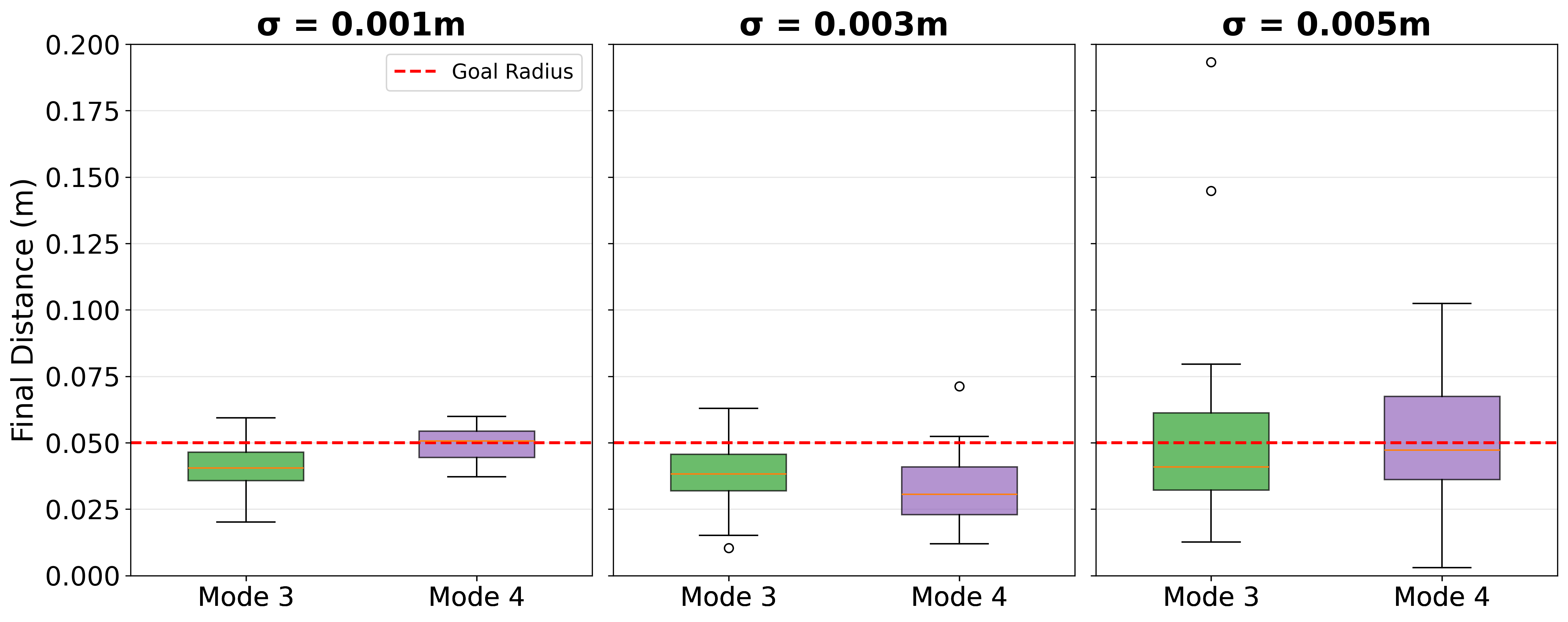}
    \caption{Object's distance to goal at final time step.}
    \label{fig:robustness_detailed_final_dist_game}
    \vspace{-0.5cm}
\end{figure}
Figure~\ref{fig:robustness_detailed_final_dist_game} further evaluates task-level robustness by measuring the final object distance to the goal under additive random motion disturbances of varying magnitudes. The adversarial variant exhibits similar robustness with our used planner under evaluated disturbances.
However, in practice, we observed that explicitly injecting adversarial object velocities increases numerical sensitivity during trajectory optimization, leading to larger variance in CMA-ES rollouts and slower convergence. This behavior is consistent with the non-smooth interaction between contact dynamics and adversarial feedback. 

Taken together, these observations support our design choice in the main paper: the geometry-based escape-time objective captures the dominant robustness effects, while explicit adversarial rollout introduces additional complexity and sensitivity without commensurate gains in overall task performance.

\subsection{Physics-Informed Eikonal Caging for the Franka Arm}\label{app:piec_franka}
\begin{figure}[!t]
    \centering
    \includegraphics[width=\linewidth]{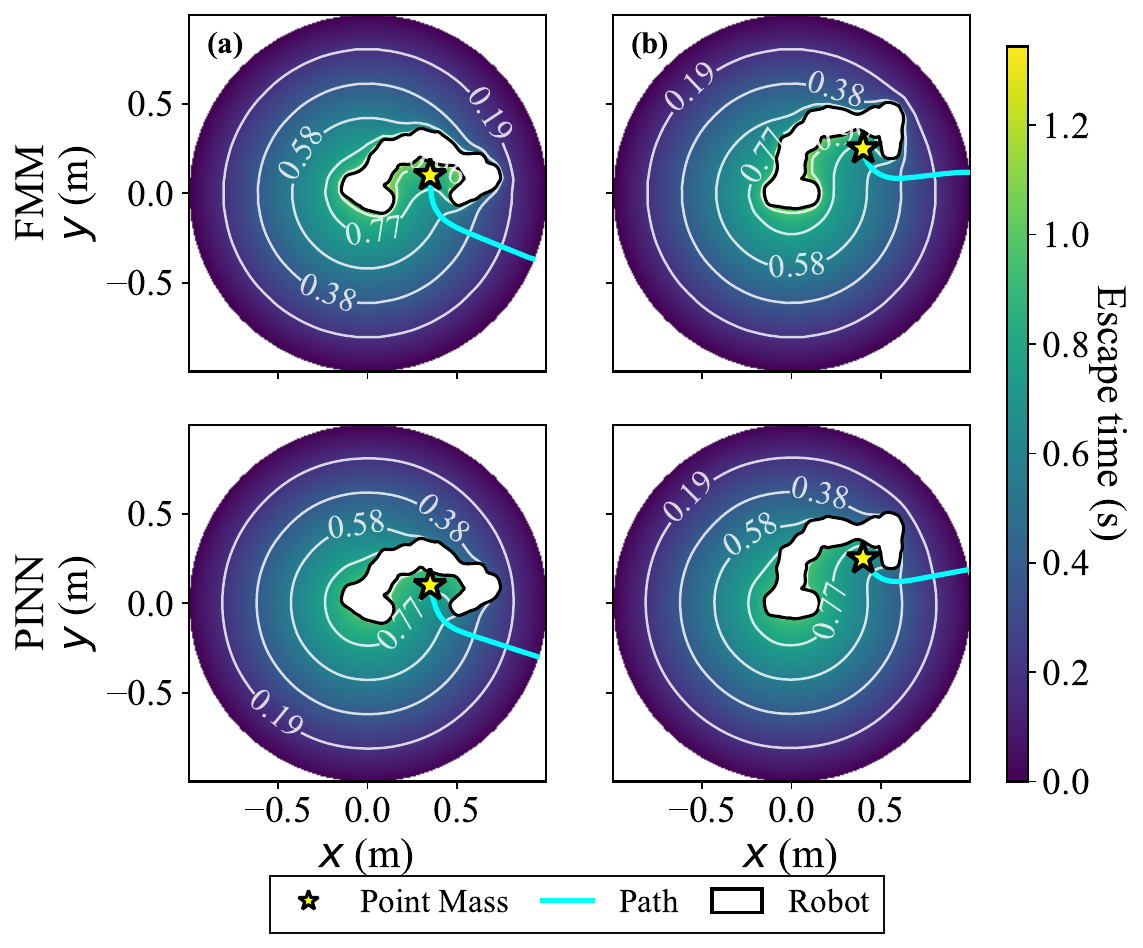}
    \caption{Comparison between the PINN-based eikonal caging time field and the ground-truth solution from FMM with Franka arm. Two representative robot configurations and point-mass positions are shown, illustrating that the learned model closely matches the true escape-time field and recovers the corresponding optimal escape paths.}
    \label{fig:pinn_vs_fmm_franka}
    \vspace{-0.5cm}
\end{figure}
We also provide a qualitative validation of the learned escape-time field for a realistic articulated robot geometry. Specifically, we compare the PINN-based eikonal caging time field against a ground-truth solution computed using the FMM for a planar Franka arm in Figure~\ref{fig:pinn_vs_fmm_franka}. This experiment is intended to verify that the proposed physics-informed formulation accurately captures escape-time structure for complex robot geometry, rather than to introduce new quantitative benchmarks.

\end{document}